\documentclass[twoside,11pt]{article}

\usepackage{blindtext}

\usepackage[preprint]{jmlr2e}

\usepackage{lastpage}
\jmlrheading{23}{2025}{1-\pageref{LastPage}}{1/21; Revised 5/22}{9/22}{21-0000}{Canzhe Zhao, Shinji Ito, and Shuai Li}

\usepackage{url}
\usepackage{dsfont}
\usepackage{amsfonts}
\usepackage{bbm}
\usepackage{amsmath}
\usepackage{amssymb}
\usepackage{mathtools}
\usepackage{enumitem}
\usepackage{booktabs}       
\usepackage{graphicx}
\usepackage{wrapfig}
\usepackage{amssymb}
\usepackage{pifont}
\newcommand{\cmark}{\ding{51}}  
\newcommand{\xmark}{\ding{55}}  

\usepackage{pifont}

\newcommand{\halfcheck}{%
    \cmark\kern-1.1ex\raisebox{.7ex}{\rotatebox[origin=c]{125}{--}}
}

\newcommand{\adv}{{\textbf{Adv.}}}%
\newcommand{\stoc}{{\textbf{Stoc.}}}%
\usepackage{placeins}

\usepackage{cleveref}
\usepackage{crossreftools}
\usepackage[table]{xcolor}
\usepackage{float}

\usepackage{multirow,array,footmisc,colortbl,bm,bbm}
\usepackage{arydshln}
\usepackage{dashrule}
\usepackage{makecell}
\usepackage[toc,page,header]{appendix}
\usepackage{minitoc}
\usepackage{footnote}
\usepackage{letterspace}
\usepackage{setspace}
\usepackage{algorithm}
\usepackage{algorithmic}
\usepackage{xspace}
\usepackage{threeparttable}

\allowdisplaybreaks[4]

\newtheorem{thm}{Theorem}
\newtheorem{lem}{Lemma}

\newtheorem{defn}{Definition}
\newtheorem{rmk}{Remark}
\newtheorem{prop}{Proposition}
\newtheorem{coro}{Corollary}

\crefname{lemma}{Lemma}{Lemmas}
\crefname{definition}{Definition}{Definitions}

\usepackage{hyperref}
\hypersetup{ 
  colorlinks   = true, 
  urlcolor     = blue, 
  linkcolor    = blue, 
  citecolor   = blue 
}

\usepackage{nicefrac}
\setlist[itemize]{leftmargin=*}

\usepackage{amsmath,amsfonts,bm}

\def\eqref#1{(\ref{#1})}

\def\1{\bm{1}}

\DeclareMathAlphabet{\mathsfit}{\encodingdefault}{\sfdefault}{m}{sl}
\SetMathAlphabet{\mathsfit}{bold}{\encodingdefault}{\sfdefault}{bx}{n}

\def\gA{{\mathcal{A}}}

\def\gF{{\mathcal{F}}}

\def\gO{{\mathcal{O}}}
\def\gP{{\mathcal{P}}}

\def\gT{{\mathcal{T}}}

\def\sN{{\mathbb{N}}}

\def\sP{{\mathbb{P}}}

\def\sR{{\mathbb{R}}}

\def\sZ{{\mathbb{Z}}}

\newcommand{\E}{\mathbb{E}}

\DeclareMathOperator*{\argmax}{arg\,max}
\DeclareMathOperator*{\argmin}{arg\,min}

\newcommand{\abs}[1]{\lvert #1 \rvert}
\newcommand{\Reg}{\operatorname{Reg}_T}

\newcommand{\wtO}{\widetilde{\gO}}

\newcommand{\hath}{\hat{h}}

\newcommand{\hatell}{\hat{\ell}}
\newcommand{\hatellta}{\hatell_{t,a}}
\newcommand{\tildeell}{\tilde{\ell}}
\newcommand{\tildeellta}{\tildeell_{t,a}}
\newcommand{\indicator}[1]{\mathds{1}\left\{#1\right\}}
\newcommand{\inlineindicator}[1]{\mathds{1}\{#1\}}
\newcommand{\phia}{\phi_{a}}

\newcommand{\phiat}{\phi_{a_t}}
\newcommand{\sumt}{\sum_{t=1}^{T}}
\newcommand{\suma}{\sum_{a\in\gA}}
\newcommand{\sumaixas}{\sum_{a\in\gA\setminus\{a^*\}}}

\newcommand{\balpha}{{\bar{\alpha}}}

\newcommand{\bphia}{\bar{\phi}_{a}}
\newcommand{\bphiat}{\bar{\phi}_{a_t}}
\newcommand{\bphib}{\bar{\phi}_{b}}

\ShortHeadings{Adversarial and BOBW Heavy-tailed Linear Bandits}{Zhao, Ito, and Li}
\firstpageno{1}

\begin{document}

\title{Heavy-tailed Linear Bandits: Adversarial Robustness, Best-of-both-worlds, and Beyond}

\makeatletter
\let\oldthefootnote\thefootnote  
\makeatother

\renewcommand{\thefootnote}{\fnsymbol{footnote}}
\author{\name Canzhe Zhao\footnotemark[1]\enspace\footnotemark[2] 
  \email canzhezhao@sjtu.edu.cn \\
  \addr John Hopcroft Center for Computer Science, Shanghai Jiao Tong University\\
  Shanghai, China
  \AND
  \name Shinji Ito\footnotemark[2] \email shinji@mist.i.u-tokyo.ac.jp \\
  \addr Department of Mathematical Informatics, The University of Tokyo\\
  RIKEN AIP\\
  Tokyo, Japan
  \AND
  \name Shuai Li\footnotemark[3] \email shuaili8@sjtu.edu.cn \\
  \addr School of Artificial Intelligence, Shanghai Jiao Tong University\\
  Shanghai, China
}
\footnotetext[1]{Part of this work was done during Canzhe Zhao's internship at RIKEN AIP.}
\footnotetext[2]{Equal contribution.}
\footnotetext[3]{Corresponding author.}  
\editor{My editor}

\makeatletter
\let\thefootnote\oldthefootnote 
\makeatother

\maketitle

\begin{abstract}

Heavy-tailed bandits have been extensively studied since the seminal work of \citet{Bubeck2012BanditsWH}. In particular, heavy-tailed linear bandits, enabling efficient learning with both a large number of arms and heavy-tailed noises, have recently attracted significant attention \citep{ShaoYKL18,XueWWZ20,ZhongHYW21,Wang2025heavy,tajdini2025improved}. However, prior studies focus almost exclusively on stochastic regimes, with few exceptions limited to the special case of heavy-tailed multi-armed bandits (MABs) \citep{Huang0H22,ChengZ024,Chen2024uniINF}. Nevertheless, these approaches face intrinsic barriers in linear settings, due either to assuming truncated non-negativity of the optimal arm's losses or to the reliance on follow-the-regularized-leader (FTRL) with a log-barrier regularizer.

In this work, we propose a general framework for adversarial heavy-tailed bandit problems, which performs FTRL over the loss estimates shifted by a bonus function. Via a delicate setup of the bonus function, we devise the first FTRL-type best-of-both-worlds (BOBW) algorithm for heavy-tailed MABs that does not require the truncated non-negativity assumption. Our algorithm achieves an $\widetilde{\mathcal{O}}(T^{\frac{1}{\varepsilon}})$ worst-case regret in the adversarial regime and an $\widetilde{\mathcal{O}}(\log T)$ gap-dependent regret in the stochastic regime, improving upon the current best BOBW result for heavy-tailed MABs without the truncated non-negativity assumption 
in both the adversarial and stochastic regimes.
We then extend our framework to the linear case, proposing the first algorithm for adversarial heavy-tailed linear bandits with finite arm sets. This algorithm achieves an $\widetilde{\mathcal{O}}(d^{\frac{1}{2}}T^{\frac{1}{\varepsilon}})$ regret, matching the best-known worst-case regret bound in stochastic regimes.
Moreover, we propose a general data-dependent learning rate, termed \textit{heavy-tailed noise aware stability-penalty matching} (HT-SPM). We prove that HT-SPM guarantees BOBW regret bounds for general heavy-tailed bandit problems once certain conditions are satisfied. By using HT-SPM and, in particular, a variance-reduced linear loss estimator, we obtain the first BOBW result for heavy-tailed linear bandits.
\end{abstract}


\section{Introduction}
The canonical multi-armed bandit (MAB) model usually assumes bounded losses or losses with sub-Gaussian noises. Nevertheless, in practical scenarios such as financial problems \citep{Gagliolo2011AlgorithmPS}, network routing problems \citep{Liebeherr2009DelayBI}, and online deep learning \citep{ZhangKVKRKS20}, the losses may be neither bounded nor have tails that decay as fast as a Gaussian. Instead, the losses in these cases may be naturally heavy-tailed and not have finite-order moments. To deal with online sequential decision-making problems accompanied by heavy-tailed losses, \citet{Bubeck2012BanditsWH} pioneer the study of MABs with stochastic heavy-tailed losses. In particular, they consider heavy-tailed MABs where the loss $\ell_{a}$ of arm $a$ in each round $t$ is sampled from an unknown yet \textit{fixed} underlying distribution $\nu_a$ such that $\mathbb{E}_{\ell_a \sim \nu_{a}}\left[|\ell_a|^\varepsilon\right] \leq \sigma$ for some known $\varepsilon\in (1,2]$ and $\sigma>0$. For this problem, they establish a (nearly) minimax-optimal regret of $\widetilde{\mathcal{O}}(\sigma^{\frac{1}{\varepsilon}} K^{1-\frac{1}{\varepsilon}} T^{\frac{1}{\varepsilon}})$ and a (nearly) optimal gap-dependent regret of $\mathcal{O}(\sum_{a \neq a^*}(\frac{\sigma}{\Delta_a})^{\frac{1}{\varepsilon-1}} \log T)$, where $K$ is the number of arms, $a^*$ is the optimal arm and $\Delta_a$ is the suboptimality gap between arm $a$ and $a^*$. Subsequently, this problem has attracted considerable research attention \citep{VakiliLZ13,KagrechaNJ19,LeeYLO20,WeiS21,AgrawalJK21,AshutoshNKJ21,Huang0H22,GenaltiM0M24,LeeL24,ChengZ024,Dorn2024Fast,Chen2024uniINF}. Nevertheless, most of these works mainly focus on heavy-tailed MABs in the stochastic regime, with the only exceptions of \citet{Huang0H22,ChengZ024} and \citet{Chen2024uniINF}. Remarkably, the algorithms of these works enjoy the \textit{best-of-both-worlds} (BOBW) property, which simultaneously guarantees a (nearly) minimax-optimal regret in the adversarial regime and a (nearly) optimal gap-dependent regret in the stochastic regime.

When dealing with a large number of arms, \citet{MedinaY16} make the first step to study the heavy-tailed bandit problem where the rewards admit a linear structure over them, \textit{i.e.}, the heavy-tailed linear bandit problem. 
Subsequent works have extensively explored this setting \citep{ShaoYKL18,XueWWZ20,ZhongHYW21,KangK23,0004WWY023,Wang2025heavy,tajdini2025improved}, with some extensions to reinforcement learning under linear function approximation and heavy-tailed noise \citep{Huang00023,Li2023variance}.
Though significant advances have emerged in heavy-tailed linear bandits, we note that all existing heavy-tailed linear bandit algorithms assume the stochastic regime, where the losses of all the arms are sampled from fixed underlying distributions. This restricts their applicability in the adversarial regime, where losses may be arbitrarily chosen and vary over time.

The persistent lack of progress on adversarial heavy-tailed linear bandits may stem from the fact that this problem remains not well-understood even in the special case of adversarial heavy-tailed MABs, in contrast to its stochastic counterpart. In particular, all existing works on adversarial---or more generally, BOBW---heavy-tailed MABs require that the loss $\ell_{t,a^*}$ of the optimal arm $a^*$, across all rounds $t \in [T]$, satisfies the truncated non-negativity assumption \citep{Huang0H22,Chen2024uniINF}, with the only exception of \citet{ChengZ024}.
Formally, this assumption requires that $\mathbb{E}[\ell_{t,a^*} \cdot \mathds{1}[|\ell_{t,a^*}| > M]] \geq 0$ for any $M \geq 0$. Although this condition is weaker than assuming that the losses of the optimal arm are strictly non-negative, it still implicitly requires the underlying distribution of the optimal arm's losses to place more probability mass on the positive part than the negative part at an arbitrary level. As discussed by \citet{ChengZ024}, this assumption is generally difficult to verify in practice. Moreover, in heavy-tailed linear bandits where $\mathbb{E}[\ell_{t,a^*}] = \langle \theta_t, \phi_{a^*} \rangle$, this assumption imposes implicit constraints on the unknown loss vector $\theta_t$ and the feature vector $\phi_{a^*}$ of the optimal arm $a^*$. As a result, it becomes even more restrictive in the linear setting. More importantly, we show that even under this assumption, directly generalizing existing approaches for adversarial heavy-tailed MABs is insufficient to address the challenges posed by adversarial heavy-tailed linear bandits (see \Cref{sec:previous_approach} for details).
On the other hand, although the algorithm proposed by \citet{ChengZ024} can deal with adversarial heavy-tailed MABs without relying on the truncated non-negativity assumption, it does so by employing follow-the-regularized-leader (FTRL) with a log-barrier regularizer to stabilize the algorithm's updates and eliminate the need for this assumption. However, the strong stability provided by the log-barrier regularizer comes at the cost of introducing a penalty term with polynomial dependence on $K$, which precludes the possibility of efficient learning in the linear setting. Moreover, the BOBW algorithm of \citet{ChengZ024} is based on the detect-and-switch scheme, which performs statistical tests to distinguish between stochastic and adversarial regimes, and subsequently applies a specialized algorithm tailored to the detected regime \citep{BubeckS12,AuerC16}. Compared to purely FTRL-based BOBW algorithms \citep{ZimmertS21}, the detect-and-switch approach is less adaptive, as it requires explicitly identifying the underlying regime \citep{ZimmertS21,LeeLWZ021}. In addition, detect-and-switch-based BOBW algorithms are usually not able to show a (nearly) optimal regret in the intermediate regime, such as the stochastic regime with adversarial corruptions. As a consequence, whether a purely FTRL-based BOBW algorithm for heavy-tailed MABs can be designed without relying on the truncated non-negativity assumption remains an open problem.

In this work, we propose a general framework to address the aforementioned challenges in adversarial heavy-tailed bandits. Our framework employs FTRL on clipped loss estimates that are shifted by a carefully designed bonus function. Through a principled selection of both the loss clipping threshold and the bonus magnitude, we demonstrate that FTRL can achieve (nearly) optimal regret for adversarial heavy-tailed bandits \emph{without} requiring the truncated non-negativity assumption. Furthermore, we show that this framework can be extended to achieve the more adaptive BOBW guarantee for heavy-tailed MABs. Remarkably, with additional technical elements, this framework can also be adapted to achieve the BOBW regret guarantee for the more challenging heavy-tailed linear bandit problem.
Our detailed contributions are summarized as follows:
\begin{itemize}
    \item We propose a general framework for the adversarial heavy-tailed bandit problem, where FTRL is performed over the clipped loss estimates that are shifted by a bonus function. The bonus function is carefully designed so that (a) it sufficiently cancels for the potential bias introduced by loss clipping, particularly on the optimal arm, while (b) maintaining controlled magnitude to ensure the overall regret remains (nearly) optimal.
By instantiating our framework with the $\frac{1}{\varepsilon}$-Tsallis regularizer and an optimally tuned clipping threshold, we obtain the first FTRL-based BOBW algorithm for heavy-tailed MABs that eliminates the need for the truncated non-negativity assumption. Our algorithm achieves an $\wtO (\sigma^{\frac{1}{\varepsilon}} K^{1-\frac 1\varepsilon}T^{\frac 1\varepsilon})$ regret in the adversarial regime, and an $\gO(\sum_{a\ne a^\ast} (\frac{\sigma}{\Delta_a})^{\frac{1}{(\varepsilon - 1)}}\log T+\sigma^{\frac{1}{\varepsilon}}(\sum_{a \neq a^*}\Delta_a^{\frac{1}{(1-\varepsilon)}})^{1-\frac{1}{\varepsilon}}\left(\log T\right)^{1-\frac{1}{\varepsilon}}C^{\frac{1}{\varepsilon}})$ regret in the adversarial regime with a self-bounding constraint, where $C \geq 0$ represents the total corruption level (see \Cref{thm:thm1} for details). In particular, the latter regime subsumes the stochastic regime as a special case when $C=0$. Notably, our result improves upon the current best-known bound of \citet{ChengZ024} for BOBW heavy-tailed MABs without the truncated non-negativity assumption, achieving a tighter regret by an $\gO(\log K \cdot \log^4 T)$ factor in adversarial regimes and an $\gO(\log K \cdot \log^3 T)$ factor in stochastic regimes. A comprehensive comparison with prior works on heavy-tailed MABs is provided in \Cref{tab:tab1}.
    \item By instantiating our framework using FTRL with the (negative) Shannon entropy regularizer, we obtain the first algorithm for the adversarial heavy-tailed linear bandit problem with finite arm sets. This algorithm achieves a regret of $\gO(\sigma^{\frac{1}{\varepsilon}} (\log K)^{1-\frac{1}{\varepsilon}} d^{\frac{1}{2}} T^{\frac{1}{\varepsilon}})$, which matches the current best-known result for the same problem in the stochastic regime \citet{tajdini2025improved}. See \Cref{thm:thm2} for details.
    \item To achieve BOBW regret guarantees for general heavy-tailed bandit problems with possibly structured feedback, we introduce a data-dependent learning rate schedule within our bonus-shifted FTRL framework, which we refer to as \textit{heavy-tailed noise aware stability-penalty matching} (HT-SPM). HT-SPM is designed to ensure that the stability term (as well as the regret incurred by additional exploration) matches the penalty term in the presence of heavy-tailed noise. This learning rate generalizes the original stability-penalty matching (SPM) learning rate proposed by \citet{ito2024adaptive} for bandit problems with bounded losses or losses with sub-Gaussian noises. We prove that under certain conditions, this learning rate leads to BOBW regret guarantees for heavy-tailed bandit problems. See \Cref{prop:prop1} for details.

    \item 
Finally, for the heavy-tailed linear bandit problem with finite arm sets, we propose a new variance-reduced linear loss estimator. This estimator exhibits lower variance than the standard linear loss estimator commonly used in the adversarial linear bandit literature (see, \textit{e.g.}, \citet{BubeckCK12}), and it enables the removal of dependence on the optimal arm $a^*$ when bounding the stability term of FTRL—a key step for applying the self-bounding inequality to derive the final BOBW regret guarantee. We show that this estimator, combined with our HT-SPM learning rate, yields the first BOBW regret guarantee for heavy-tailed linear bandits, achieving an $\wtO(\sigma^{\frac{1}{\varepsilon}} (\log K)^{1 - \frac{1}{\varepsilon}} d^{\frac{1}{2}} T^{\frac{1}{\varepsilon}})$ regret in the adversarial regime and an $\gO( (\frac{\sigma}{\Delta_{\min}})^{\frac{1}{(\varepsilon-1)}}d^{\frac{\varepsilon}{2(\varepsilon-1)}} \log K\cdot \log T)$ regret in the stochastic regime (see \Cref{thm:bobw_lb} for details). \Cref{tab:tab2} offers a detailed comparison between our results and those in related works on heavy-tailed linear bandits.
\end{itemize}

\subsection{Literature Review}
\paragraph{Heavy-tailed Bandits} The (stochastic) heavy-tailed MAB problem is first formulated and studied by \citet{Bubeck2012BanditsWH}. For this problem,
they prove a worst-case regret of $\gO(\sigma^{\frac{1}{\varepsilon}} K^{1-\frac 1\varepsilon}T^{\frac 1\varepsilon}(\log T)^{1-\frac 1\varepsilon} )$ and an gap-dependent regret of 
$\gO (\sum_{a\ne a^\ast}(\frac{\sigma}{\Delta_a})^{\frac{1}{(\varepsilon-1)}}\log T )$. 
They also provide regret lower bounds showing that their upper bounds are nearly optimal in the sense that there is only an extra $(\log T)^{1-\frac 1\varepsilon}$ factor in their worst-case regret upper bound. Subsequently, numerous studies have advanced this research area. In particular, tighter regret upper bounds or the pure exploration setting are further investigated by \citet{VakiliLZ13,KagrechaNJ19,LeeYLO20,WeiS21,LeeL24}. 
Moreover, Lipschitz bandits with heavy-tailed losses have been investigated by \citet{LuWHZ19}.
Besides, \citet{Huang0H22} initiate the study of adversarial and BOBW heavy-tailed MAB problems. Specifically, they establish a BOBW algorithm with (nearly) optimal regret guarantees in both the adversarial and stochastic regimes with the knowledge of $\varepsilon$ and $\sigma$, and a (nearly) minimax optimal algorithm for adversarial heavy-tailed MABs without the knowledge of $\varepsilon$ and $\sigma$. Further, \citet{Chen2024uniINF} establish a BOBW algorithm for heavy-tailed MABs without the knowledge of $\varepsilon$ and $\sigma$. We note that all the existing works studying adversarial and BOBW heavy-tailed MABs postulate that the loss sequence of the optimal arm across the whole learning process satisfies the truncated non-negativity assumption, except for \citet{ChengZ024}, which develop a detect-and-switch based BOBW algorithm using FTRL with a log-barrier regularizer as a subroutine.

To address the challenge of learning with a large number of arms, the stochastic heavy-tailed linear bandit problem has received considerable attention in recent years \citep{MedinaY16,ShaoYKL18,XueWWZ20,ZhongHYW21,KangK23,0004WWY023,Wang2025heavy,tajdini2025improved}. In particular, \citet{MedinaY16} extend the truncation-based and median-of-means (MOM) methods of \citet{Bubeck2012BanditsWH} to the linear setting, and propose the first algorithm for the stochastic heavy-tailed linear bandit problem, achieving regret bounds of $\widetilde{\mathcal{O}}(d T^{\frac{1+\varepsilon}{2\varepsilon}})$ and $\widetilde{\mathcal{O}}(d T^{\frac{2 \varepsilon-1}{3 \varepsilon-2}})$, respectively. A follow-up work by \citet{ShaoYKL18} develop both a truncation-based algorithm and an MOM-based algorithm, each attaining a regret bound of $\widetilde{\mathcal{O}}(d T^{\frac{1}{\varepsilon}})$. Furthermore, when the arm set is fixed and finite, \citet{XueWWZ20} propose two algorithms based on MOM and truncation, which achieve regret upper bounds of $\widetilde{\mathcal{O}}(\sqrt{d\log (KT)} T^{\frac{1}{\varepsilon}}\log T)$ and $\widetilde{\mathcal{O}}(\sqrt{d} T^{\frac{1}{\varepsilon}}\log (KT)\log T)$, respectively. For the same setting, they also establish a regret lower bound of $\Omega(d^{\frac{\varepsilon-1}{\varepsilon}} T^{\frac{1}{\varepsilon}})$. The other line of research applies the Huber regression technique \citep{Sun2020AdaptiveHR} to stochastic heavy-tailed linear bandits \citep{KangK23,Li2023variance,Huang00023,Wang2025heavy}. In particular, several works obtain a data-dependent regret bound of $\wtO( d T^{\frac{1}{\varepsilon}-\frac{1}{2}} \sqrt{\sum_{t=1}^T \sigma_t^{2/\varepsilon}}+d T^{\frac{1}{\varepsilon}-\frac{1}{2}} )$
under the assumption that $\E_{\ell_{t,a} \sim \nu_{t,a}}[|\ell_{t,a} - \langle \phi_a, \theta \rangle|^{\varepsilon}] \le \sigma_t$ for all $t \in [T]$ \citep{Huang00023,Wang2025heavy}. This result recovers the previously known regret upper bound of $\widetilde{\mathcal{O}}(d T^{\frac{1}{\varepsilon}})$ \citep{ShaoYKL18} when $\sigma_t = \sigma$ is constant across all rounds. To the best of our knowledge, however, all existing studies focus exclusively on the stochastic setting of the heavy-tailed linear bandit problem, and no work has addressed its adversarial counterpart, in contrast to the case of the heavy-tailed MABs.

\FloatBarrier
\begin{table}[htbp]
\caption{Comparisons with existing literature for heavy-tailed MABs.}\label{tab:tab1}
\begin{minipage}{\textwidth}
    \centering
    \begin{savenotes}
    \renewcommand{\arraystretch}{1.5}
    \resizebox{\textwidth}{!}{%
    \begin{tabular}{|c|c|c|c|c|}\hline
    \phantom{\footnotemark}\textbf{Algorithm}\footnote{We note that OptHTINF, AdaHINF \citep{Huang0H22}, AdaR-UCB \citep{GenaltiM0M24}, and uniTNF \citep{Chen2024uniINF} can work in cases where the knowledge of $\varepsilon$ and $\sigma$ is unknown.}  & \phantom{\footnotemark} \textbf{TNA-Free?}\footnote{\textbf{TNA-Free?} denotes whether the algorithm can work in cases where the truncated non-negativity assumption (TNA) of the loss sequence of the optimal arm $a^*$ does not hold. Please see \Cref{sec:previous_approach} for the definition and discussions on this assumption.} & \phantom{\footnotemark} \textbf{Regime}\footnote{
    \textbf{Regime} specifies the environments in which the algorithm can operate. While the works of \citet{Huang0H22} and \citet{Chen2024uniINF} do not explicitly state this, their algorithms---like ours---can operate in \textbf{stochastic regimes with adversarial corruptions}, in addition to \textbf{stochastic} and \textbf{adversarial regimes}. However, this capability is not shared by SAO-HT, the switching-based BOBW algorithm proposed by \cite{ChengZ024}.
    Please see \Cref{def:def1} and \Cref{thm:thm1} for details.} & \phantom{\footnotemark} \textbf{Regret} 
    \footnote{Expected regret bounds for algorithms in respective regimes. We note that the expected regret bounds presented for RobustUCB \citep{Bubeck2012BanditsWH}, OMD-LB-HT, and SAO-HT \citep{ChengZ024} are derived from their high-probability ones.}
& \phantom{\footnotemark} \textbf{Opt.?} \footnote{\textbf{Opt.?} specifies whether the regret bound is optimal up to multiplicative logarithmic factors of the form $\log \square$, where $\square$ may represent any parameters other than $K$ and $T$.} \\\hline
    \multirow{2}{*}{\shortstack{Lower Bound\\[3pt]\citep{Bubeck2012BanditsWH}}} & \multirow{2}{*}{---} & \multirow{2}{*}{---} & $\Omega\left (\sigma^{\frac{1}{\varepsilon}} K^{1-\frac 1\varepsilon}T^{\frac 1\varepsilon}\right )$  & \multirow{2}{*}{---} \\\cdashline{4-4}
    & & & $\Omega\left (\sum_{a\ne a^\ast}(\frac{\sigma}{\Delta_a})^{\frac{1}{\varepsilon-1}}\log T\right )$ & \\\hline
    \multirow{2}{*}{\shortstack{{RobustUCB}\\[3pt]\citep{Bubeck2012BanditsWH}}} & \multirow{2}{*}{\cmark} & \multirow{2}{*}{\stoc} & $\gO\left(\sigma^{\frac{1}{\varepsilon}} K^{1-\frac 1\varepsilon}T^{\frac 1\varepsilon}(\log T)^{1-\frac 1\varepsilon} \right)$ & \cmark \\\cdashline{4-5}
    & & & $\gO\left (\sum_{a\ne a^\ast}(\frac{\sigma}{\Delta_a})^{\frac{1}{\varepsilon-1}}\log T\right )$ & \cmark \\\hline
    \multirow{2}{*}{\shortstack{{APE$^2$}\\[3pt] \citep{LeeYLO20}}} & \multirow{2}{*}{\cmark} & \multirow{2}{*}{\stoc} & $\wtO\left (\exp(\sigma^{\nicefrac{1}{\varepsilon^2}}) K^{1-\frac 1\varepsilon}T^{\frac 1\varepsilon}\right )$  & \xmark \\\cdashline{4-5}
    & & & $\gO\left(\exp(\sigma^{\nicefrac{1}{\varepsilon}}) + \sum_{a\neq a^\ast}(\frac{1}{\Delta_a})^{\frac{1}{\varepsilon - 1}}(T\Delta_a^{\frac{\varepsilon}{\varepsilon - 1}}\log K)^{\frac{\varepsilon}{(\varepsilon - 1)\log K}}\right)$ & \xmark \\\hline
    \multirow{2}{*}{\shortstack{{Robust MOSS}\\[3pt]\citep{WeiS21}}} & \multirow{2}{*}{\cmark} & \multirow{2}{*}{\stoc} & $\gO\left (\sigma^{\frac{1}{\varepsilon}} K^{1-\frac 1\varepsilon}T^{\frac 1\varepsilon}\right )$  & \cmark  \\\cdashline{4-5}
    & & & $\gO\left (\sum_{a\ne a^\ast} (\frac{\sigma}{\Delta_a})^{\frac{1}{\varepsilon-1}} \log (\frac{T}{K} (\frac{\sigma}{\Delta_a^\varepsilon})^{\frac{1}{\varepsilon-1}})\right )$ & \phantom{\footnotemark} {\cmark} \footnote{Up to logarithmic factors of $\log \sigma$ and $\log(1/\Delta_a^\varepsilon)$.} \\\hline
    \multirow{2}{*}{\shortstack{{HTINF}\\[3pt]\citep{Huang0H22}}} & \multirow{2}{*}{\xmark} & \adv & $\gO\left(\sigma^{\frac{1}{\varepsilon}} K^{1-\frac 1\varepsilon} T^{\frac 1 \varepsilon} \right )$  & \cmark\\
    \cdashline{3-5}
    & & \stoc & $\gO\left(\sum_{a\ne a^\ast} (\frac{\sigma}{\Delta_a})^{\frac 1 {\varepsilon - 1}}\log T\right)$ & \cmark\\\hline
    \multirow{2}{*}{\shortstack{{OptHTINF}\\[3pt]\citep{Huang0H22}}} & \multirow{2}{*}{\xmark} & \adv & $\gO\left (\sigma K^{\frac{\varepsilon-1}{2}}T^{\frac{3-\varepsilon}{2}}\right )$  & \xmark\\
    \cdashline{3-5}
    & & \stoc & $\gO\left(\sum_{a\ne a^\ast}(\frac{\sigma^{2}}{\Delta_a^{3-\varepsilon}})^{\frac{1}{\varepsilon-1}}\log T\right )$ & \xmark\\\hline
    \multirow{2}{*}{\shortstack{{AdaTINF} \\[3pt] \citep{Huang0H22}}} & \multirow{2}{*}{\xmark} & \multirow{2}{*}{\adv} & \multirow{2}{*}{$\gO\left( \sigma^{\frac{1}{\varepsilon}} K^{1-\frac{1}\varepsilon} T^{\frac 1 \varepsilon} \right)$} & \multirow{2}{*}{\cmark}\\
    & & & & \\\hline
    \multirow{2}{*}{\shortstack{{AdaR-UCB}\\[3pt]\citep{GenaltiM0M24}}} & \multirow{2}{*}{\xmark} & \multirow{2}{*}{\stoc} & $\wtO\left (\sigma^{\frac{1}{\varepsilon}} K^{1-\frac 1\varepsilon}T^{\frac 1\varepsilon}\right )$  & \cmark \\\cdashline{4-5}
    & & & $\gO\left (\sum_{a\ne a^\ast} (\frac{\sigma}{\Delta_a})^{\frac{1}{\varepsilon-1}} \log T\right )$ & \cmark \\\hline
    \multirow{2}{*}{\shortstack{{OMD-LB-HT} \\[3pt] \citep{ChengZ024}}} & \multirow{2}{*}{\cmark} & \multirow{2}{*}{\adv} & \multirow{2}{*}{$\gO\left( \sigma^{\frac{1}{\varepsilon}} K^{1-\frac{1}\varepsilon} T^{\frac 1 \varepsilon}\log^3 T \right)$} & \multirow{2}{*}{\xmark}\\
    & & & & \\\hline
    \multirow{2}{*}{\shortstack{SAO-HT\\[3pt]\citep{ChengZ024}}} & \multirow{2}{*}{\cmark} & \adv & $\gO\left (\sigma^{\frac{1}{\varepsilon}} K^{1-\frac 1\varepsilon}T^{\frac 1\varepsilon}\log^4 T\cdot \log K\right )$ &  {\xmark} \\\cdashline{3-5}
    & & \stoc & $\gO\left (K (\frac{\sigma}{\Delta_{\min}})^{\frac{1}{\varepsilon-1}} \log^4 T\cdot \log K\right )$ & \xmark \\\hline
    \multirow{2}{*}{\shortstack{uniINF\\[3pt]\citep{Chen2024uniINF}}} & \multirow{2}{*}{\xmark} & \adv & $\wtO\left (\sigma^{\frac{1}{\varepsilon}} K^{1-\frac 1\varepsilon}T^{\frac 1\varepsilon}\right )$  & \cmark \\\cdashline{3-5}
    & & \stoc & $\gO\left (K (\frac{\sigma}{\Delta_{\min}})^{\frac{1}{\varepsilon-1}} \log T \cdot \log \frac{\sigma}{\Delta_{\min}}\right )$ &  \phantom{\footnotemark} {\cmark} \footnote{Up to logarithmic factors of $\log\sigma$ and $\log(1/\Delta_{\min})$, assuming all sub-optimal gaps $\Delta_a$ are comparable to the minimal gap $\Delta_{\min}$.} \\\hline
    \rowcolor{lightgray!30} 
     &  & \adv & $\gO\left (\sigma^{\frac{1}{\varepsilon}} K^{1-\frac 1\varepsilon}T^{\frac 1\varepsilon}\right )$  & \cmark \\\cdashline{3-5} 
     \rowcolor{lightgray!30}
     \multirow{-2}{*}{\shortstack{\Cref{algo:algo1}\\[3pt] \textbf{(Ours)}}}
     & \multirow{-2}{*}{\cmark} & \stoc & $\gO\left(\sum_{a\ne a^\ast} (\frac{\sigma}{\Delta_a})^{\frac 1 {\varepsilon - 1}}\log T\right)$ &  {\cmark}  \\\hline
    \end{tabular}}
    \end{savenotes}
\end{minipage}
\end{table}

\paragraph{Best-of-Both-Worlds Bandit Algorithms} 
In online learning, two canonical regimes characterize how losses are generated throughout the learning process. In the \emph{adversarial} regime, the (distributions of) loss functions may be chosen arbitrarily by an adversary and can vary across different rounds. In contrast, the \emph{stochastic} regime assumes that the loss distributions are unknown yet fixed, and the learner receives i.i.d.\ samples from these distributions in each round. While numerous algorithms achieve (nearly) optimal regret guarantees in either the stochastic or adversarial regime for various bandit problems, they typically require prior knowledge of the underlying regime to be applied effectively. In practice, however, the nature of the environment is often unknown. Without this knowledge, algorithms designed for adversarial settings may incur overly conservative regret in stochastic regimes. Conversely, algorithms tailored to stochastic regimes may suffer linear regret in adversarial regimes, leading to a complete failure of the learning process.

To tackle this challenge, the development of best-of-both-worlds (BOBW) algorithms, which can simultaneously achieve (nearly) optimal worst-case regret guarantees in adversarial regimes and (nearly) optimal gap-dependent regret guarantees in stochastic regimes, has attracted significant attention in the online learning literature. 
Broadly speaking, existing BOBW algorithms fall into two categories. The first category is based on the \emph{detect-and-switch} paradigm, which performs statistical tests to distinguish between regimes and subsequently applies algorithms tailored to the detected regime. This approach has been studied in MABs \citep{BubeckS12,SeldinS14,AuerC16,SeldinL17}, bandits with graph feedback \citep{0002ZL22}, and linear bandits \citep{LeeLWZ021}. In particular, \citet{ChengZ024} propose a detect-and-switch BOBW algorithm for heavy-tailed MABs. The second category of BOBW algorithms operates entirely within the FTRL framework, without relying on statistical tests or explicit regime switching. These algorithms typically offer better adaptivity and can even achieve (nearly) optimal regret guarantees in \emph{intermediate} settings, such as the stochastic regime with adversarial corruptions. This line of works includes BOBW algorithms for MABs~\citep{WeiL18,ZimmertS21,Ito21,JinLL23}, linear bandits~\citep{DannWZ23,KongZ023,ItoT23,ItoT23EXO,ito2024adaptive}, combinatorial MABs~\citep{ZimmertLW19,Ito21Hybrid,TsuchiyaIH23}, bandits with graph feedback~\citep{ErezK21,Ito2022nearly,Rouyer2022nearoptimal,DannWZ23,ito2024adaptive,tsuchiyasimple}, online learning to rank~\citep{ChenZL22}, online submodular minimization~\citep{Ito22}, partial monitoring~\citep{Tsuchiya2022bobw,tsuchiyasimple}, and (tabular) Markov decision processes (MDPs)~\citep{jin2020simultaneously,jin2021best,dann2023best}. Notably, \citet{Huang0H22} and \citet{Chen2024uniINF} have developed FTRL-based BOBW algorithms for heavy-tailed MABs.

\FloatBarrier
\begin{table}[htbp]
\caption{Comparisons with existing literature for heavy-tailed linear bandits.}\label{tab:tab2}
\begin{minipage}{\textwidth}
    \centering
    \begin{savenotes}
    \renewcommand{\arraystretch}{1.5}
    \resizebox{\textwidth}{!}{%
    \begin{tabular}{|c|c|c|c|}\hline
    \phantom{\footnotemark}\textbf{Algorithm}\footnote{All the algorithms in this table do not require the truncated non-negativity assumption.}\footnote{Note that HEAVY-OFUL \citep{Huang00023}, Hvt-UCB \citep{Wang2025heavy}, and MED-PE \citep{tajdini2025improved} work under the assumption that the $\varepsilon$-th order central moment is bounded (that is, $\E_{\ell_{a}\sim \nu_{a}}[|\ell_{a}-\langle \phi_{a}, \theta\rangle|^{\varepsilon}]\le \sigma$), while the remaining algorithms including ours assume that the $\varepsilon$-th order raw moment is bounded.}  & \textbf{Regime} &  \textbf{Regret} 
& \phantom{\footnotemark} \textbf{Opt.?} \footnote{\textbf{Opt.?} specifies whether the regret is optimal up to all the logarithmic factors in the worst case or whether the regret is an $\gO(\log T)$-type gap-dependent one. {\cmark \dag} and {\cmark \ddag} indicate that the regret is optimal when $\varepsilon=2$ (up to logarithmic factors) in the setting of finite arm sets and infinite arm sets, respectively.} \\\hline
    \multirow{2}{*}{\shortstack{Lower Bound\\[3pt]\citep{tajdini2025improved}}} & \multirow{2}{*}{---} & \multirow{2}{*}{\phantom{\footnotemark} $\Omega\left(\sigma^{\frac{1}{\varepsilon}}(\log K)^{1-\frac{1}{\varepsilon}} d^{1-\frac{1}{\varepsilon}}T^{\frac{1}{\varepsilon}} \right)$ \footnote{Regret lower bound for the setting of finite arm sets. \citet{tajdini2025improved} also establish a regret lower bound of $\Omega(\sigma^{\frac{1}{\varepsilon}} d^{2-\frac{2}{\varepsilon}}T^{\frac{1}{\varepsilon}})$ for heavy-tailed linear bandits with an infinite arm set.
    The factor $\sigma^{\frac{1}{\varepsilon}}$ does not appear in their original lower bound as they additionally assume $\sup _{a \in \gA}\left|\phia^{\top} \theta\right| \leq 1$.} }  & \multirow{2}{*}{---} \\
    & & & \\\hline
    \multirow{2}{*}{\shortstack{{TOFU}\\[3pt]\citep{ShaoYKL18}}} & \multirow{2}{*}{\phantom{\footnotemark}\stoc \phantom{\footnotemark}} & \multirow{2}{*}{$\gO\left(\sigma^{\frac{1}{\varepsilon}}(\log T)^{\frac{3}{2}-\frac{1}{\varepsilon}} dT^{\frac{1}{\varepsilon}}\right)$} & \multirow{2}{*}{ \cmark \ddag \footnote{Optimal when $\varepsilon=2$ in the setting of infinite arm sets (up to logarithmic factors). While \cite{ShaoYKL18} prove a lower bound of $\Omega(dT^{\frac{1}{1+\varepsilon}})$ and claim that their upper bound is optimal for all $\varepsilon\in (1,2]$, the hard instance used in their lower bound is constructed only for the case with $\E_{\ell_{a}\sim \nu_{a}}[|\ell_{a}-\langle \phi_{a}, \theta\rangle|^{\varepsilon}]=\gO(d)$. \citet{tajdini2025improved} prove a refined lower bound $\Omega(d^{2-\frac{2}{\varepsilon}} T^{\frac{1}{\varepsilon}})$ for $\E_{\ell_{a}\sim \nu_{a}}[|\ell_{a}-\langle \phi_{a}, \theta\rangle|^{\varepsilon}]=\gO(1)$ as well as an improved upper bound $\wtO(d^{\frac{3 \varepsilon-2}{2\varepsilon}} T^{\frac{1}{\varepsilon}})$ for heavy-tailed linear bandits with an infinite arm set.}} \\
    & & & \\\hline
    \multirow{2}{*}{\shortstack{SupBTC\\[3pt] \citep{XueWWZ20}}} & \multirow{2}{*}{\phantom{\footnotemark}\stoc \phantom{\footnotemark}} & \multirow{2}{*}{$\gO\left(\sigma^{\frac{1}{\varepsilon}}\sqrt{\log (KT)}\log T\cdot d^{\frac{1}{2}}T^{\frac{1}{\varepsilon}}\right)$} & \multirow{2}{*}{\cmark\dag} \\
    & & &  \\\hline
    \multirow{2}{*}{\shortstack{{HEAVY-OFUL}\\[3pt]\citep{Huang00023}}} & \multirow{2}{*}{\phantom{\footnotemark}\stoc \phantom{\footnotemark}} & \multirow{2}{*}{\phantom{\footnotemark} $\wtO\left( d T^{\frac{1}{\varepsilon}-\frac{1}{2}} \sqrt{\sum_{t=1}^T \sigma_t^{\nicefrac{2}{\varepsilon}}}+d T^{\frac{1}{\varepsilon}-\frac{1}{2}} \right)$ \footnote{Here \citet{Huang00023} suppose that $\E_{\ell_{t,a}\sim \nu_{t,a}}[|\ell_{t,a}-\langle \phi_{a}, \theta\rangle|^{\varepsilon}]\le \sigma_t$ for all $t\in[T]$.} }  & \multirow{2}{*}{\cmark \ddag}  \\
    & & & \\\hline
    \multirow{2}{*}{\shortstack{{Hvt-UCB}\\[3pt]\citep{Wang2025heavy}}} & \multirow{2}{*}{\phantom{\footnotemark}\stoc \phantom{\footnotemark}} & \multirow{2}{*}{ $\wtO\left( d T^{\frac{1}{\varepsilon}-\frac{1}{2}} \sqrt{\sum_{t=1}^T \sigma_t^{\nicefrac{2}{\varepsilon}}}+d T^{\frac{1}{\varepsilon}-\frac{1}{2}} \right)$ }  & \multirow{2}{*}{\cmark \ddag}  \\
    & & & \\\hline
    \multirow{2}{*}{\shortstack{{MED-PE}\\[3pt]\citep{tajdini2025improved}}} & \multirow{2}{*}{\phantom{\footnotemark}\stoc\phantom{\footnotemark}} & \multirow{2}{*}{ $\gO\left(\sigma^{\frac{1}{\varepsilon}} (\log K)^{1-\frac{1}{\varepsilon}} d^{\frac{1}{2}} T^{\frac{1}{\varepsilon}} \right)$ }  & \multirow{2}{*}{\cmark \dag}  \\
    & & & \\\hline
    \rowcolor{lightgray!30}& & & \\
    \rowcolor{lightgray!30}\multirow{-2}{*}{\shortstack{\Cref{algo:algo2}\\[3pt] \textbf{(Ours)}} } & \multirow{-2}{*}{\phantom{\footnotemark}\adv\phantom{\footnotemark}} & \multirow{-2}{*}{$\gO\left(\sigma^{\frac{1}{\varepsilon}} (\log K)^{1-\frac{1}{\varepsilon}} d^{\frac{1}{2}} T^{\frac{1}{\varepsilon}}\right)$} & \multirow{-2}{*}{\cmark \dag}\\ \hline
    \rowcolor{lightgray!30}
     & \phantom{\footnotemark}\adv\phantom{\footnotemark} & $\wtO\left ( \sigma^{\frac{1}{\varepsilon}}(\log K)^{1-\frac{1}{\varepsilon}} d^{\frac{1}{2}} T^{\frac{1}{\varepsilon}}\right )$  & \cmark\dag \\ \cdashline{2-4} 
     \rowcolor{lightgray!30}
     \multirow{-2}{*}{\shortstack{\Cref{algo:algo3}\\[3pt] \textbf{(Ours)}}}
     & \phantom{\footnotemark}\stoc\footnote{Similar to the tabular case, our algorithm also enjoys an $\gO(\log T+C^{\frac{1}{\varepsilon}})$-type regret in \textbf{stochastic regimes with adversarial corruptions}.
    Please refer to \Cref{thm:bobw_lb} for more details.} &  $\gO\left( d^{\frac{\varepsilon}{2(\varepsilon-1)}}(\frac{\sigma}{\Delta_{\min}})^{\frac{1}{\varepsilon-1}} \log T \cdot \log K /(\varepsilon-1)\right)$  &  \cmark\phantom{\dag}  \\\hline
    \end{tabular}}
    \end{savenotes}
\end{minipage}
\end{table}
\section{Preliminaries}
In this section, we formulate the problem setup of the heavy-tailed linear bandits in the adversarial regime and the adversarial regime with a self-bounding constraint, which encompasses the stochastic regime as a special case.

\paragraph{Notations}
We denote by $\gA$ the finite (and fixed) set of arms and by $K=|\gA|$ the number of arms. Let $\{\phi_a\}_{a\in\gA} \subseteq \mathbb{R}^d$ be the set of the arm feature vectors, where $d$ is the ambient dimension of the arm feature vectors. Without loss of generality, we assume $\{\phi_a\}_{a\in\gA}$ spans $\mathbb{R}^d$ \citep{lattimore2020bandit}. 
Also, we define $\log_{+} (x)=\max\{1,\log x\}$ for any $x>0$ and denote by $\gP(\gA)=\{p\in\sR^K_{\ge 0}:\|p\|_1=1\}$ the set of probability distributions over $\gA$.
For any natural number $n\in \sN$, let $[n]=\{1,2,\ldots,n\}$. 
For any (differentiable) convex function $\psi$, let $D_{\psi}(x,y)=\psi(x)-\psi(y)-\langle \nabla \psi(y),x-y \rangle$ be the Bregman divergence induced by $\psi$. Additionally, we sometimes write $f\lesssim g$ if $f=\gO(g)$.

The interaction between the learner and the environment proceeds in $T$ rounds. In each round $t\in[T]$, the environment determines an \emph{unknown} loss vector $ \theta_t \in \mathbb{R}^d$. For each arm $a\in\gA$, the loss $\ell_{t,a}$ of arm $a$ in round $t$ adheres to a distribution $\nu_{t,a}$ such that $\E_{\ell_{t,a}\sim \nu_{t,a}}[\ell_{t,a}]=\langle \phi_{a}, \theta_{t}\rangle$. Meanwhile, the learner specifies a distribution $p_t\in\gP(\gA)$ and samples an arm $a_t\sim p_t$. Then the learner observes an incurred loss $\ell_{t,a_t}$ sampled from the loss distribution $\nu_{t,a_t}$. The performance of an algorithm adopted by the learner is measured in terms of the \textit{pseudo-regret} (\textit{regret} for short in what follows), defined as
\begin{align}
    \Reg=\max_{a\in\gA}\mathbb{E}\left[\sum_{t=1}^T\ell_{t,a_t}-\sum_{t=1}^T\ell_{t,a}\right]\,,\notag
\end{align}
where the expectation is taken over both the randomness in the loss sequence $\{\ell_t\}_{t=1}^T$ and the internal randomness of the algorithm. 
In particular, we denote by $a^\ast\in\argmin_{a\in \gA}\mathbb{E}\left[\sum_{t=1}^T\ell_{t,a}\right]$ the optimal arm in expectation in hindsight.
Further, we denote by $\gF_t$ the $\sigma$-algebra generated by $\{a_1,\ell_{1,a_1},\ldots,a_t,\ell_{t,a_t}\}$. For simplicity, we abbreviate $\E[\cdot\mid \gF_t]$ as $\E_{t}[\cdot]$.

In the \textit{adversarial regime}, the parameter vector $\theta_t$ is chosen arbitrarily in each round $t$. In this work, we focus on the setting of the \textit{non-oblivious} adversary, meaning that $\theta_t$ may be chosen depending on the history of the learner's past actions $\{a_{\tau}\}_{\tau=1}^{t-1}$ up to round $t-1$, as well as the adversary's internal randomization. Besides, we also consider the \textit{adversarial regime with a self-bounding constraint},  formally defined as follows.
\begin{defn}[Adversarial Regime with a Self-bounding Constraint, \citet{ZimmertS21}]\label{def:def1}
    Let $C\ge 0$, $\Delta\in \sR^K_{\ge 0}$ and $T\in \sN$. An environment is in an adversarial regime with a $(\Delta, C, T)$ self-bounding constraint, if for any algorithm, the regret satisfies 
    \begin{align}
        \Reg\ge \E\left[\sum_{t=1}^T \Delta_{a_t}\right]-C=\E \left[\sum_{t=1}^T \sum_{a\in\gA} \Delta_a p_{t,a}\right]-C\,.\notag
    \end{align}
\end{defn}
As discussed by \citet{ZimmertS21}, this regime subsumes the \textit{stochastic regime} and the \textit{stochastic regime with adversarial corruptions} \citep{LykourisML18} as special cases, where $\Delta_a\ge 0$ is the suboptimality gap, $C=0$ for the stochastic regime, and $C>0$ is the total amount of the corruption for the stochastic regime with adversarial corruptions. As in previous studies for BOBW bandit problems \citep{ZimmertS21,ItoT23EXO,DannWZ23,KongZ023,ItoT23,ZimmertLW19,Ito21Hybrid,TsuchiyaIH23,Ito2022nearly,Tsuchiya2022bobw,tsuchiyasimple,ito2024adaptive}, in the adversarial regime with a self-bounding constraint, we assume that the optimal arm $a^*$ is unique and $\Delta_a>0$ holds for all $a\in \gA\setminus\{a^*\}$. In this case, we 
denote by $p^{\ast}$ the Dirac measure centered on $a^*$ (\textit{i.e.}, $p^*$ is the probability distribution assigning unit mass to $a^*$ and zero elsewhere), and
let $\Delta_{\min }=\min _{a \in\gA \backslash\left\{a^*\right\}} \Delta_a$ be the minimum suboptimality gap. 
Further, we sometimes write $a \neq a^*$ as shorthand for all $a \in\gA \setminus \{a^*\}$.

Unlike the conventional assumptions of bounded variance or even bounded support of the loss functions in existing works studying adversarial or BOBW linear bandits \citep{Abernethy2008CompetingIT,BubeckCK12,DannWZ23,ItoT23,KongZ023,ItoT23EXO,ito2024adaptive}, we only require that the distributions of the losses are \textit{heavy-tailed}, as specified below.
\begin{defn}[Heavy-tailed Loss Distributions]\label{def:HT}
    The $\varepsilon$-th (raw) moment of the loss distribution $\nu_{t,a}$ satisfies $\E_{\ell_{t,a}\sim \nu_{t,a}}[|\ell_{t,a}|^{\varepsilon}]\le \sigma$ for some fixed $\varepsilon\in (1,2]$, $\sigma>0$, all $t\in [T]$ and $a\in\gA$.
\end{defn}
\begin{rmk}\label{rmk:rmk1}
Throughout this work, we assume that the learner has access to the parameters $\varepsilon$ and $\sigma$, consistent with the assumptions made in prior studies on stochastic heavy-tailed linear bandits \citep{ShaoYKL18,XueWWZ20,0004WWY023,Huang00023,Wang2025heavy,tajdini2025improved}.
While some works on adversarial and BOBW algorithms for heavy-tailed MABs do not require knowledge of $\varepsilon$ and $\sigma$ \citep{Huang0H22,Chen2024uniINF}, these approaches rely on the truncated non-negativity assumption over the loss sequence of the optimal arm $a^*$ (see \Cref{sec:previous_approach} for details).
Since we do not impose any additional assumptions on the heavy-tailed distribution, assuming knowledge of $\varepsilon$ and $\sigma$ is technically reasonable. In particular, \citet{GenaltiM0M24} prove that, in the absence of further assumptions on the heavy-tailed loss distribution, the worst-case lower bound of the case with known $\varepsilon$ and $\sigma$ cannot be achieved in the case with unknown $\varepsilon$ and $\sigma$.
\end{rmk}
\section{Challenges and Technical Insights}
In this section, we first discuss the limitations of existing works on BOBW algorithms for heavy-tailed MABs, as well as the challenges in extending these approaches to the linear setting. We then present the core idea of our proposed algorithmic framework for addressing heavy-tailed bandit problems.
\subsection{Challenges in Adversarial Heavy-tailed Linear Bandits}\label{sec:previous_approach}

In online learning and adversarial bandits, FTRL is one of the powerful frameworks for achieving sublinear regret. In each round $t$ of this framework, a probability distribution $p_t \in \gP(\gA)$ is computed by solving the convex optimization problem
\begin{align}
    p_t \in \argmin_{q \in \gP(\gA)} \sum_{s=1}^{t-1} \hat{f}_s(q) + \frac{1}{\eta_t} \psi(q)\,,
\end{align}
where $\hat{f}_t : \gP(\gA) \to \sR$ is an estimator of the true loss function $f_t$, 
$\eta_t > 0$ is the learning rate (with $\{\eta_t\}$ being a monotone non-increasing sequence), 
and $\psi$ is a convex regularizer. An action $a_t$ is then sampled from $p_t$ to interact with the environment.  
A standard analysis of FTRL (see, \textit{e.g.}, Exercise~28.12 of \citet{lattimore2020bandit}) shows that its regret can be upper bounded as
\begin{align}
\Reg \lesssim 
\underbrace{\sumt \eta_t z_t}_{\text{Stability term}}
+ \underbrace{\frac{1}{\eta_{1}} h_1 
+ \sum_{t=2}^T \left(\frac{1}{\eta_t} - \frac{1}{\eta_{t-1}}\right) h_t}_{\text{Penalty term}} \,,
\end{align}
for some quantities $z_t$ and $h_t$ that depend on the specific problem setting and the choice of the regularizer $\psi$.

For the adversarial heavy-tailed linear bandit problem, a plausibly reasonable approach might be to extend the FTRL algorithms and analysis established for adversarial heavy-tailed MABs \citep{Huang0H22,ChengZ024,Chen2024uniINF}. 
However, directly applying their methods would not suffice for solving this problem. In detail, when using the FTRL framework to deal with the adversarial losses, as the losses in heavy-tailed MABs no longer have bounded support, the usual loss estimator $\tildeellta=\frac{\ell_{t,a_t}\indicator{a=a_t}}{p_{t,a}}$ can be arbitrarily negative, which might preclude a stable update of FTRL and prevent from upper bounding the stability term of FTRL effectively. 
To this end, \citet{Huang0H22} use an $\alpha$-Tsallis entropy with $\alpha=1/\varepsilon$ and use the clipped loss estimator 
$\hatellta=\tildeellta\inlineindicator{|\tildeellta|\le s_{t,a}}$ to prevent the loss estimator from being prohibitively negative. Particularly, for FTRL with $\alpha$-Tsallis entropy and learning rate $\eta_t$ in round $t$, it (approximately) requires $\eta_t \hatellta\gtrsim -p_{t,a}^{\alpha-1}$. Thus they set the clipping threshold $s_{t,a}=\Theta(\eta_t^{-1} p_{t,a}^{\alpha-1})=\Theta(\eta_t^{-1} p_{t,a}^{\nicefrac{1}{\varepsilon}-1})$. In theoretical analysis, this enables the following regret decomposition:
\begin{align}\label{eq:previous_app}
    \Reg&=\E\left[\sum_{t=1}^T \left\langle p_t-p^*, \ell_t \right\rangle\right] \notag\\
    &=\E\left[\sum_{t=1}^T \left\langle p_t-p^*, \hatell_t \right\rangle\right]+\mathbb{E}\left[\sum_{t=1}^T \left\langle p_t-p^*, \tildeell_t-\hat{\ell}_t \right\rangle\right] \notag\\
    &=\E\left[\sum_{t=1}^T \left\langle p_t-p^*, \hatell_t \right\rangle\right]
    +\mathbb{E}\left[\sum_{t=1}^T \left\langle p_t, \tildeell_t-\hat{\ell}_t \right\rangle\right]
    +\mathbb{E}\left[\sum_{t=1}^T \hat{\ell}_{t,a^*}-\tildeell_{t,a^*}\right]\,,
\end{align}
where the second line follows from $\tildeell_t$ is an unbiased estimator of $\ell_t$. The first term in Eq.~\eqref{eq:previous_app} can be bounded by normal FTRL analysis, and the second term can also be bounded by controlling the bias via the clipping threshold $s_{t,a}$. However, in this way, the third term in Eq.~\eqref{eq:previous_app} can only be upper bounded by $\gO(\eta_t^{\varepsilon-1}p_{t,a^*}^{(1-\varepsilon)/\varepsilon})$, which might be prohibitively large as $p_{t,a^*}$ might approach zero.
Nevertheless, the truncated non-negativity assumption, requiring $\mathbb{E}[\ell_{t,a^*} \cdot \mathds{1}[|\ell_{t,a^*}|>M]] \geq 0$ for any $M \ge 0$ \citep{Huang0H22,Chen2024uniINF}, ensures the non-positivity of the third term in Eq.~\eqref{eq:previous_app}. Consequently, this term can be safely omitted from the analysis with the help of the truncated non-negativity assumption.

As noted by \citet{ChengZ024}, this assumption may be difficult to verify in practice. Further, it may be unnatural in the linear setting, as it implicitly imposes additional constraints on the unknown loss vector $\theta_t$ and the feature vector $\phi_{a^*}$ of the optimal arm.
More importantly, even under this assumption, it remains infeasible to obtain the desired regret guarantees for adversarial heavy-tailed linear bandits. Specifically, in the linear setting, the unbiased loss estimator is typically chosen as $\tilde{\ell}_{t,a} = \phi_a^\top S_t^{-1} \phi_{a_t} \ell_{t,a_t}$, where $S_t = \sum_{a \in \gA} p_{t,a} \phi_a \phi_a^\top$ denotes the feature covariance matrix. Analogous to the MAB case, a clipped estimator can be constructed as $\hat{\ell}_{t,a} = \tilde{\ell}_{t,a} \cdot \mathds{1}\{|\tilde{\ell}_{t,a}| \le s_{t,a}\}$. However, the bias term on the optimal arm $a^*$, corresponding to the third term in Eq.~\eqref{eq:previous_app}, now becomes
\begin{align*}
    \mathbb{E}[\hat{\ell}_{t,a^*} - \tilde{\ell}_{t,a^*}]
    = \mathbb{E}\left[-\tilde{\ell}_{t,a^*} \cdot \mathds{1}\{|\tilde{\ell}_{t,a^*}| > s_{t,a^*}\}\right]
    = \mathbb{E}\left[-\phi_{a^*}^\top S_t^{-1} \phi_{a_t} \ell_{t,a_t} \cdot \mathds{1}\{|\tilde{\ell}_{t,a^*}| > s_{t,a^*}\}\right]\,.
\end{align*}
Unlike in the MAB setting, the truncated non-negativity assumption no longer ensures that this bias term is non-positive. This is due to two key differences: (a) the quantity $-\phi_{a^*}^\top S_t^{-1} \phi_{a_t}$ is not necessarily non-negative for general feature vectors of arms; and (b) the regression target in $\tilde{\ell}_{t,a}$ is now $\ell_{t,a_t}$ rather than $\ell_{t,a^*}$. The latter issue does not arise in the MAB case as $\tilde{\ell}_{t,a^*} = \frac{\ell_{t,a_t} \cdot \mathds{1}\{a_t = a^*\}}{p_{t,a^*}} = \frac{\ell_{t,a^*} \cdot \mathds{1}\{a_t = a^*\}}{p_{t,a^*}}$.

On the other hand, to eliminate the need for the truncated non-negativity assumption, \citet{ChengZ024} resort to the FTRL algorithm with a log-barrier regularizer. This choice enables more stable algorithm updates and relaxes the stability condition to $\eta_t \hat{\ell}_{t,a} \gtrsim -p_{t,a}^{-1}$, which can be satisfied by setting a larger clipping threshold $s_{t,a} = \Theta(\eta_t^{-1} p_{t,a}^{-1})$. As a result, the bias term on the optimal arm $a^*$ can be well controlled without the truncated non-negativity assumption. However, while this approach yields a (nearly) minimax-optimal regret in adversarial heavy-tailed MABs, the use of the log-barrier regularizer introduces a penalty term that grows polynomially with $K$, thereby preventing efficient learning in the linear bandit setting. Furthermore, the BOBW algorithm proposed by \citet{ChengZ024} relies on the detect-and-switch scheme, which requires statistical tests to distinguish between regimes and is thus less adaptive compared to fully FTRL-based BOBW approaches.

\subsection{A General Framework for Adversarial Heavy-tailed Bandits}\label{sec:framework}
To address the above challenges, we propose a general framework for adversarial heavy-tailed bandits that enables the use of FTRL with a general $\alpha$-Tsallis entropy regularizer, without relying on the truncated non-negativity assumption on the loss sequence of the optimal arm.

In particular, we consider running FTRL over the loss functions shifted by a bonus function as follows:
\begin{align}
    p_t\in\argmin_{p \in \gP(\gA)}\left\langle\sum_{s=1}^{t-1} \hat{\ell}_s-b_s, p\right\rangle+\beta_t \psi(p)\,,\notag
\end{align}
where $b_s\colon \gA \to \sR_{\ge 0}^K$ denotes the bonus function in round $s$, $\beta_t = 1/\eta_t$ is the inverse learning rate, and $\psi(p) = -\frac{1}{\alpha} \sum_{a \in \gA} \left(p_a^\alpha - p_a\right)$ is the $\alpha$-Tsallis entropy regularizer with $\alpha \in (0,1)$.

Intuitively, shifting the loss estimates by a positive bonus function $b_t$ introduces a negative part to the bias term of the optimal arm $a^*$, and can effectively eliminate this bias if $b_{t,a^*}$ serves as an upper bound of the bias of the clipped loss estimate. To see this, incorporating the bonus function into the FTRL updates permits the following regret decomposition:
\begin{align}\label{eq:bonus_explain}
    \Reg&=\E\left[\sum_{t=1}^T \left\langle p_t-p^*, \ell_t \right\rangle\right] \notag\\
    &=\E\left[\sum_{t=1}^T \left\langle p_t-p^*, \hatell_t-b_t \right\rangle\right]+\mathbb{E}\left[\sum_{t=1}^T \left\langle p_t-p^*, \tildeell_t-\hat{\ell}_t+b_t \right\rangle\right] \notag\\
    &\le \E\left[\sum_{t=1}^T \left\langle p_t-p^*, \hatell_t-b_t \right\rangle\right]
    +\mathbb{E}\left[\sum_{t=1}^T \sum_{a\in\gA}\abs{p_{t,a}-p^*_a} \abs{\tildeell_{t,a}}\indicator{|\tildeell_{t,a}|> s_{t,a}}\right] \notag\\
    &\quad+\mathbb{E}\left[\sum_{t=1}^T \left\langle p_t-p^*, b_t \right\rangle\right] \notag\\
    &=\E\left[\sum_{t=1}^T \left\langle p_t-p^*, \hatell_t-b_t \right\rangle\right]
    +\mathbb{E}\left[\sum_{t=1}^T \sum_{a\in\gA\setminus \{a^*\}}p_{t,a} \abs{\tildeell_{t,a}}\indicator{|\tildeell_{t,a}|> s_{t,a}}\right] \notag\\
    &\quad+\mathbb{E}\left[\sum_{t=1}^T\sum_{a\in\gA\setminus \{a^*\}}p_{t,a}b_{t,a}  \right]+\mathbb{E}\left[\sum_{t=1}^T (1-p_{t,a^*})\left(\abs{\tildeell_{t,a^*}}\indicator{|\tildeell_{t,a^*}|> s_{t,a^*}}-b_{t,a^*}\right)\right]\,.
\end{align}
Therefore, as long as we can ensure that $b_{t,a} \ge \E_{t-1}\left[\abs{\tildeell_{t,a}} \inlineindicator{|\tildeell_{t,a}| > s_{t,a}}\right]$, the last term in Eq.~\eqref{eq:bonus_explain} will be non-positive and can thus be omitted in the analysis.

The caveat of incorporating bonus functions into the FTRL updates is, of course, that the regret incurred by competing with $a^*$ on the shifted loss sequence $\{\hatell_t - b_t\}_{t=1}^T$ using FTRL (\textit{i.e.}, the first term in Eq.~\eqref{eq:bonus_explain}) will be enlarged, as the stability term of FTRL now scales with $\sum_{a \in \gA} (\hatell_{t,a}^2 + b_{t,a}^2)$ rather than the usual $\sum_{a \in \gA} \hatell_{t,a}^2$. Moreover, since $b_t \ge 0$, the shifted loss estimate $\hatell_t - b_t$ becomes more negative, which may result in unstable updates of the algorithm. Nevertheless, we note that these issues can be mitigated by appropriately selecting the clipping threshold $s_{t,a}$ and incorporating a certain amount of additional exploration to prevent $b_{t,a}$ from becoming excessively large.

We will shortly demonstrate that applying our framework to both heavy-tailed MABs and linear bandits yields (nearly) minimax-optimal regret guarantees in adversarial regimes, and even achieves BOBW regret guarantees.

\section{Best-of-both-worlds Heavy-tailed Multi-Armed Bandits}\label{sec:bobw_mab}
In heavy-tailed MABs, the loss functions do not necessarily satisfy the linear structure $\E_{\ell_{t,a} \sim \nu_{t,a}}[\ell_{t,a}] = \langle \phi_a, \theta_t \rangle$, but they still adhere to the heavy-tailed loss distribution conditions defined in \Cref{def:HT}. In this section, we present the first FTRL-based BOBW algorithm as well as the analysis for heavy-tailed MABs that does not rely on the truncated non-negativity assumption, building upon the framework introduced in \Cref{sec:framework}.

\begin{algorithm}[!t]
\caption{Best-of-both-worlds Algorithm for Heavy-tailed MABs}\label{algo:algo1}
\begin{algorithmic}[1]
   \REQUIRE Number of rounds $T$, set of arms $\gA$, $0 < \alpha<1$, $\varepsilon\in(1,2]$ and $\sigma>0$.
    \STATE Set $\psi(q)=-\frac{1}{\alpha} \sum_{a \in \gA}\left(q_a^\alpha-q_a\right)$.
   \FOR{$t=1$ {\bfseries to} $T$}
        \STATE Set $\beta_t=\sigma^{\frac {1}{\varepsilon}}\max\{8\varepsilon K^{\frac{\varepsilon-1}{\varepsilon}}/(\varepsilon-1),t^{\frac{1}{\varepsilon}}\}$. \label{alg:line:tabular_beta1}
      \STATE Compute $q_t\in\gP(\gA)$ by 
      \begin{align}
        q_t\in\argmin_{q \in \gP(\gA)}\left\langle\sum_{s=1}^{t-1} \hat{\ell}_s-b_s, q\right\rangle+\beta_t \psi(q)\,.
    \end{align}\label{algo:tabular:line:q_t}
    \STATE Set $s_{t, a}=(1-\alpha)\tilde{q}_{t,a}^{\alpha-1}\beta_t/8$ for all $a\in\gA$.\label{algo:tabular:line:s_ta}
    \STATE Compute $p_t=(1-\gamma_t) q_t+\gamma_t p_0$, where $p_{0}$ is a uniform distribution over $\gA$ and $\gamma_t=\sigma^{\frac{1}{\varepsilon-1}} K s_{t, \tilde{a}_t}^{\varepsilon/(1-\varepsilon)}$.\label{algo:tabular:line:p_t}
    \STATE Sample arm $a_t\sim p_t$ and observe the incurred loss $\ell_{t,a_t}$.\label{algo:tabular:line:sample}
    \STATE Set 
    $\tilde{\ell}_{t,a}=\frac{\ell_{t, a} \indicator{a_t=a}}{p_{t, a}}$, 
    $\hat{\ell}_{t, a}=\tilde{\ell}_{t, a} \inlineindicator{|\tilde{\ell}_{t, a}| \le s_{t, a}}$, 
    $b_{t,a}=\sigma p_{t, a}^{1-\varepsilon} s_{t, a}^{1-\varepsilon}$ for all $a\in\gA$.\label{algo:tabular:line:loss_estimate}
    \ENDFOR
\end{algorithmic}
\end{algorithm}

\subsection{Algorithm}
The pseudocode of our BOBW algorithm for heavy-tailed MABs is presented in \Cref{algo:algo1}.  
At the commencement of each round $t \in [T]$, the algorithm first sets the (inverse) learning rate as 
$\beta_t=\sigma^{\frac{1}{\varepsilon}}\max\{8\varepsilon K^{\frac{\varepsilon-1}{\varepsilon}}/(\varepsilon-1), t^{\frac{1}{\varepsilon}}\}$, where $\beta_t$ is slightly inflated for small $t$ to ensure that the exploration strength $\gamma_t$ is small enough across all rounds $t\in[T]$. 
The algorithm then computes the FTRL update $q_t$ over the shifted loss estimates $\{\hat{\ell}_s - b_s\}_{s=1}^{t-1}$ (Line~\ref{algo:tabular:line:q_t}).
On Lines~\ref{algo:tabular:line:s_ta}--\ref{algo:tabular:line:p_t}, the algorithm sets the clipping threshold $s_{t,a}$ and mixes $q_t$ with the uniform distribution $p_0$ to conduct additional exploration. In particular, to ensure that the clipped loss estimates $\hatellta$ and the bonus function $b_{t,a}$ are not too large to destabilize the FTRL update, the clipping threshold $s_{t,a}$ and the additional exploration strength $\gamma_t$ are carefully designed as
\begin{align}
    s_{t, a} = (1 - \alpha) \tilde{q}_{t,a}^{\alpha - 1} \beta_t/8 \,, \quad 
    \gamma_t =  \sigma^{\frac{1}{\varepsilon - 1}} K s_{t, \tilde{a}_t}^{\varepsilon / (1 - \varepsilon)}\,,\notag
\end{align}
where $\tilde{q}_{t,a} = \min\{q_{t,a}, q_{t*}\}$, $q_{t*} = \min\left\{\|q_t\|_{\infty}, 1 - \|q_t\|_{\infty}\right\}$, and $\tilde{a}_t \in \arg\max_{a \in \gA} q_{t,a}$.
The dependence on $\tilde{q}_{t,a}$ rather than $q_{t,a}$ in defining $s_{t,a}$ servers as a key ingredient that eliminates the dependence on $a^*$ in the FTRL stability term, which is essential for achieving the desired BOBW regret guarantee. The learner then samples an arm $a_t \sim p_t$ to interact with the environment and observes the incurred loss $\ell_{t,a_t}$ (Line~\ref{algo:tabular:line:sample}). Finally, round~$t$ concludes with the construction of the clipped loss estimate $\hat{\ell}_{t,a}$ and the bonus function $b_{t,a}$ (Line~\ref{algo:tabular:line:loss_estimate}). In particular, the bonus is defined as $b_{t,a} =\sigma p_{t,a}^{1 - \varepsilon} s_{t,a}^{1 - \varepsilon}$ for all $a \in \gA$, which is moderately large to eliminate the loss bias of the optimal arm while still enabling a BOBW regret guarantee.

\subsection{Analysis}
The regret of \Cref{algo:algo1} is guaranteed by the following theorem.
\begin{thm}\label{thm:thm1}
By setting $\alpha=1/\varepsilon$, for heavy-tailed MABs in adversarial regimes,  \Cref{algo:algo1} with  achieves
    \begin{align}
    \operatorname{Reg}_T=\gO\left(\sigma^{\frac{1}{\varepsilon}}K^{\frac{\varepsilon-1}{\varepsilon}}T^{\frac{1}{\varepsilon}}+\sigma^{\frac{1}{\varepsilon}}K\iota(\varepsilon)+\kappa\right)\,,\notag
\end{align}
where $\kappa=8\varepsilon^2\sigma^{\frac{1}{\varepsilon}}K^{\frac{2(\varepsilon-1)}{\varepsilon}}/(\varepsilon-1)$, $\iota(\varepsilon)=\log T$ if $\varepsilon=2$ and $\iota(\varepsilon)=\left(\frac{8\varepsilon}{\varepsilon-1}\right)^{\frac{\varepsilon}{\varepsilon-1}}\frac{\varepsilon-1}{2-\varepsilon}$ if $\varepsilon\in (1,2)$.
Further, in adversarial regimes with a $(\Delta, C, T)$-self-bounding constraint, \Cref{algo:algo1} achieves
    \begin{align}
        \operatorname{Reg}_T=\gO\left(\left(\nicefrac{2\sigma}{\varepsilon}\right)^{\frac{1}{\varepsilon-1}}
    \omega(\Delta)^{\frac{\varepsilon}{\varepsilon-1}} \log T
    + \omega(\Delta) \left(\log T\right)^{1-\frac{1}{\varepsilon}}\left(\nicefrac{2\sigma}{\varepsilon}\right)^{\frac{1}{\varepsilon}}C^{\frac{1}{\varepsilon}}
    +\sigma^{\frac{1}{\varepsilon}}K\iota(\varepsilon)+\kappa
    \right)\notag\,,
    \end{align}
where $\omega(\Delta)= \left(\sum_{a \neq a^*}\Delta_a^{1/(1-\varepsilon)} \right)^{(\varepsilon-1)/\varepsilon}$.
\end{thm}

The proof of \Cref{thm:thm1} is deferred to Appendix~\ref{sec:app:mab}. Ignoring lower-order terms, our \Cref{algo:algo1} guarantees a regret of $\gO(\sigma^{\frac{1}{\varepsilon}} K^{1 - \frac{1}{\varepsilon}} T^{\frac{1}{\varepsilon}})$ in adversarial regimes, and a regret of $\gO(\sum_{a \ne a^*} (\frac{\sigma}{\Delta_a})^{\frac{1}{\varepsilon - 1}} \log T)$ in stochastic regimes (\textit{i.e.}, when $C = 0$). These bounds match the $\Omega(\sigma^{\frac{1}{\varepsilon}} K^{1 - \frac{1}{\varepsilon}} T^{\frac{1}{\varepsilon}})$ worst-case lower bound and the $\Omega(\sum_{a \ne a^*} (\frac{\sigma}{\Delta_a})^{\frac{1}{\varepsilon - 1}} \log T)$ gap-dependent lower bound of \citet{Bubeck2012BanditsWH}. 
Furthermore, our result improves upon the best-known BOBW regret bound without the truncated non-negativity assumption, achieved by the SAO-HT algorithm \citep{ChengZ024}, by a factor of $\gO(\log K \cdot \log^4 T)$ in adversarial regimes and $\gO(\log K \cdot \log^3 T)$ in stochastic regimes.
Moreover, compared to SAO-HT \citep{ChengZ024}, our algorithm can also guarantee an $\gO(\log T + C^{\frac{1}{\varepsilon}})$-type regret in the intermediate setting (\textit{e.g.}, stochastic regimes with adversarial corruptions), when the total amount of corruption is bounded by $C$.

\section{Adversarial Heavy-tailed Linear Bandits}\label{sec:adv_htlb}
In this section, we introduce the proposed algorithm for learning adversarial heavy-tailed linear bandits with fixed arm sets in \Cref{sec:adv_htlb_algo} and the analysis in \Cref{sec:adv_htlb_analysis}.
\subsection{Algorithm}\label{sec:adv_htlb_algo}
\begin{algorithm}[!t]
\caption{Algorithm for Adversarial Heavy-tailed Linear Bandits}\label{algo:algo2}
\begin{algorithmic}[1]
   \REQUIRE Number of rounds $T$, set of arms $\gA$, $\varepsilon\in (1,2]$ and $\sigma>0$.
    \STATE Set $\psi(q)=\sum_{a\in\gA} q_a\log q_a$, $\beta=\left(\frac{\log K}{\sigma d^{\varepsilon/2}T}\right)^{-1/\varepsilon}$, $\gamma=4 \sigma^{\frac{2}{\varepsilon}} d \beta^{-2}$, and $s_{t,a}=\beta/2$ for all $a\in\gA$ and $t\in[T]$.
   \FOR{$t=1$ {\bfseries to} $T$}
      \STATE Compute $q_t\in\gP(\gA)$ by 
      \begin{align}
        q_t\in\argmin_{q \in \gP(\gA)}\left\langle\sum_{s=1}^{t-1} \hat{\ell}_s-b_s, q\right\rangle+\beta \psi(q)\,.
    \end{align}\label{algo:line:q_t}
    \STATE Compute $p_t=(1-\gamma) q_t+\gamma p_0$, where $p_{0}$ is a G-optimal design distribution over $\{\phi_a\}_{a\in\gA}$.\label{algo:line:p_t}
    \STATE Sample arm $a_t\sim p_t$ and observe the incurred loss $\ell_{t,a_t}$.\label{algo:line:sample}
    \STATE Set $\tilde{\ell}_{t,a}=\phi_a^{\top}S_t^{-1}\phi_{a_t}\ell_{t, a_t}$, $\hat{\ell}_{t, a}=\tilde{\ell}_{t, a} \inlineindicator{|\tilde{\ell}_{t, a}| \le s_{t, a}}$, $b_{t,a}=\sigma s_{t, a}^{1-\varepsilon} (\phia^{\top} S_t^{-1} \phia)^{\varepsilon/2}$ for all $a\in\gA$, where $S_t=\sum_{a\in\gA} p_{t,a}\phia\phia^{\top}$.\label{algo:line:loss_estimate}
    \ENDFOR
\end{algorithmic}
\end{algorithm}
The pseudocode for the algorithm that learns adversarial heavy-tailed linear bandits is detailed in \Cref{algo:algo2}. This algorithm also instantiates the general framework introduced in \Cref{sec:framework}.  
To learn the heavy-tailed linear bandits in the adversarial case alone, it suffices to use a (negative) Shannon entropy regularizer $\psi(q) = \sum_{a \in \gA} q_a \log q_a$ along with time-invariant choices for the learning rate, clipping threshold, and exploration strength. Specifically, we set the (inverse) learning rate as $\beta = \left(\frac{\log K}{\sigma d^{\varepsilon / 2} T}\right)^{-1/\varepsilon}$, the clipping threshold as $s_{t,a} = \beta/2$ for all $a\in\gA$, and the exploration strength as $\gamma = 4 \sigma^{\frac{2}{\varepsilon}} d \beta^{-2}$.
Compared to \Cref{algo:algo1}, other differences are that we use loss estimators $\hat{\ell}_{t, a} = \tilde{\ell}_{t, a} \inlineindicator{|\tilde{\ell}_{t, a}| \le s_{t, a}}$, where $\tilde{\ell}_{t,a} = \phi_a^{\top} S_t^{-1} \phi_{a_t} \ell_{t, a_t}$ and $S_t=\sum_{a\in\gA} p_{t,a}\phia\phia^{\top}$, and bonus functions $b_{t,a} = \sigma s_{t, a}^{1 - \varepsilon} (\phi_a^{\top} S_t^{-1} \phi_a)^{\varepsilon / 2}$, both tailored to the linear setting.

\subsection{Analysis}\label{sec:adv_htlb_analysis}
The following theorem establishes the regret guarantee of \Cref{algo:algo2} for adversarial heavy-tailed linear bandits with fixed arm sets.
\begin{thm}\label{thm:thm2}
 For adversarial heavy-tailed linear bandits, \Cref{algo:algo2} achieves
\begin{align}
    \operatorname{Reg}_T=\gO\left(\sigma^{\frac{1}{\varepsilon}} (\log K)^{1-\frac{1}{\varepsilon}} d^{\frac{1}{2}} T^{\frac{1}{\varepsilon}}\right)\,.\notag
\end{align}
\end{thm}

We postpone the proof of \Cref{thm:thm2} to Appendix~\ref{sec:app:adv_htlb}.
For the problem of heavy-tailed linear bandits with fixed arm sets in the stochastic regime, \citet{XueWWZ20} propose an algorithm that achieves a regret of $\gO(\sigma^{\frac{1}{\varepsilon}}\sqrt{\log (KT)}\log T\cdot d^{\frac{1}{2}}T^{\frac{1}{\varepsilon}})$, and also establish a lower bound of $\gO(d^{1 - \frac{1}{\varepsilon}} T^{\frac{1}{\varepsilon}})$. More recently, \citet{tajdini2025improved} improve the upper and lower bounds to $\gO(\sigma^{\frac{1}{\varepsilon}} (\log K)^{1 - \frac{1}{\varepsilon}} d^{\frac{1}{2}} T^{\frac{1}{\varepsilon}})$ and $\Omega(\sigma^{\frac{1}{\varepsilon}} (\log K)^{1 - \frac{1}{\varepsilon}} d^{1 - \frac{1}{\varepsilon}} T^{\frac{1}{\varepsilon}})$, respectively, which represent the best-known results to date. Our \Cref{algo:algo2} achieves a regret upper bound that matches that of \citet{tajdini2025improved} in all factors, while additionally offering robustness in the adversarial regime.

\section{Heavy-Tailed Noise Aware Stability-Penalty Matching}
In this section, we first introduce a general data-dependent learning rate rule that extends the stability-penalty matching (SPM) learning rate proposed by \citet{ito2024adaptive} to bandit problems with heavy-tailed noises. Then, in \Cref{sec:bobw_htlb}, we demonstrate that this data-dependent learning rate, combined with carefully designed clipping thresholds, bonus function, and in particular, a new variance-reduced least-squares loss estimator, leads to the BOBW regret guarantee for heavy-tailed linear bandits.

To begin with, note that, similar to the cases in \Cref{sec:bobw_mab} and \Cref{sec:adv_htlb}, applying our bonus-shifted FTRL framework to bandit problems with heavy-tailed noises yields the following regret upper bound:
\begin{align}\label{eq:round_wise}
\Reg \lesssim \E\left[\sumt \left( \left(\beta_t - \beta_{t-1}\right) h_t + \beta_t^{1 - \varepsilon} z_t + \beta_t^{-2} w_t \right)\right]\,,
\end{align}
where $(\beta_t - \beta_{t-1}) h_t > 0$ is an upper bound of the round-wise penalty term, $\beta_t^{1 - \varepsilon} z_t > 0$ is an upper bound of the sum of the round-wise stability and bias terms, and $\beta_t^{-2} w_t = \gO(\sigma^{\frac{1}{\varepsilon}} \gamma_t) > 0$ is an upper bound of the regret incurred by additional exploration in round $t$. For notational convenience, we define $\beta_0 = 0$ and $\beta_{T+1} = \beta_T$.

At a high level, our objective is to determine an update rule for the learning rate of the form $\text{Rule} : (z_{1:t-1}, w_{1:t-1}, h_{1:t-1}) \mapsto \beta_t$ such that the regret bound in Eq.~\eqref{eq:round_wise} is minimized, subject to the constraint that the sequence $(\beta_t)_t$ is non-decreasing. To this end, we consider the following update rule:
\begin{align}\label{eq:beta_t}
    \beta_t = \beta_{t-1} + \frac{1}{\hat{h}_t} \left( \beta_{t-1}^{1 - \varepsilon} z_{t-1} + \beta_{t-1}^{-2} w_{t-1} \right)\,,
\end{align}
where $\hat{h}_t$ denotes an estimate of the true penalty term $h_t$. We refer to this update as the \textit{heavy-tailed noise aware stability-penalty matching} (HT-SPM) learning rate. 
This update rule is designed so that the stability term $\beta_{t-1}^{1 - \varepsilon} z_{t-1}$ (along with the additional exploration regret term $\beta_{t-1}^{-2} w_{t-1}$) from round $t-1$ is matched by the estimated penalty term $(\beta_t - \beta_{t-1}) \hat{h}_t$ in round $t$. This data-dependent learning rate generalizes the SPM rule proposed by \citet{ito2024adaptive} to settings with heavy-tailed noises.
Moreover, when $\varepsilon=3/2$, it also subsumes the \textit{stability–penalty–bias matching} (SPB-matching) learning rate \citep{tsuchiyasimple}, which is designed for achieving BOBW results for bandit problems with $\Theta(T^{2/3})$ minimax regret, as a special case.
To ensure that $\hat{h}_t\coloneqq h_{t-1}$ serves as a good estimate of the true penalty term $h_t$, in addition to the primary $\alpha$-Tsallis entropy regularizer, we incorporate an auxiliary $\bar{\alpha}$-Tsallis entropy regularizer to further stabilize the algorithm \citep{ito2024adaptive}. Specifically, the distribution $q_t$ is obtained by solving the following FTRL update at the end of round $t$:
\begin{align}\label{algo3:line:q_t}
    q_t \in \argmin_{q \in \gP(\gA)} \left\langle \sum_{s=1}^{t-1} \hat{\ell}_s - b_s, q \right\rangle + \beta_t \psi(q) + \bar{\beta} \bar{\psi}(q)\,,
\end{align}
where $\bar{\psi}(q) = -\frac{1}{\bar{\alpha}} \sum_{a \in \gA} \left(q_a^{\bar{\alpha}} - q_a\right)$ with $\bar{\alpha} \in (0,1)$, and $\hat{\ell}_s$ and $b_s$ denote the clipped loss estimate and the bonus function, respectively.

Then, similar to the cases of BOBW heavy-tailed MABs and adversarial heavy-tailed linear bandits, the distribution $q_t$ is mixed with an exploration distribution $p_0$ to form $p_t = (1 - \gamma_t) q_t + \gamma_t p_0$, from which an arm $a_t$ is sampled to interact with the environment. Based on the loss feedback associated with arm $a_t$, along with $p_t$ and the clipping threshold $s_t$, the loss estimator $\hat{\ell}_t$ and the bonus function $b_t$ are constructed. Finally, round $t$ concludes with the update of $\beta_{t+1}$. We summarize the details of this general algorithmic framework for BOBW bandit problems with heavy-tailed noises in \Cref{algo:algo3}.

The following proposition demonstrates that the general algorithmic framework in \Cref{algo:algo3} using the HT-SPM learning rate in Eq.~\eqref{eq:beta_t} yields a BOBW regret guarantee, provided certain conditions are satisfied. The proof is deferred to Appendix~\ref{sec:app:htddlr}.

\begin{prop}\label{prop:prop1}
If \Cref{algo:algo3} satisfies that
    \begin{align}\label{eq:regret_F}
\Reg \lesssim  \E\left[\sumt \left( \left(\beta_t - \beta_{t-1}\right) h_t + \beta_t^{1 - \varepsilon} z_t + \beta_t^{-2} w_t \right)\right]+\bar{\beta} \bar{h}\,,
    \end{align}
where $\bar{h}=-\bar{\psi}\left(q_1\right) \leq \frac{1}{\bar{\alpha}} K^{1-\bar{\alpha}}$,
the learning rate $\beta_t$ is updated according to the HT-SPM rule in Eq.~\eqref{eq:beta_t}, and
    \begin{align}\label{eq:stable_update}
        h_t=\gO\left(h_{t-1}\right)\,,
    \end{align}
then, in the adversarial regime, \Cref{algo:algo3} guarantees that
    \begin{align*}
        \Reg=\gO\left(h_1^{1-\frac{1}{\varepsilon}}\left(z_{\max } T\right)^{\frac{1}{\varepsilon}}+h_1^{\frac{2}{3}}\left(w_{\max} T\right)^{\frac{1}{3}}
        +\kappa\right)\,,
    \end{align*}
    where $z_t \leq z_{\max }$, $w_t \leq w_{\max }$ for all $t\in[T]$, $\kappa=z_{\max } \beta_1^{1-\varepsilon}+w_{\max } \beta_1^{-2}+\beta_1 {h}_1+\bar{\beta} \bar{h}$, and $h_1=-\psi\left(q_1\right) \leq \frac{1}{\alpha} K^{1-\alpha}$.
    In an adversarial regime with a $(\Delta, C, T)$-self-bounding constraint, if 
    \begin{align}\label{eq:round_wise_0}
        {h}_t^{\varepsilon-1} z_t \le \zeta(\Delta) \cdot\left\langle\Delta, p_t\right\rangle \quad\text{and}\quad {h}_t^{2} w_t \le \omega(\Delta) \cdot\left\langle\Delta, p_t\right\rangle
    \end{align}
    hold for some $\zeta(\Delta),\omega(\Delta)>0$, \Cref{algo:algo3} can simultaneously guarantee that
    \begin{align*}
        \Reg=\gO\left( \iota+\iota^{\prime}+C^{\frac{1}{\varepsilon}}(\iota+\iota^{\prime})^{1-\frac{1}{\varepsilon}}+\kappa \right)\,,
    \end{align*}
    where 
    $\iota=\zeta(\Delta)^{\frac{1}{\varepsilon-1}}\left(\frac{\varepsilon}{2}\right)^{\frac{1}{1-\varepsilon}} \log _{+}\left(\frac{z_{\max } h_1 T}{\zeta(\Delta)^{\varepsilon/(\varepsilon-1)}+C}\right) /(\varepsilon-1)$ and 
    $\iota^{\prime}=\omega(\Delta)^{\frac{1}{2}}\left(\frac{3}{2}\right)^{-\frac{1}{2}} \log _{+}\left(\frac{w_{\max} h_1 T}{\omega(\Delta)^{3/2}+C}\right)$.
\end{prop}

\begin{algorithm}[!t]
\caption{General Best-of-both-worlds Algorithm Framework with HT-SPM Learning Rate and Tsallis-entropy Regularizer for Heavy-tailed Bandit Problems}\label{algo:algo3}
\begin{algorithmic}[1]
   \REQUIRE Number of rounds $T$, set of arms $\gA$, $\varepsilon\in (1,2]$, $\sigma>0$, $\alpha,\bar{\alpha}\in(0,1)$, $\beta_1$, $\bar{\beta}$, $p_{0}$.
    \STATE Set $\psi(q)=-\frac{1}{\alpha} \sum_{a \in \gA}\left(q_a^\alpha-q_a\right)$, $\bar{\psi}(q)=-\frac{1}{\balpha} \sum_{a \in \gA}\left(q_a^\balpha-q_a\right)$.
   \FOR{$t=1$ {\bfseries to} $T$}
        \STATE Compute $q_t\in\gP(\gA)$ by Eq. \eqref{algo3:line:q_t}.
        \STATE Set $\gamma_t,s_t$ based on $q_t$ and $\beta_t$.\label{algo3:line:s_ta}
        \STATE Set $p_t=(1-\gamma_t) q_t+\gamma_t p_0$, where $p_0\in \gP(\gA)$ is an exploration distribution.\label{algo3:line:p_t}
        \STATE Sample arm $a_t\sim p_t$ so that $\sP(a_t=a\mid p_t)=p_{t,a}$ and observe the loss feedback from the environment.\label{algo3:line:sample}
        \STATE Construct loss estimator $\hatell_t$ and bonus $b_t$ based on $p_t$, $s_t$, and the loss feedback.\label{algo3:line:loss_estimate}
        \STATE Set $h_t=-\psi\left(q_t\right)$ and $z_t,w_t> 0$ based on $q_t$ and $p_t$.\label{algo:line:zwt}
        \STATE Update $\beta_{t+1}$ by the update rule in Eq. \eqref{eq:beta_t} with $\hath_{t+1}=h_t$.\label{algo3:line:beta_update}
    \ENDFOR
\end{algorithmic}
\end{algorithm}

\subsection{Application: Best-of-Both-Worlds Heavy-tailed Linear Bandits}\label{sec:bobw_htlb}

In this section, we instantiate the general algorithmic framework in \Cref{algo:algo3} with appropriate designs for the clipping threshold and the bonus function, and in particular, introduce a new least-squares loss estimator. Based on this instantiation, we show that the conditions required in \Cref{prop:prop1} are indeed satisfied, thereby establishing the first BOBW regret guarantee for the heavy-tailed linear bandit problem.

\paragraph{Variance-reduced Linear Loss Estimator} One of the key conditions for achieving the BOBW guarantee in \Cref{prop:prop1} is the inequality ${h}_t^{\varepsilon-1} z_t \le \zeta(\Delta) \cdot \left\langle \Delta, p_t \right\rangle$ in Eq.~\eqref{eq:round_wise_0}. In the setting of heavy-tailed linear bandits, after observing the incurred loss $\ell_{t,a_t}$, this condition will fail to hold if we adopt the standard least-squares loss estimator (\textit{e.g.}, $\tilde{\ell}_{t,a} = \phi_a^{\top} S_t^{-1} \phi_{a_t} \ell_{t,a_t}$ as in \Cref{sec:adv_htlb_algo}). Specifically, using this estimator allows us to bound the stability term only as $z_t = \gO((1 - q_{t,a^*})^{(\varepsilon - 1)(1 - \alpha)})$ (ignoring dependence on other parameters). Combined with the fact that $h_t^{\varepsilon - 1} = \frac{1}{\alpha} \left( \sum_{a \in \gA} q_{t,a}^\alpha - 1 \right) = \gO(\left\langle \Delta, q_t \right\rangle^{\alpha(\varepsilon - 1)})$, this only leads to a bound of the form $h_t^{\varepsilon - 1} z_t \le \zeta'(\Delta) \cdot \left\langle \Delta, p_t \right\rangle^{\varepsilon - 1}$ for some function $\zeta'(\Delta)$. However, this is weaker than the desired bound involving $\left\langle \Delta, p_t \right\rangle$.

To address the issue discussed above, we introduce a variance-reduced least-squares loss estimator in the heavy-tailed setting. Specifically, the proposed linear loss estimator is defined as
\begin{align}\label{eq:vr_linear_loss_est}
    \tilde{\ell}_{t,a} = \bar{\phi}_a^\top V_t^{-1} \bar{\phi}_{a_t} \ell_{t,a_t}\,,
\end{align}
which resembles the form of the commonly used estimator in the adversarial linear bandit literature, but is constructed using mean-centered feature vectors $\{ \bar{\phi}_a \}_{a \in \gA}$, where $\bar{\phi}_a = \phi_a - \mu_t$, $\mu_t = \mathbb{E}_{a \sim p_t}[\phi_a]$, and $V_t = \mathbb{E}_{a \sim p_t}[\bar{\phi}_a \bar{\phi}_a^\top]$.\footnote{Note that $\bphia$ actually has dependence on $t$, but the subscript on $t$ of $\bphia$ is omitted in this work when the context is clear.} 
Compared to the conventional linear loss estimator $\tilde{\ell}_{t,a} = \phi_a^\top S_t^{-1} \phi_{a_t} \ell_{t,a_t}$, our variance-reduced least-squares loss estimator is no longer an unbiased estimator of the true loss. Nevertheless, it still satisfies that $\mathbb{E}_{t-1}[\tilde{\ell}_{t,a} - \tilde{\ell}_{t,b}] = \ell_{t,a} - \ell_{t,b}$ for any $a,b\in\gA$.
More importantly, our loss estimator offers a key advantage in exhibiting lower variance (see \Cref{lem:loss_var} for details).
This improvement enables a tighter bound on the stability term, specifically $z_t = \gO\left((1 - q_{t,a^*})^{(\varepsilon - 1)(1 - \alpha) + 2 - \varepsilon} \right)$.
As a result, we are able to satisfy the desired condition ${h}_t^{\varepsilon - 1} z_t \le \zeta(\Delta) \cdot \left\langle \Delta, p_t \right\rangle$ for some function $\zeta(\Delta)$, which serves as one of the key ingredients in establishing the final BOBW regret guarantee.

\paragraph{Exploration Distribution}
In Line~\ref{algo3:line:p_t}, it is required to mix $q_t$ with an exploration distribution $p_0$ to prevent the bonus function $b_t$ from becoming prohibitively large
(the construction of $b_t$ will be detailed later). However, with the leverage of our variance-reduced loss estimator, simply mixing $q_t$ with the vanilla G-optimal design distribution is no longer sufficient to achieve this goal, unlike in \Cref{algo:algo2}. 
This issue arises because
$V_t = (1 - \gamma_t)\E_{a \sim q_t}[\phia\phia^{\top}] + \gamma_t\E_{a \sim p_0}[\phia\phia^{\top}] - \mu_t\mu_t^{\top}$,
and thus we cannot guarantee that $V_t \succeq \gamma_t \E_{a \sim p_0}[\phia\phia^{\top}]$ for any $p_0\in\gP(\gA)$.
To address this, we instead consider selecting $p_0$ as the ``optimal design distribution'' defined by
\begin{align}\label{eq:barp0}
    p_0 \in \argmax_{p \in \gP(\gA)} \log \det (V(p))\,,
\end{align}
where $V(p) \coloneqq \E_{a \sim p}[(\phia - \mu(p))(\phia - \mu(p))^{\top}]$ and $\mu(p) \coloneqq \E_{a \sim p}[\phia]$.
By mixing $q_t$ with the $p_0$ defined above, it is not hard to see that $V_t \succeq \gamma_t V(p_0)$. More importantly, in \Cref{lem:barp0_exp}, we show that this ``optimal design distribution'' $p_0$ shares a similar property with the G-optimal design distribution in the sense that $\|\phi_a - \mu(p_0)\|_{V(p_0)^{-1}} \le \sqrt{d}$ holds for all $a \in \gA$, thereby providing a desired upper bound on the bonus function.

\paragraph{Parameter Setups}
In Line~\ref{algo3:line:s_ta} of \Cref{algo:algo3}, we set the exploration parameter $\gamma_t$ and the clipping threshold $s_{t,a}$ for all $a \in \gA$  as follows:
\begin{align}\label{eq:bobw_htlb:gamma_s}
    \gamma_t = 256(1 - \alpha)^{-2} \sigma^{\frac{2}{\varepsilon}} d \beta_t^{-2} q_{t*}^{2(1 - \alpha)}\,, \quad
    s_{t,a} = (1 - \alpha)\beta_t q_{t*}^{\alpha - 1}/8\,,
\end{align}
where recall $q_{t*} = \min\left\{\|q_t\|_{\infty}, 1 - \|q_t\|_{\infty}\right\}$.
Note that $\gamma_t \le 1/2$ for all $t\in[T]$ can be satisfied by setting a sufficiently large $\beta_1$.
In Line~\ref{algo3:line:loss_estimate}, the clipped loss estimator is defined as $\hat{\ell}_{t, a} = \tilde{\ell}_{t, a} \inlineindicator{|\tilde{\ell}_{t, a}| \le s_{t, a}}$, where $\tilde{\ell}_{t,a}$ is the variance-reduced linear loss estimator from Eq.~\eqref{eq:vr_linear_loss_est}. Also in Line~\ref{algo3:line:loss_estimate}, the bonus function $b_t$ is set as
\begin{align}\label{eq:bobw_htlb_bonus}
    b_{t,a} = \sigma \sum_{b \in \gA} p_{t,b} \left| \bar{\phi}_a^{\top} V_t^{-1} \bar{\phi}_b \right|^{\varepsilon} s_{t,a}^{1 - \varepsilon}\,.
\end{align}
Note that this bonus function slightly differs from the one in \Cref{sec:adv_htlb_algo}, and this distinction is also crucial for eliminating the dependence on $a^*$ in the stability term. 
At the end of round $t$, we define
\begin{align}\label{eq:bobw_htlb:zwt}
    z_t = (1 - \alpha)^{1 - \varepsilon} \sigma q_{t*}^{(\varepsilon - 1)(1 - \alpha)} d^{\frac{\varepsilon}{2}} (1 - \|p_t\|_{\infty})^{2 - \varepsilon}\,, \quad
    w_t = \sigma^{\frac{3}{\varepsilon}} (1 - \alpha)^{-2}  d q_{t*}^{2(1 - \alpha)}\,,
\end{align}
which, together with $h_t$, are used to update $\beta_{t+1}$ in Line~\ref{algo3:line:beta_update}.

\paragraph{Analysis}
By verifying the conditions in \Cref{prop:prop1}, we obtain the following theorem, which guarantees the BOBW regret guarantee by instantiating \Cref{algo:algo3} with the aforementioned technical components.
\begin{thm}\label{thm:bobw_lb}
Fix arbitrary $\alpha\ge 1/2$.
For heavy-tailed linear bandits, by setting 
$\tilde{\ell}_t$ in Eq.~\eqref{eq:vr_linear_loss_est}, $p_0$ in Eq.~\eqref{eq:barp0},
$\gamma_t$ and $s_t$ in Eq.~\eqref{eq:bobw_htlb:gamma_s}, $b_t$ in Eq.~\eqref{eq:bobw_htlb_bonus}, $z_t$ and $w_t$ in Eq.~\eqref{eq:bobw_htlb:zwt}, 
$\bar{\alpha}=(\varepsilon-1)(1-\alpha)$, and 
$\bar{\beta}\ge 64 (1-\alpha)^{-3} d^{\varepsilon}\beta_{1}^{1-\varepsilon}\max\{\sigma^{\frac{3}{\varepsilon}},\sigma\}$, \Cref{algo:algo3} achieves BOBW regret bound in \Cref{prop:prop1} with 
    $h_1^{1-\frac{1}{\varepsilon}}z_{\max }^{\frac{1}{\varepsilon}}\le (\alpha(1-\alpha))^{\frac{1}{\varepsilon}-1}\sigma^{\frac{1}{\varepsilon}} d^{\frac{1}{2}} K^{(1-\alpha)(1-\frac{1}{\varepsilon})}$, 
    $h_1^{\frac{2}{3}} (w_{\max})^{\frac{1}{3}}\le (\alpha(1-\alpha))^{-\frac{2}{3}}\sigma^{\frac{1}{\varepsilon}}d^{\frac{1}{3}} K^{\frac{2(1-\alpha)}{3}}$, 
\begin{align*}
    \zeta(\Delta)=2^{\varepsilon-1}\left((1-\alpha)\alpha\right)^{1-\varepsilon}\sigma\left(\sum_{a \neq a^*} \Delta_a^{\frac{\alpha}{\alpha-1}}\right)^{(1-\alpha)(\varepsilon-1)}\Delta_{\min}^{\alpha(\varepsilon-1)-1}d^{\frac{\varepsilon}{2}}\,,
\end{align*}
and
\begin{align*}
    \omega(\Delta)=\left(\alpha(1-\alpha)\right)^{-2}\sigma^{\frac{3}{\varepsilon}}\left(\sum_{a \neq a^*} \Delta_a^{\frac{\alpha}{\alpha-1}}\right)^{2(1-\alpha)}d\,.
\end{align*}
\end{thm}
In particular, by choosing $\alpha$ sufficiently close to $1$, we obtain the following corollary.
\begin{coro}\label{coro:coro1}
For heavy-tailed linear bandits, by setting $\alpha=\Theta(1-\nicefrac{1}{\log K})$ and all other parameters the same as in \Cref{thm:bobw_lb}, we have
    $h_1 z_{\max }\lesssim (\log K)^{\varepsilon-1} d^{\frac{\varepsilon}{2}}$, 
    $h_1 w_{\max }\le d\log^2 K $,
    $h_1^{1-\frac{1}{\varepsilon}}z_{\max }^{\frac{1}{\varepsilon}}\le (\log K)^{1-\frac{1}{\varepsilon}}\sigma^{\frac{1}{\varepsilon}} d^{\frac{1}{2}}$,
    $h_1^{\frac{2}{3}} (w_{\max})^{\frac{1}{3}}\le 
    (\log K)^{\frac{2}{3}}\sigma^{\frac{1}{\varepsilon}} d^{\frac{1}{3}}$,
    $\zeta(\Delta)\lesssim (\log K)^{\varepsilon-1}\sigma d^{\frac{\varepsilon}{2}}/\Delta_{\min}$, and $\omega(\Delta)\lesssim \sigma^{\frac{3}{\varepsilon}}\left(\log K/\Delta_{\min}\right)^2$.
    Therefore, in the adversarial regime, it holds that
    \begin{align*}
        \Reg=
        \gO\left(
        (\log K)^{1-\frac{1}{\varepsilon}}
        \sigma^{\frac{1}{\varepsilon}} d^{\frac{1}{2}} T^{\frac{1}{\varepsilon}}
        +(\log K)^{\frac{2}{3}}\sigma^{\frac{1}{\varepsilon}} d^{\frac{1}{3}} T^\frac{1}{3}
        +\kappa 
        \right)\,,
    \end{align*}
    where $\kappa=z_{\max } \beta_1^{1-\varepsilon}+w_{\max } \beta_1^{-2}+\beta_1 h_1+\bar{\beta} \bar{h}$.
    In an adversarial regime with $(\Delta, C, T)$-self-bounding constraint, we further have
    \begin{align*}
        \Reg=\gO\left(\iota+\iota^{\prime}+C^{\frac{1}{\varepsilon}}(\iota+\iota^{\prime})^{1-\frac{1}{\varepsilon}}+\kappa \right)\,,
    \end{align*}
    where 
    $\iota=(2\sigma/\varepsilon)^{\frac{1}{\varepsilon-1}} d^{\frac{\varepsilon}{2(\varepsilon-1)}}\frac{\log K}{\Delta_{\min}^{1/(\varepsilon-1)}} \log _{+}\left(\frac{ d^{\varepsilon/2} T (\log K)^{\varepsilon-1}}{d^{\varepsilon^2/(2(\varepsilon-1))}\log^{\varepsilon} K/\Delta_{\min}^{\varepsilon/(\varepsilon-1)}+C}\right) /(\varepsilon-1)$ 
    and \\
    $\iota^{\prime}=(3/2)^{-\frac{1}{2}}\sigma^{\frac{3}{2\varepsilon}}\frac{\log K}{\Delta_{\min}} \log _{+}\left(\frac{dT\log^2 K}{(\log K/\Delta_{\min})^3+C}\right)$.
\end{coro}
Compared to \Cref{algo:algo2}, \Cref{algo:algo3}, when instantiated with appropriate parameter setups, achieves the same $\gO(\sigma^{\frac{1}{\varepsilon}} (\log K)^{1 - \frac{1}{\varepsilon}} d^{\frac{1}{2}} T^{\frac{1}{\varepsilon}})$ worst-case regret, while additionally enjoying an $\gO(\log T+C^{\frac{1}{\varepsilon}})$-type regret in adversarial regimes with a $(\Delta, C, T)$-self-bounding constraint.

\section{Conclusions and Future Works}
We propose a general bonus-enhanced FTRL framework for adversarial bandits with heavy-tailed noises, removing the commonly assumed truncated non-negativity condition. Instantiating this framework with Tsallis entropy and tailored bonus functions yields the first purely FTRL-based algorithm for heavy-tailed MABs with BOBW regret guarantees in this setting. Using Shannon entropy, we further derive the first algorithm for adversarial heavy-tailed linear bandits with a fixed arm set, matching the best-known regret bound in stochastic regimes. Additionally, we introduce HT-SPM, a data-dependent learning rate that ensures BOBW regret for general heavy-tailed bandit problems, and a variance-reduced linear loss estimator that, when combined with HT-SPM, provides the first BOBW guarantee for heavy-tailed linear bandits. 
We believe that our results advance the understanding of adversarial and BOBW bandit problems under heavy-tailed noises, and that the techniques developed in this work may be of independent interest for other heavy-tailed bandit optimization problems with structured loss feedback.

There also remain several interesting directions worthy of future exploration.
In the case of stochastic linear bandits with bounded or sub-Gaussian noises, it is known that the optimal instance-dependent regret depends on the solution $c(\gA,\ell)$ to an optimization problem without a closed-form expression \citep{lattimore2020bandit}. As a result, the commonly seen $\gO(\nicefrac{d}{\Delta_{\min}})$-style dependence can be arbitrarily suboptimal compared to $c(\gA,\ell)$. This suboptimality also applies to the regret of our \Cref{algo:algo3} in the stochastic regime. Nonetheless, we note that achieving a BOBW algorithm that adapts to the solution of this optimization problem has only been realized by \citet{LeeLWZ021}, via a detect-and-switch-based strategy. To the best of our knowledge, no FTRL-based BOBW algorithm has been shown to attain such instance-optimal regret, even in the simpler case of linear bandits with bounded or sub-Gaussian noises. We leave further improvement of our regret guarantee in stochastic regimes as an important direction for future work.

Besides, we also note that our bonus-enhanced FTRL framework requires the knowledge of the parameters $(\varepsilon,\sigma)$ of the heavy-tailed loss distributions to facilitate the setup of appropriate bonus functions. As mentioned in \Cref{rmk:rmk1}, this condition is required by all existing works studying heavy-tailed linear bandits, even in the stochastic setting. More importantly, \citet{GenaltiM0M24} prove that it is not possible to match the worst-case regret lower bound of the case with known $\varepsilon$ and $\sigma$, if either $\varepsilon$ or $\sigma$ is unknown. It remains an interesting but also challenging future direction to explore whether it is feasible to match the worst-case lower bound of the case with known $\varepsilon$ and $\sigma$ in the presence of unknown $\varepsilon$ or $\sigma$ by imposing (minimal) additional assumptions.
\vskip 0.2in
\bibliography{ref}

\clearpage
\appendices
\section{Omitted Proofs in Section~\ref{sec:bobw_mab}}\label{sec:app:mab}
\begin{proof}[Proof of \Cref{thm:thm1}]

To start with, we decompose the regret as follows:
\begin{align}
&\quad \Reg \notag\\
 &=  \mathbb{E}\left[\sum_{t=1}^T \ell_{t, a_t}-\ell_{t, a^*}\right] =  \mathbb{E}\left[\sum_{t=1}^T \left\langle p_t-p^*, \ell_t\right\rangle\right] \notag\\
    &=  \mathbb{E}\left[\sumt\left\langle q_t-p^*, \ell_t\right\rangle\right]+\mathbb{E}\left[\sumt \gamma_t\left\langle p_0-q_t, \ell_t\right\rangle\right]\notag\\
    &=  \mathbb{E}\left[\sumt\left\langle q_t-p^*, \tildeell_t\right\rangle\right]+\mathbb{E}\left[\sumt \gamma_t\left\langle p_0-q_t, \ell_t\right\rangle\right]\notag\\
    &= \underbrace{\mathbb{E}\left[\sumt\left\langle q_t-p^*, \hat{\ell}_t-b_t\right\rangle\right]}_{\text{FTRL Regret}}
    +\underbrace{\mathbb{E}\left[\sumt\left\langle q_t-p^*, \tildeell_t-\hat{\ell}_t+b_t\right\rangle\right]}_{\text{Bias Term}}+\underbrace{\mathbb{E}\left[\sumt \gamma_t\left\langle p_0-q_t, \ell_t\right\rangle\right]}_{\text{Exploration Term}}\label{eq:tabular_regret_decom}\,.
\end{align}
Then, applying \Cref{lem:tabular_ftrl_reg}, \Cref{lem:tabular_bias_term}, and \Cref{lem:tabular_exp_term}, one can see that
\begin{align}
    \Reg\lesssim \mathbb{E}\left[ \sigma^{\frac{1}{\varepsilon}}\sum_{t=1}^T
    t^{\frac{1-\varepsilon}{\varepsilon}}\sum_{a \neq a^*} {q}_{t,a}^{\frac{1}{\varepsilon}}\right]
    +\kappa
    +\sigma^{\frac{1}{\varepsilon}}K\iota(\varepsilon)\,,\label{eq:tabular_overall_6}
\end{align}
where recall $\kappa=8\varepsilon^2\sigma^{\frac{1}{\varepsilon}}K^{\frac{2(\varepsilon-1)}{\varepsilon}}/(\varepsilon-1)$, and $\iota(\varepsilon)=\log T$ if $\varepsilon=2$ and $\iota(\varepsilon)=\left(\frac{8\varepsilon}{\varepsilon-1}\right)^{\frac{\varepsilon}{\varepsilon-1}}\frac{\varepsilon-1}{2-\varepsilon}$ if $\varepsilon\in (1,2)$.

\paragraph{Regret in Adversarial Regimes} Using Cauchy–Schwarz inequality, we proceed to bound Eq.~\eqref{eq:tabular_overall_6} as 
\begin{align}
    \Reg
    \lesssim \mathbb{E}\left[ \sigma^{\frac{1}{\varepsilon}}\sum_{t=1}^T
    t^{\frac{1-\varepsilon}{\varepsilon}}K^{\frac{\varepsilon-1}{\varepsilon}}\right]
    +\kappa
    +\sigma^{\frac{1}{\varepsilon}}K\iota(\varepsilon) \lesssim \sigma^{\frac{1}{\varepsilon}}K^{\frac{\varepsilon-1}{\varepsilon}}T^{\frac{1}{\varepsilon}}
    +\kappa
    +\sigma^{\frac{1}{\varepsilon}}K\iota(\varepsilon)\,. \notag
\end{align}

\paragraph{Regret in Adversarial Regimes with a Self-bounding Constraint} In adversarial regime with a $(\Delta, C, T)$ self-bounding constraint, we further have
\begin{align}
\Reg &\lesssim \sigma^{\frac{1}{\varepsilon}}\mathbb{E}\left[ \sum_{t=1}^T
    \sum_{a \neq a^*}\left(\Delta_a^{\frac{1}{1-\varepsilon}} t^{-1}\right)^{1-\frac{1}{\varepsilon}} \left( {q}_{t,a} \Delta_a\right)^{\frac{1}{\varepsilon}}\right]
    +\kappa
    +\sigma^{\frac{1}{\varepsilon}}K\iota(\varepsilon) \notag\\
    &\le \sigma^{\frac{1}{\varepsilon}}\mathbb{E}\left[ \left(\sum_{t=1}^T
    \sum_{a \neq a^*}\Delta_a^{\frac{1}{1-\varepsilon}} t^{-1}\right)^{1-\frac{1}{\varepsilon}} \left(\sum_{t=1}^T
    \sum_{a \neq a^*} {q}_{t,a} \Delta_a\right)^{\frac{1}{\varepsilon}}\right]
    +\kappa
    +\sigma^{\frac{1}{\varepsilon}}K\iota(\varepsilon) \label{eq:tabular_sto_reg_1}\\
    &\lesssim \sigma^{\frac{1}{\varepsilon}}\left(
    \sum_{a \neq a^*}\Delta_a^{\frac{1}{1-\varepsilon}} \log T\right)^{1-\frac{1}{\varepsilon}}\mathbb{E}\left[  \left(\sum_{t=1}^T
    \sum_{a \neq a^*} {q}_{t,a} \Delta_a\right)^{\frac{1}{\varepsilon}}\right]
    +\kappa
    +\sigma^{\frac{1}{\varepsilon}}K\iota(\varepsilon) \notag\\
    &\lesssim \sigma^{\frac{1}{\varepsilon}}\left(
    \sum_{a \neq a^*}\Delta_a^{\frac{1}{1-\varepsilon}} \log T\right)^{1-\frac{1}{\varepsilon}}
    \left(\mathbb{E}\left[  \sum_{t=1}^T\sum_{a \neq a^*} {q}_{t,a} \Delta_a\right]\right)^{\frac{1}{\varepsilon}}
    +\kappa
    +\sigma^{\frac{1}{\varepsilon}}K\iota(\varepsilon) \label{eq:tabular_sto_reg_2}\,,
\end{align}
where Eq.~\eqref{eq:tabular_sto_reg_1} is by Cauchy–Schwarz inequality and Eq.~\eqref{eq:tabular_sto_reg_2} follows from Jensen's inequality and $\varepsilon\in (1,2]$.

Moreover, \Cref{def:def1} along with \Cref{eq:tabular_gamma_0.5} implies that
\begin{align}
    \operatorname{Reg}_T\ge \E\left[\sum_{t=1}^T \sum_{a\in\gA} \Delta_a p_{t,a}\right]-C\ge \frac{1}{2}\E\left[\sum_{t=1}^T \sum_{a\in\gA} \Delta_a q_{t,a}\right]-C\,.\notag
\end{align}
Let $Q_T\coloneqq \E\left[\sum_{t=1}^T \sum_{a\in\gA} \Delta_a q_{t,a}\right]$ and $\omega(\Delta)\coloneqq \left(\sum_{a \neq a^*}\Delta_a^{\frac{1}{1-\varepsilon}} \right)^{1-\frac{1}{\varepsilon}}$.
For any $\lambda\in(0,1]$, combining the above display with Eq.~\eqref{eq:tabular_sto_reg_2} leads to
\begin{align}
\operatorname{Reg}_T&=(1+\lambda) \operatorname{Reg}_T-\lambda \operatorname{Reg}_T \notag\\
    &\lesssim(1+\lambda)\left(\sigma^{\frac{1}{\varepsilon}}
    \omega(\Delta) \left(\log T\right)^{1-\frac{1}{\varepsilon}}
    Q_T^{\frac{1}{\varepsilon}}+\kappa+\sigma^{\frac{1}{\varepsilon}}K\iota(\varepsilon)\right)-\lambda\left(\frac{Q_T}{2}-C\right) \notag\\
    &\lesssim \left((1+\lambda)\sigma^{\frac{1}{\varepsilon}}
    \omega(\Delta) \left(\log T\right)^{1-\frac{1}{\varepsilon}}\right)^{\frac{\varepsilon}{\varepsilon-1}}\left( \frac{\varepsilon\lambda}{2}\right)^{\frac{1}{1-\varepsilon}}
    +\kappa
    +\sigma^{\frac{1}{\varepsilon}}K\iota(\varepsilon)+\lambda C \label{eq:tabular_sto_reg_3}\\
    &\le 2^{\frac{\varepsilon}{\varepsilon-1}} \left(
    \sigma^{\frac{1}{\varepsilon}}\omega(\Delta) \left(\log T\right)^{1-\frac{1}{\varepsilon}}\right)^{\frac{\varepsilon}{\varepsilon-1}}\left( \frac{\varepsilon}{2}\right)^{\frac{1}{1-\varepsilon}}\lambda^{\frac{1}{1-\varepsilon}}
    +\kappa
    +\sigma^{\frac{1}{\varepsilon}}K\iota(\varepsilon)+\lambda C\,,\notag
\end{align}
where Eq.~\eqref{eq:tabular_sto_reg_3} is due to that $f(x)=a x^{\frac{1}{c}}-b x=O(a^{\frac{c}{c-1}}(cb)^{\frac{1}{1-c}})$ for any $a,b,x>0$ and $c>1$ as $f(x)$ is a concave function and the condition $\lambda\in(0,1]$.

Now we consider two cases. If $C = 0$, setting $\lambda=1$ leads to 
\begin{align}\label{eq:tabular_sto_c=0}
    \operatorname{Reg}_T\lesssim 2^{\frac{\varepsilon}{\varepsilon-1}} \left(\sigma^{\frac{1}{\varepsilon}}
    \omega(\Delta) \left(\log T\right)^{1-\frac{1}{\varepsilon}}\right)^{\frac{\varepsilon}{\varepsilon-1}}\left( \frac{\varepsilon}{2}\right)^{\frac{1}{1-\varepsilon}}
    +\kappa
    +\sigma^{\frac{1}{\varepsilon}}K\iota(\varepsilon)\,.
\end{align}
If $C >0$, setting 
\begin{align*}
    \lambda=\left(2^{\frac{\varepsilon}{\varepsilon-1}} \left(\sigma^{\frac{1}{\varepsilon}}\omega(\Delta) \left(\log T\right)^{1-\frac{1}{\varepsilon}}\right)^{\frac{\varepsilon}{\varepsilon-1}}\left( \frac{\varepsilon}{2}\right)^{\frac{1}{1-\varepsilon}}C^{-1} \right)^{\frac{\varepsilon-1}{\varepsilon}}
    =2 \sigma^{\frac{1}{\varepsilon}}\omega(\Delta) \left(\log T\right)^{1-\frac{1}{\varepsilon}}\left( \frac{\varepsilon}{2}\right)^{-\frac{1}{\varepsilon}}C^{\frac{1-\varepsilon}{\varepsilon}}
\end{align*}
    leads to 
\begin{align}\label{eq:tabular_sto_c>0}
    \operatorname{Reg}_T\lesssim \omega(\Delta) \left(\log T\right)^{1-\frac{1}{\varepsilon}}\left(\frac{2\sigma}{\varepsilon}\right)^{\frac{1}{\varepsilon}}C^{\frac{1}{\varepsilon}}
    +\kappa
    +\sigma^{\frac{1}{\varepsilon}}K\iota(\varepsilon)\,.
\end{align}
Combining Eq.~\eqref{eq:tabular_sto_c=0} and Eq.~\eqref{eq:tabular_sto_c>0} shows that
\begin{align}
    \operatorname{Reg}_T
    &\lesssim 2^{\frac{\varepsilon}{\varepsilon-1}} \left(\sigma^{\frac{1}{\varepsilon}}
    \omega(\Delta) \left(\log T\right)^{1-\frac{1}{\varepsilon}}\right)^{\frac{\varepsilon}{\varepsilon-1}}\left( \frac{\varepsilon}{2}\right)^{\frac{1}{1-\varepsilon}}
    +\omega(\Delta) \left(\log T\right)^{1-\frac{1}{\varepsilon}}\left(\frac{2\sigma}{\varepsilon}\right)^{\frac{1}{\varepsilon}}C^{\frac{1}{\varepsilon}} +\kappa
    +\sigma^{\frac{1}{\varepsilon}}K\iota(\varepsilon) \notag\\
    &=2^{\frac{\varepsilon}{\varepsilon-1}}\left(\frac{2\sigma}{\varepsilon}\right)^{\frac{1}{\varepsilon-1}}
    \omega(\Delta)^{\frac{\varepsilon}{\varepsilon-1}} \log T
    + \omega(\Delta) \left(\log T\right)^{1-\frac{1}{\varepsilon}}\left(\frac{2\sigma}{\varepsilon}\right)^{\frac{1}{\varepsilon}}C^{\frac{1}{\varepsilon}}+\kappa
    +\sigma^{\frac{1}{\varepsilon}}K\iota(\varepsilon)\,.\notag
\end{align}
\end{proof}

\subsection{Bounding FTRL Regret}
\begin{lem}\label{lem:tabular_ftrl_reg}
    The FTRL Regret in Eq.~\eqref{eq:tabular_regret_decom} satisfies
    \begin{align}
        \text{FTRL Regret}\lesssim\mathbb{E}\left[ \sigma^{\frac{1}{\varepsilon}}\sum_{t=1}^T
    t^{\frac{1-\varepsilon}{\varepsilon}}\sum_{a \neq a^*} {q}_{t,a}^{\frac{1}{\varepsilon}}\right]+\kappa\,,\notag
    \end{align}
where $\kappa= 8\varepsilon^2\sigma^{\frac{1}{\varepsilon}}K^{\frac{2(\varepsilon-1)}{\varepsilon}}/(\varepsilon-1)$.
\end{lem}
\begin{proof}
Via standard analysis for FTRL (\textit{e.g.}, Exercise~28.12 of \citet{lattimore2020bandit}), we have
\begin{align}
    \text{FTRL Regret}\le \mathbb{E}\left[\sumt\left(\left\langle q_t-q_{t+1}, \hat{\ell}_t-b_t\right\rangle-\beta_t D_{\psi}\left(q_{t+1}, q_t\right)+\left(\beta_t-\beta_{t-1}\right) h_t\right) \right]\,,\label{eq:FTRL_reg}
\end{align}
where we define $\beta_0 = 0$, $h_t=-\psi\left(q_t\right)$, and recall $D_{\psi}(x,y)=\psi(x)-\psi(y)-\langle \nabla \psi(y),x-y \rangle$ is the Bregman divergence associated with $\psi$.

\paragraph{Bounding Stability Term}
Notice that 
\begin{align}\label{eq:tabular_bta_sta}
    b_{t,a}\coloneqq 
    p_{t, a}^{1-\varepsilon} s_{t, a}^{1-\varepsilon}\sigma
    \le \left(\frac{\gamma_t}{K}\right)^{1-\varepsilon}s_{t, a}^{1-\varepsilon}\sigma
    =s_{t, \tilde{a}_t}^{\varepsilon}s_{t, a}^{1-\varepsilon} \le s_{t, a}^{\varepsilon}s_{t, a}^{1-\varepsilon}=s_{t, a}\,,
\end{align}
where in the second equality we use the definition $\gamma_t=\sigma^{\frac{1}{\varepsilon-1}} K s_{t, \tilde{a}_t}^{\varepsilon/(1-\varepsilon)}$ and the last inequality follows from that $s_{t,a}=(1-\alpha)\tilde{q}_{t,a}^{\alpha-1}\beta_t/8$ and $\tilde{q}_{t,\tilde{a}_t}^{\alpha-1}\le\tilde{q}_{t, a}^{\alpha-1}$ for all $a\in\gA$. This together with $|\hat{\ell}_{t, a}|=|\tildeell_{t,a} \cdot \inlineindicator{| \tildeell_{t, a}| \le s_{t, a}} |\le s_{t,a}$ implies that
\begin{align}\label{eq:loss_magnitude_tabular}
    |\hat{\ell}_{t, a}-b_{t, a}| \le 2s_{t,a}= \frac{(1-\alpha)\beta_t \tilde{q}_{t,a}^{\alpha-1}}{4}\,.
\end{align}
Therefore, by applying \Cref{lem:stab_condition_m1}, one can derive that
    \begin{align}\label{eq:stab_fdxxs1_tabular}
\E_{t-1}\left[\langle q_t-q_{t+1}, \hat{\ell}_t-b_t\rangle-\beta_t D_{\psi}\left(q_{t+1}, q_t\right)\right]
        &\lesssim \frac{1}{1-\alpha} \beta_t^{-1}\mathbb{E}_{t-1}\left[\sum_{a\in\gA} \tilde{q}_{t, a}^{2-\alpha}\left(\hat{\ell}_{t , a}+b_{t , a}\right)^2\right] \notag\\
        &\lesssim \frac{1}{1-\alpha} \beta_t^{-1}\mathbb{E}_{t-1}\left[\sum_{a\in\gA} \tilde{q}_{t, a}^{2-\alpha}\left(\hat{\ell}_{t , a}^2+b_{t , a}^2\right)\right]\,.
    \end{align}
For the term $\beta_t^{-1}\mathbb{E}_{t-1}\left[\sum_{a\in\gA} \tilde{q}_{t, a}^{2-\alpha}\hat{\ell}_{t , a}^2\right]$ in Eq.~\eqref{eq:stab_fdxxs1_tabular}, we have
\begin{align}
     \beta_t^{-1}\mathbb{E}_{t-1}\left[\sum_{a\in\gA} \tilde{q}_{t, a}^{2-\alpha}\hat{\ell}_{t , a}^2\right]
    &\le \beta_t^{-1}\mathbb{E}_{t-1}\left[\sum_{a\in\gA} \tilde{q}_{t, a}^{2-\alpha}\abs{\tildeell_{t,a}}^{\varepsilon}  s_{t, a}^{2-\varepsilon} \right] \label{eq:internal_fdll1}\\
    &= \beta_t^{-1}\mathbb{E}_{t-1}\left[\sum_{a\in\gA} \tilde{q}_{t, a}^{2-\alpha}\frac{\abs{\ell_{t, a}}^{\varepsilon} \inlineindicator{a=a_t} }{p_{t, a}^\varepsilon}  s_{t, a}^{2-\varepsilon} \right] \notag\\
    &\le \sigma\beta_t^{-1} \sum_{a\in\gA} \tilde{q}_{t, a}^{2-\alpha} p_{t, a}^{1-\varepsilon} s_{t, a}^{2-\varepsilon} \label{eq:internal_fdll2}\\
    &= \sigma\left(1-\frac{1}{\varepsilon}\right)^{2-\varepsilon}\beta_t^{1-\varepsilon} \sum_{a\in\gA} \tilde{q}_{t, a}^{\frac{1}{\varepsilon}+\varepsilon-1} p_{t, a}^{1-\varepsilon} \label{eq:internal_fdll3}\\
    &\le \sigma\left(1-\frac{1}{\varepsilon}\right)^{2-\varepsilon}\beta_t^{1-\varepsilon} \sum_{a\in\gA} \tilde{q}_{t, a}^{\frac{1}{\varepsilon}}\left(2 p_{t, a}\right)^{\varepsilon-1} p_{t, a}^{1-\varepsilon} \label{eq:eq:internal_fdll3_1}\\
    &\lesssim \sigma\left(1-\frac{1}{\varepsilon}\right)^{2-\varepsilon}\beta_t^{1-\varepsilon} \sum_{a\in\gA} \tilde{q}_{t, a}^{\frac{1}{\varepsilon}} \notag\\
    &\le \sigma\left(1-\frac{1}{\varepsilon}\right)^{2-\varepsilon}\beta_t^{1-\varepsilon} \left(\sum_{a\in\gA\setminus \{\tilde{a}_t\}} {q}_{t, a}^{\frac{1}{\varepsilon}}+(1-q_{t,\tilde{a}_t})^{\frac{1}{\varepsilon}} \right) \notag\\
    &\le 2\sigma\left(1-\frac{1}{\varepsilon}\right)^{2-\varepsilon}\beta_t^{1-\varepsilon} \sum_{a\in\gA\setminus \{\tilde{a}_t\}} {q}_{t, a}^{\frac{1}{\varepsilon}} \notag\\
    &\le 2\sigma\left(1-\frac{1}{\varepsilon}\right)^{2-\varepsilon}\beta_t^{1-\varepsilon} \sum_{a\neq a^*} {q}_{t, a}^{\frac{1}{\varepsilon}}\,,\label{eq:loss_square_tabular}
\end{align}
where in Eq.~\eqref{eq:internal_fdll1} we use 
\begin{align*}
    \mathbb{E}_{t-1}\left[\hat{\ell}_{t,a}^2\right]
    =
    \mathbb{E}_{t-1}\left[\tildeell_{t,a}^2 \inlineindicator{|{\tildeell_{t,a}}| \le s_{t,a}}\right]
    =
    \mathbb{E}_{t-1}\left[|\tildeell_{t,a}|^{\varepsilon} \abs{\tildeell_{t,a}}^{2-\varepsilon} \inlineindicator{|{\tildeell_{t,a}}| 
\le 
s_{t,a}}\right]
\le 
\mathbb{E}_{t-1}\left[|\tildeell_{t,a}|^{\varepsilon} s_{t,a}^{2-\varepsilon}\right]\,,
\end{align*}
Eq.~\eqref{eq:internal_fdll2} is by \Cref{def:HT}, 
Eq.~\eqref{eq:internal_fdll3} follows from choosing $s_{t,a}=(1-\alpha)\tilde{q}_{t,a}^{\alpha-1}\beta_t/8$ and $\alpha=1/\varepsilon$, and Eq.~\eqref{eq:eq:internal_fdll3_1} is because $\gamma_t\le 1/2$ by \Cref{eq:tabular_gamma_0.5} and thus $\tilde{q}_{t,a}\le q_{t,a}\le 2p_{t,a}$.

Meanwhile, for the term $ \beta_t^{-1}\mathbb{E}_{t-1}\left[\sum_{a\in\gA} \tilde{q}_{t, a}^{2-\alpha}b_{t , a}^2\right]$ in Eq.~\eqref{eq:stab_fdxxs1_tabular}, we have
\begin{align}
     \beta_t^{-1}\mathbb{E}_{t-1}\left[\sum_{a\in\gA} \tilde{q}_{t, a}^{2-\alpha}b_{t , a}^2\right]
    &\le \beta_t^{-1} \sum_{a\in\gA} \tilde{q}_{t, a}^{2-\alpha}b_{t , a}s_{t,a} = \sigma\beta_t^{-1} \sum_{a\in\gA} \tilde{q}_{t, a}^{2-\alpha} p_{t, a}^{1-\varepsilon} s_{t, a}^{2-\varepsilon} \notag\\
    &\lesssim 2\sigma\left(1-\frac{1}{\varepsilon}\right)^{2-\varepsilon}\beta_t^{1-\varepsilon} \sum_{a\neq a^*} {q}_{t, a}^{\frac{1}{\varepsilon}}\,,\label{eq:bonus_square_tabular}
\end{align}
where the first inequality follows from Eq.~\eqref{eq:tabular_bta_sta} and the second inequality is by a similar reasoning as in Eq.~\eqref{eq:loss_square_tabular}.

Substituting Eq.~\eqref{eq:loss_square_tabular} and Eq.~\eqref{eq:bonus_square_tabular} into Eq.~\eqref{eq:stab_fdxxs1_tabular} leads to
\begin{align}
    \E_{t-1}\left[\langle q_t-q_{t+1}, \hat{\ell}_t-b_t\rangle-\beta_t D_{\psi}\left(q_{t+1}, q_t\right)\right]
    &\lesssim  \sigma\left(1-\frac{1}{\varepsilon}\right)^{2-\varepsilon}\beta_t^{1-\varepsilon} \sum_{a \neq a^*} {q}_{t, a}^{\frac{1}{\varepsilon}} \notag\\
    &\lesssim \sigma^{\frac{1}{\varepsilon}}t^{\frac{1-\varepsilon}{\varepsilon}}\sum_{a \neq a^*} {q}_{t,a}^{\frac{1}{\varepsilon}}\,,\label{eq:tabular_stability_term}
\end{align}
where the second inequality is by $\beta_t\ge\sigma^{1/\varepsilon}t^{1/\varepsilon}$ for all $t\in[T]$.

\paragraph{Bounding Penalty Term}
First, note that the round-wise penalty term satisfies that
\begin{align}
 \left(\beta_t-\beta_{t-1}\right) h_t
    =\left(\beta_t-\beta_{t-1}\right) \left(\sum_{a \in \gA} q_{t, a}^{\frac{1}{\varepsilon}}-1\right) 
    &\le \left(\beta_t-\beta_{t-1}\right) \left(\sum_{a \in \gA} q_{t, a}^{\frac{1}{\varepsilon}}-q_{t, a^{*}}^{\frac{1}{\varepsilon}}\right) \notag\\
    &= \left(\beta_t-\beta_{t-1}\right)\sum_{a \neq a^{*}} q_{t, a}^{\frac{1}{\varepsilon}} \label{eq:tabular_penalty_term}\,.
\end{align}
Let $m=\left(\frac{8\varepsilon}{\varepsilon-1}\right)^{\varepsilon}K^{\varepsilon-1}$.
Recall $\beta_t=\sigma^{1/\varepsilon}\max\{ t^{1/\varepsilon}, \frac{8\varepsilon}{\varepsilon-1}K^{\frac{\varepsilon-1}{\varepsilon}}\}$. Then, one can see that
\begin{align}
    \sumt\left(\beta_t-\beta_{t-1}\right) h_t
    &\le\beta_1 h_1+\sum_{t=2}^{\left\lfloor m \right\rfloor}\left(\beta_t-\beta_{t-1}\right) h_t + \sum_{t=\left\lceil m \right\rceil}^{T}\left(\beta_t-\beta_{t-1}\right) h_t \notag\\
    &\le 
    \beta_1 h_1
    +
    \sum_{t=\left\lceil m \right\rceil}^{T} 
    \sigma^{\frac{1}{\varepsilon}} \left(t^{\frac{1}{\varepsilon}}-(t-1)^{\frac{1}{\varepsilon}}\right)\sum_{a \neq a^{*}} q_{t, a}^{\frac{1}{\varepsilon}} \notag\\
    &\le 
    \beta_1 h_1
    +
    \sum_{t=\left\lceil m \right\rceil}^{T} \sigma^{\frac{1}{\varepsilon}} \frac{1}{\varepsilon} (t-1)^{\frac{1}{\varepsilon}-1}\sum_{a \neq a^{*}} q_{t, a}^{\frac{1}{\varepsilon}} \label{eq:tabular_ftrl_pen_x1}\\
    &\lesssim 
    \beta_1 h_1
    +
    \sum_{t=\left\lceil m \right\rceil}^{T} \sigma^{\frac{1}{\varepsilon}} \frac{1}{\varepsilon}t^{\frac{1}{\varepsilon}-1}\sum_{a \neq a^{*}} q_{t, a}^{\frac{1}{\varepsilon}} \label{eq:tabular_ftrl_pen_x3}\\
    &=
    \sigma^{\frac{1}{\varepsilon}}\frac{8\varepsilon^2}{\varepsilon-1}K^{\frac{2(\varepsilon-1)}{\varepsilon}}
    +
    \sum_{t=\left\lceil m \right\rceil}^{T} \sigma^{\frac{1}{\varepsilon}} \frac{1}{\varepsilon}t^{\frac{1}{\varepsilon}-1}\sum_{a \neq a^{*}} q_{t, a}^{\frac{1}{\varepsilon}}\,, \label{eq:tabular_ftrl_pen_x4}
\end{align}
where Eq.~\eqref{eq:tabular_ftrl_pen_x1} is due to $t^{1/\varepsilon}-(t-1)^{1/\varepsilon}\le \frac{1}{\varepsilon} (t-1)^{1/\varepsilon-1}$ as $t\mapsto t^{1/\varepsilon}$ is a concave function, and Eq.~\eqref{eq:tabular_ftrl_pen_x3} holds as $(t-1)^{\frac{1-\varepsilon}{\varepsilon}}\le 2t^{\frac{1-\varepsilon}{\varepsilon}}$ for any $t\ge \frac{1}{3}$.

\paragraph{Overall Regret}
Substituting Eq.~\eqref{eq:tabular_stability_term} and Eq.~\eqref{eq:tabular_ftrl_pen_x4} into Eq.~\eqref{eq:FTRL_reg} shows that
\begin{align}
    \text{FTRL Regret}\lesssim \mathbb{E}\left[ \sigma^{\frac{1}{\varepsilon}}\sum_{t=1}^T
    t^{\frac{1-\varepsilon}{\varepsilon}}\sum_{a \neq a^*} {q}_{t,a}^{\frac{1}{\varepsilon}}\right]+\kappa \label{eq:tabular_overall_5}\,,
\end{align}
where $\kappa\coloneqq 8\varepsilon^2\sigma^{\frac{1}{\varepsilon}}K^{\frac{2(\varepsilon-1)}{\varepsilon}}/(\varepsilon-1)$.
\end{proof}

\subsection{Bounding Bias Term}
\begin{lem}\label{lem:tabular_bias_term}
    The Bias Term in Eq.~\eqref{eq:tabular_regret_decom} satisfies
    \begin{align}
        \text{Bias Term}\le \mathbb{E}\left[ \sigma^{\frac{1}{\varepsilon}}\sum_{t=1}^T
    t^{\frac{1-\varepsilon}{\varepsilon}}\sum_{a \neq a^*} {q}_{t,a}^{\frac{1}{\varepsilon}}\right]\,.\notag
    \end{align}
\end{lem}
\begin{proof}
For the Bias Term, one can deduce that
\begin{align}
    \E_{t-1}\left[\left\langle q_t-p^*, \tildeell_t-\hat{\ell}_t+b_t\right\rangle \right]
    &=\E_{t-1}\left[\sum_{a\in\gA} \left(q_{t,a}-p^*_a\right) \tildeell_{t, a} \inlineindicator{|\tildeell_{t,a}|>s_{t, a}}\right]+\left\langle q_t-p^*,b_t\right\rangle \notag\\
    &\le \E_{t-1}\left[\sum_{a\in\gA} \left|q_{t,a}-p^*_a\right| |\tildeell_{t, a}| \inlineindicator{|\tildeell_{t,a}|>s_{t, a}}\right]
    +\left\langle q_t-p^*,b_t\right\rangle \notag\\
    &\le \E_{t-1}\left[\sum_{a\in\gA} \left|q_{t,a}-p^*_a\right| |\tildeell_{t, a}|^{\varepsilon} s_{t, a}^{1-\varepsilon}\right]
    +\left\langle q_t-p^*,b_t\right\rangle \label{eq:bias_tabular_dsfkk1}\\
    &=\E_{t-1}\left[\sum_{a\in\gA} \left|q_{t,a}-p^*_a\right| \frac{\abs{\ell_{t, a}}^{\varepsilon}\inlineindicator{a=a_t} }{p_{t, a}^\varepsilon} s_{t, a}^{1-\varepsilon}\right]
        +\left\langle q_t-p^*,b_t\right\rangle \notag\\
    &\le\sigma \E_{t-1}\left[\sum_{a\in\gA} \left|q_{t,a}-p^*_a\right| {p_{t, a}^{1-\varepsilon}} s_{t, a}^{1-\varepsilon}\right]
        +\left\langle q_t-p^*,b_t\right\rangle \label{eq:bias_tabular_dsfkk2}\\
    &= 2\sigma\sum_{a \neq a^*} q_{t,a}  p_{t, a}^{1-\varepsilon} s_{t, a}^{1-\varepsilon} \label{eq:bias_tabular_dsfkk3}\\
    &\lesssim (1-\alpha)^{1-\varepsilon}\sigma\beta_t^{1-\varepsilon}\sum_{a \neq a^*} q_{t,a} p_{t, a}^{1-\varepsilon} \tilde{q}_{t,a}^{(\alpha-1)(1-\varepsilon)} \label{eq:bias_tabular_dsfkk4}\\
    &\le \sigma\beta_t^{1-\varepsilon}\sum_{a \neq a^*} p_{t, a}^{1-\varepsilon}{q}_{t,a}^{\frac{1}{\varepsilon}+\varepsilon-1} \label{eq:bias_tabular_dsfkk5}\\
    &\le \sigma\beta_t^{1-\varepsilon}\sum_{a \neq a^*} p_{t, a}^{1-\varepsilon} \left(2{p}_{t,a}\right)^{\varepsilon-1} {q}_{t,a}^{\frac{1}{\varepsilon}}
    \lesssim \sigma\beta_t^{1-\varepsilon}\sum_{a \neq a^*} {q}_{t,a}^{\frac{1}{\varepsilon}} \label{eq:tabular_bias_term}\,,
\end{align}
where Eq.~\eqref{eq:bias_tabular_dsfkk2} is due to \Cref{def:HT}, Eq.~\eqref{eq:bias_tabular_dsfkk3} is by $b_{t,a}= \sigma p_{t, a}^{1-\varepsilon} s_{t, a}^{1-\varepsilon}$, Eq.~\eqref{eq:bias_tabular_dsfkk4} uses $s_{t,a}=(1-\alpha)\tilde{q}_{t,a}^{\alpha-1}\beta_t/8$, and Eq.~\eqref{eq:bias_tabular_dsfkk5} follows from $\tilde{q}_{t,a}\le q_{t,a}$ for all $a\in\gA$ and choosing $\alpha=1/\varepsilon$.

Therefore, it holds that
\begin{align}
    \text{Bias Term}
    \lesssim
    \E\left[\sumt\sigma\beta_t^{1-\varepsilon}\sum_{a \neq a^*} {q}_{t,a}^{\frac{1}{\varepsilon}} \right]
    \le\mathbb{E}\left[ \sigma^{\frac{1}{\varepsilon}}\sum_{t=1}^T
    t^{\frac{1-\varepsilon}{\varepsilon}}\sum_{a \neq a^*} {q}_{t,a}^{\frac{1}{\varepsilon}}\right]\,,
\end{align}
where the second inequality is due to that $\beta_t\ge\sigma^{1/\varepsilon} t^{1/\varepsilon}$ for all $t\in[T]$ by the set up of $\beta_t$ in Line \ref{alg:line:tabular_beta1} of \Cref{algo:algo1}.
\end{proof}

\subsection{Bounding Exploration Term and Properties of $\gamma_t$}\label{sec:app:exp_tabular}
\begin{lem}\label{lem:tabular_exp_term}
    The Exploration Term in Eq.~\eqref{eq:tabular_regret_decom} satisfies
    \begin{align}
        \text{Exploration Term}\lesssim \sigma^{\frac{1}{\varepsilon}}K\iota(\varepsilon)\,,\notag
    \end{align}
where $\iota(\varepsilon)=\log T$ if $\varepsilon=2$ and $\iota(\varepsilon)=\left(\frac{8\varepsilon}{\varepsilon-1}\right)^{\frac{\varepsilon}{\varepsilon-1}}\frac{\varepsilon-1}{2-\varepsilon}$ if $\varepsilon\in (1,2)$.
\end{lem}
\begin{proof}
First note that
\begin{align}
    \text{Exploration Term}=\mathbb{E}\left[\sumt \gamma_t\left\langle p_0-q_t, \ell_t\right\rangle\right]
    \le 
    2\sigma^{\frac{1}{\varepsilon}}\mathbb{E}\left[\sumt \gamma_t\right]\,,\label{eq:exp_term_tabular}
\end{align}
where the inequality follows from 
\begin{align}
    \E_{\ell_{t,a}\sim \nu_{t,a}}[\ell_{t,a}]
    \le \E_{\ell_{t,a}\sim \nu_{t,a}}[|\ell_{t,a}|]
    \le \left(\E_{\ell_{t,a}\sim \nu_{t,a}}[|\ell_{t,a}|^{\varepsilon}]\right)^{\frac{1}{\varepsilon}}
    \le \sigma^{\frac{1}{\varepsilon}}\label{eq:lower_moment}
\end{align}
for all $a\in\gA$ by Jensen's inequality and \Cref{def:HT}.

Furthermore, by setting $\alpha=1/\varepsilon$ and the set up of $\beta_t$ in Line \ref{alg:line:tabular_beta1} of \Cref{algo:algo1}, for all $t\in[T]$, we have
\begin{align}
    \gamma_t
    &\coloneqq \sigma^{\frac{1}{\varepsilon-1}}K s_{t, \tilde{a}_t}^{\varepsilon/(1-\varepsilon)}
    =\sigma^{\frac{1}{\varepsilon-1}}\left(\frac{8}{1-\alpha}\right)^{\frac{\varepsilon}{\varepsilon-1}}K\beta_t^{\frac{\varepsilon}{1-\varepsilon}}\tilde{q}_{t,\tilde{a}_t}^{(\alpha-1)\varepsilon/(1-\varepsilon)} \notag\\
    &=\tilde{q}_{t,\tilde{a}_t}\sigma^{\frac{1}{\varepsilon-1}}\left(\frac{8\varepsilon}{\varepsilon-1}\right)^{\frac{\varepsilon}{\varepsilon-1}}K\beta_t^{\frac{\varepsilon}{1-\varepsilon}}
    \le \frac{1}{2}\sigma^{\frac{1}{\varepsilon-1}}\left(\frac{8\varepsilon}{\varepsilon-1}\right)^{\frac{\varepsilon}{\varepsilon-1}}K\beta_t^{\frac{\varepsilon}{1-\varepsilon}} \label{eq:tabular_gamma0}\\
    &\le \frac{1}{2}\sigma^{\frac{1}{\varepsilon-1}}\left(\frac{8\varepsilon}{\varepsilon-1}\right)^{\frac{\varepsilon}{\varepsilon-1}}K \cdot \left(\sigma^{1/\varepsilon} t^{1/\varepsilon}\right)^{\frac{\varepsilon}{1-\varepsilon}}
    =\frac{1}{2}\left(\frac{8\varepsilon}{\varepsilon-1}\right)^{\frac{\varepsilon}{\varepsilon-1}}K\cdot t^{\frac{1}{1-\varepsilon}} \label{eq:tabular_gamma00}\,,
\end{align}
which together with  Eq.~\eqref{eq:exp_term_tabular} implies that
\begin{align}
    \text{Exploration Term}\le \sigma^{\frac{1}{\varepsilon}}\mathbb{E}\left[\sumt \left(\frac{8\varepsilon}{\varepsilon-1}\right)^{\frac{\varepsilon}{\varepsilon-1}}K t^{\frac{1}{1-\varepsilon}}\right] 
    \lesssim\sigma^{\frac{1}{\varepsilon}}K \iota(\varepsilon) \,,\notag
\end{align}
where the last inequality follows by considering the case $\varepsilon=2$ and $\varepsilon\in (1,2)$ separately.
\end{proof}

\begin{lem}\label{eq:tabular_gamma_0.5}
    For all $t\in [T]$, it holds that $\gamma_t\le 1/2$.
\end{lem}
\begin{proof}
    When $t\le \left(\frac{8\varepsilon}{\varepsilon-1}\right)^{\varepsilon}K^{\varepsilon-1}$, by the set up of $\beta_t$ in Line \ref{alg:line:tabular_beta1} of \Cref{algo:algo1},  it holds that $\beta_t=\sigma^{1/\varepsilon} \frac{8\varepsilon}{\varepsilon-1}K^{\frac{\varepsilon-1}{\varepsilon}}$. In this case, using Eq.~\eqref{eq:tabular_gamma0} leads to 
\begin{align}
    \gamma_t\le \frac{1}{2}\sigma^{\frac{1}{\varepsilon-1}}\left(\frac{8\varepsilon}{\varepsilon-1}\right)^{\frac{\varepsilon}{\varepsilon-1}}K\cdot \left(\sigma^{1/\varepsilon} \frac{8\varepsilon}{\varepsilon-1}K^{\frac{\varepsilon-1}{\varepsilon}} \right)^{\frac{\varepsilon}{1-\varepsilon}} = \frac{1}{2} \label{eq:tabular_gamma1}\,.
\end{align}
When $t> \left(\frac{8\varepsilon}{\varepsilon-1}\right)^{\varepsilon}K^{\varepsilon-1}$, 
using Eq.~\eqref{eq:tabular_gamma00} shows that
\begin{align}
    \gamma_t\le \frac{1}{2}\left(\frac{8\varepsilon}{\varepsilon-1}\right)^{\frac{\varepsilon}{\varepsilon-1}}K\cdot 
     \left(\left(\frac{8\varepsilon}{\varepsilon-1}\right)^{\varepsilon}K^{\varepsilon-1}\right)^{\frac{1}{1-\varepsilon}} = \frac{1}{2} \label{eq:tabular_gamma3}\,.
\end{align}
The proof is concluded by considering both the cases in Eq.~\eqref{eq:tabular_gamma1} and Eq.~\eqref{eq:tabular_gamma3}.
\end{proof}

\section{Omitted Proofs in Section~\ref{sec:adv_htlb}}\label{sec:app:adv_htlb}
\begin{proof}[Proof of \Cref{thm:thm2}]

The proof starts with the following regret decomposition (note that it is slightly different from that in Eq.~\eqref{eq:tabular_regret_decom}):
\begin{align}
 \Reg 
 &=  \mathbb{E}\left[\sum_{t=1}^T \ell_{t, a_t}-\ell_{t, a^*}\right] =  \mathbb{E}\left[\sum_{t=1}^T \left\langle p_t-p^*, \ell_t\right\rangle\right] \notag\\
    &=  \mathbb{E}\left[\sumt(1-\gamma)\left\langle q_t-p^*, \ell_t\right\rangle\right]+\mathbb{E}\left[\sumt \gamma\left\langle p_0-p^*, \ell_t\right\rangle\right]\notag\\
    &=  \mathbb{E}\left[\sumt(1-\gamma)\left\langle q_t-p^*, \tildeell_t\right\rangle\right]+\mathbb{E}\left[\sumt \gamma\left\langle p_0-p^*, \ell_t\right\rangle\right]\notag\\
    &=  \underbrace{\mathbb{E}\left[\sumt\left(1-\gamma\right)\left\langle q_t-p^*, \hat{\ell}_t-b_t\right\rangle\right]}_{\text{FTRL Regret}}
    +\underbrace{\mathbb{E}\left[\sumt\left(1-\gamma\right)\left\langle q_t-p^*, \tildeell_t-\hat{\ell}_t+b_t\right\rangle\right]}_{\text{Bias Term}} \notag\\
    &\quad+\underbrace{\mathbb{E}\left[\sumt \gamma\left\langle p_0-p^*, \ell_t\right\rangle\right]}_{\text{Exploration Term}}\label{eq:adv_htlb_ftrl_regX0}\,.
\end{align}

\paragraph{Bounding FTRL Regret} 
First, notice that setting $\gamma = 4\sigma^{\frac{2}{\varepsilon}} d \beta^{-2}$ leads to 
\begin{align}\label{eq:adv_htlb_ftrl_stable}
    b_{t,a}\coloneqq \sigma s_{t,a}^{1-\varepsilon}\left(\phia^{\top} S_t^{-1} \phia\right)^{\frac{\varepsilon}{2}} 
    \leq 
    \sigma \left(\frac{2}{\beta}\right)^{\varepsilon-1}\left(\frac{d}{\gamma}\right)^{\frac{\varepsilon}{2}}
    =
    \frac{\beta}{2} \cdot\left(\frac{4 d \beta^{-2}}{\gamma}\right)^{\frac{\varepsilon}{2}}=\frac{\beta}{2}\,,
\end{align}
where the inequality holds as $p_t=(1-\gamma) q_t+\gamma p_0$ and $p_0$ is a G-optimal design over $\{\phi_a\}_{a\in\gA}$.
Then, using standard FTRL regret analysis (\textit{e.g.}, Theorem 28.5 of \citet{lattimore2020bandit}), one can deduce that 
\begin{align}
    &\quad\mathbb{E}\left[\sumt\left\langle q_t-p^*, \hat{\ell}_t-b_t\right\rangle\right]\notag\\
    &\le \mathbb{E}\left[\beta\left(\psi\left(q^*\right)-\psi\left(q_1\right)\right)
    + 
    \sumt\left(\left\langle q_t-q_{t+1}, \hat{\ell}_t-b_t\right\rangle-\beta D_\psi\left(q_{t+1}, q_t\right)\right)
    \right] \notag\\
    &\le \beta\log K
    + 
    \mathbb{E}\left[\sumt\left(\left\langle q_t-q_{t+1}, \hat{\ell}_t-b_t\right\rangle-\beta D_\psi\left(q_{t+1}, q_t\right)\right) \right] \notag\\
    &\le \beta\log K
    + 
    \mathbb{E}\left[\sumt\suma\beta q_{t,a}\left(\exp\left(-(\hatell_{t,a}-b_{t,a})/\beta \right)+(\hatell_{t,a}-b_{t,a})/\beta-1 \right) \right] \label{eq:adv_htlb_ftrl_reg0.5}\\
    &\le \beta\log K+ \mathbb{E}\left[\frac{1}{\beta} \sum_{t=1}^T \sum_{a \in \mathcal{A}} q_{t,a}\left(\hat{\ell}_{t,a}-b_{t,a}\right)^2 \right]\label{eq:adv_htlb_ftrl_reg1}\\
    &\le \beta\log K+ 2\mathbb{E}\left[\frac{1}{\beta} \sum_{t=1}^T \sum_{a \in \mathcal{A}} q_{t,a}\left(\hat{\ell}_{t,a}^2+b_{t,a}^2\right) \right] \notag\\
    &\le \beta\log K+ \mathbb{E}\left[\frac{2}{\beta} \sum_{t=1}^T \sum_{a \in \mathcal{A}} q_{t,a}\hat{\ell}_{t,a}^2 + \sum_{t=1}^T \sum_{a \in \mathcal{A}} q_{t,a}b_{t,a} \right] \label{eq:adv_htlb_ftrl_reg2}\,,
\end{align}
where Eq.~\eqref{eq:adv_htlb_ftrl_reg0.5} is by \Cref{lem:stab_condition_Shinji22},
Eq.~\eqref{eq:adv_htlb_ftrl_reg1} is due to Eq.~\eqref{eq:adv_htlb_ftrl_stable}, $|\hatell_{t,a}|\le s_{t, a}=\beta/2$ as well as the fact that $\exp (-x)+x-1 \leq x^2$ for any $x\ge -1$, and Eq.~\eqref{eq:adv_htlb_ftrl_reg2} is again by Eq.~\eqref{eq:adv_htlb_ftrl_stable}.
We further have
\begin{align}
\E\left[\sum_{a \in \mathcal{A}} q_{t,a} \hat{\ell}_{t,a}^2\right] & =\E\left[\sum_{a \in \mathcal{A}} q_{t,a} \ell_{t,a_t}^2\left(\phiat^{\top} S_t^{-1} \phia\right)^2 \indicator{\left|\ell_{t,a_t}\right| \leq \frac{s_{t,a}}{\left|\phia^{\top} S_t^{-1} \phiat\right|}}\right] \notag\\
& \leq \E\left[\sum_{a \in \mathcal{A}} q_{t,a}\left|\ell_{t,a_t}\right|^{\varepsilon}\left(\phiat^{\top} S_t^{-1} \phia\right)^2\left(\frac{s_{t,a}}{\left|\phia^{\top} S_t^{-1} \phiat\right|}\right)^{2-\varepsilon}\right] \notag\\
& \leq \E\left[\sum_{a \in \mathcal{A}}\sigma s_{t,a}^{2-\varepsilon} q_{t,a}\left|\phiat^{\top} S_t^{-1} \phia\right|^\varepsilon\right] \label{eq:adv_htlb_ftrl_reg3}\\
& \leq \E\left[\sum_{a \in \mathcal{A}}\sigma s_{t,a}^{2-\varepsilon} q_{t,a}\left(\phia^{\top} S_t^{-1} \phia\right)^{\frac{\varepsilon}{2}}\right] \label{eq:adv_htlb_ftrl_reg4} \\
& =\E\left[\sum_{a \in \mathcal{A}} s_{t,a} q_{t,a} b_{t,a}\right] \label{eq:adv_htlb_ftrl_reg5}\,,
\end{align}
where Eq.~\eqref{eq:adv_htlb_ftrl_reg3} is due to \Cref{def:HT}, and Eq.~\eqref{eq:adv_htlb_ftrl_reg4} is given by Jensen's inequality.

Substituting Eq.~\eqref{eq:adv_htlb_ftrl_reg5} into Eq.~\eqref{eq:adv_htlb_ftrl_reg2} shows that the FTRL regret can be bounded as 
\begin{align}
\mathbb{E}\left[\sumt\left\langle q_t-p^*, \hat{\ell}_t-b_t\right\rangle\right]
\le
\beta\log K+ \mathbb{E}\left[\frac{2}{\beta} \sum_{a \in \mathcal{A}} s_{t,a} q_{t,a} b_{t,a} + \sum_{t=1}^T \sum_{a \in \mathcal{A}} q_{t,a}b_{t,a} \right] \label{eq:adv_htlb_ftrl_regX1}\,.
\end{align}

\paragraph{Bounding Bias Term}
For the Bias Term, we show that the loss bias is indeed bound by the bonus term:
\begin{align}
\E\left[\abs{\tildeell_{t,a}-\hat{\ell}_{t,a}}\right]
& =\E\left[\abs{\tildeell_{t,a}} \indicator{|\tildeell_{t,a}|>s_{t,a}}\right]
\leq 
\E\left[|\tildeell_{t,a}|^\varepsilon s_{t,a}^{1-\varepsilon}\right] \notag\\
& \leq \E\left[\sigma s_{t,a}^{1-\varepsilon}\left|\phia^{\top} S_t^{-1} \phiat\right|^\varepsilon\right] \label{eq:adv_htlb_ftrl_reg6}\\
& =\E\left[\sigma s_{t,a}^{1-\varepsilon} \sum_{a^{\prime} \in \mathcal{A}} p_{t,a^{\prime}}\left|\phia^{\top} S_t^{-1} \phia^{\prime}\right|^\varepsilon\right] \notag\\
& \leq \E\left[\sigma s_{t,a}^{1-\varepsilon}\left(\sum_{a^{\prime} \in \mathcal{A}} p_{t,a^{\prime}}\left|\phia^{\top} S_t^{-1} \phi_{a^{\prime}}\right|^2\right)^{\varepsilon / 2}\right] \label{eq:adv_htlb_ftrl_reg7}\\
& =\E\left[\sigma s_{t,a}^{1-\varepsilon}\left(\phia^{\top} S_t^{-1} \phia\right)^{\varepsilon / 2}\right]=\E[b_{t,a}] \label{eq:adv_htlb_ftrl_reg8}\,,
\end{align}
where Eq.~\eqref{eq:adv_htlb_ftrl_reg6} follows from that $\tildeell_{t,a}=\phi_a^{\top}S_t^{-1}\phi_{a_t}\ell_{t, a_t}$ and \Cref{def:HT}, and Eq.~\eqref{eq:adv_htlb_ftrl_reg7} is by Jensen's inequality. 

Therefore, one can see that
\begin{align}
    &\quad\mathbb{E}\left[\left\langle q_t-p^*, \tildeell_t-\hat{\ell}_t+b_t\right\rangle\right]\\
    &\le \mathbb{E}\left[\suma \abs{q_{t,a}-p^*_a}\cdot\abs{\tildeell_{t,a}-\hat{\ell}_{t,a}}\right]+ \mathbb{E}\left[\suma (q_{t,a}-p^*_a)b_{t,a}\right] \notag\\
    &\le \mathbb{E}\left[\suma \abs{q_{t,a}-p^*_a}b_{t,a}\right]
    + \mathbb{E}\left[\suma (q_{t,a}-p^*_a)b_{t,a}\right] = 2\mathbb{E}\left[\sumaixas q_{t,a}b_{t,a}\right]\,. \label{eq:adv_htlb_ftrl_reg10}
\end{align}
where the second inequality is given by Eq.~\eqref{eq:adv_htlb_ftrl_reg8}.

\paragraph{Bounding Exploration Term}
Similar to the tabular case in Eq.~\eqref{eq:exp_term_tabular}, we have
\begin{align}
    \text{Exploration Term}=\mathbb{E}\left[\sumt \gamma\left\langle p_0-p^*, \ell_t\right\rangle\right]
    \le
    2\sigma^{\frac{1}{\varepsilon}}\gamma T\,.\label{eq:exp_term_adv_linear}
\end{align}

\paragraph{Overall Regret}
Substituting Eqs.~\eqref{eq:adv_htlb_ftrl_regX1}, \eqref{eq:adv_htlb_ftrl_reg10}, and \eqref{eq:exp_term_adv_linear} into Eq.~\eqref{eq:adv_htlb_ftrl_regX0} shows that
\begin{align}
    \Reg&\le \beta\log K+ \left(1-\gamma\right)\mathbb{E}\left[\frac{2}{\beta} \sumt\sum_{a \in \mathcal{A}} s_{t,a} q_{t,a} b_{t,a} + \sum_{t=1}^T \sum_{a \in \mathcal{A}} q_{t,a}b_{t,a} \right] \notag\\
    &\quad+2\left(1-\gamma\right)\mathbb{E}\left[\sumt\sumaixas q_{t,a}b_{t,a}\right] 
    +
    2\sigma^{\frac{1}{\varepsilon}}\gamma T \notag\\
    &\lesssim \beta\log K+ \left(1-\gamma\right)\mathbb{E}\left[\sumt\sum_{a \in \mathcal{A}} q_{t,a} b_{t,a}\right]  + \sigma^{\frac{1}{\varepsilon}}\gamma T \label{eq:adv_htlb_overall_reg1}\\
    &\le \beta\log K+ \mathbb{E}\left[\sumt\sum_{a \in \mathcal{A}} p_{t,a} b_{t,a}\right]  + \sigma^{\frac{1}{\varepsilon}}\gamma T \label{eq:adv_htlb_overall_reg2}\\
    &\lesssim\beta\log K+ \beta^{1-\varepsilon}\mathbb{E}\left[\sumt\sum_{a \in \mathcal{A}}\sigma p_{t,a} (\phia^{\top} S_t^{-1} \phia)^{\frac{\varepsilon}{2}}\right]  + \sigma^{\frac{1}{\varepsilon}}\gamma T \label{eq:adv_htlb_overall_reg3}\\
    &\le\beta\log K+ \beta^{1-\varepsilon}\sigma\mathbb{E}\left[\sumt \left(\sum_{a \in \mathcal{A}} p_{t,a} \phia^{\top} S_t^{-1} \phia\right)^{\frac{\varepsilon}{2}}\right]  + \sigma^{\frac{1}{\varepsilon}}\gamma T \label{eq:adv_htlb_overall_reg4}\\
    &=\beta\log K+ \beta^{1-\varepsilon}\sigma d^{\frac{\varepsilon}{2}} T +\sigma^{\frac{1}{\varepsilon}} \gamma T \notag\\
    &\lesssim\beta\log K+ \beta^{1-\varepsilon}\sigma d^{\frac{\varepsilon}{2}} T + \sigma^{\frac{3}{\varepsilon}}\beta^{-2} d T \label{eq:adv_htlb_overall_reg5}\\
    &\lesssim \sigma^{\frac{1}{\varepsilon}} (\log K)^{1-\frac{1}{\varepsilon}} d^{\frac{1}{2}} T^{\frac{1}{\varepsilon}}+ \sigma^{\frac{3}{\varepsilon}}(\log K)^{\frac{2}{\varepsilon}} T^{1-\frac{2}{\varepsilon}}\label{eq:adv_htlb_overall_reg6}\,,
\end{align}
where Eq.~\eqref{eq:adv_htlb_overall_reg1} follows from setting $s_{t, a}=\beta/2$ for all $a\in\gA$, Eq.~\eqref{eq:adv_htlb_overall_reg2} holds as $p_t=\left(1-\gamma\right)q_t+\gamma p_0$, Eq.~\eqref{eq:adv_htlb_overall_reg3} is by substituting $b_{t,a}=\sigma s_{t, a}^{1-\varepsilon} (\phia^{\top} S_t^{-1} \phia)^{\varepsilon/2}$ and $s_{t, a}=\beta/2$, Eq.~\eqref{eq:adv_htlb_overall_reg4} is due to Jensen's inequality, Eq.~\eqref{eq:adv_htlb_overall_reg5} uses $\gamma = 4 \sigma^{\frac{2}{\varepsilon}} d \beta^{-2}$, and Eq.~\eqref{eq:adv_htlb_overall_reg6} follows by setting $\beta=\left(\frac{\log K}{\sigma d^{\varepsilon/2}T}\right)^{-1/\varepsilon}$.
\end{proof}

\section{Omitted Proofs in Section~\ref{sec:bobw_htlb}}
In Appendix \ref{sec:app:htddlr}, we will first introduce some key properties of the data-dependent learning rate defined in Eq.~\eqref{eq:beta_t} of \Cref{sec:bobw_htlb}. Then, in Appendix \ref{app:sec:bobw_htlb}, we will present the formal proof of \Cref{thm:bobw_lb}, which indicates that this data-dependent learning rate schedule, together with the variance-reduced least squares loss estimator, leads to the BOBW regret guarantee for heavy-tailed linear bandits.

\subsection{Proof of Proposition~\ref{prop:prop1}}\label{sec:app:htddlr}
Recall $\beta_t = \beta_{t-1} + \frac{1}{\hat{h}_t} \left( \beta_{t-1}^{1 - \varepsilon} z_{t-1} + \beta_{t-1}^{-2} w_{t-1} \right)$ in Eq.~\eqref{eq:beta_t}.
In addition to this data-dependent learning rate, this section will also consider the following learning rate:
\begin{align}\label{eq:meq9}
    \beta_t=\beta_{t-1}+\frac{1}{\hath_t}\left(\beta_t^{1-\varepsilon}z_t+\beta_t^{-2}w_t \right)\,.
\end{align}
Note that in contrast to the learning rate in Eq.~\eqref{eq:beta_t}, the learning rate in Eq.~\eqref{eq:meq9} also requires the knowledge of $z_t$ and $w_t$, which are actually not available when the algorithm determines learning rate $\beta_t$ at the end of round $t-1$. Therefore, the learning rate in Eq.~\eqref{eq:meq9} is included here only for the theoretical interest.

For ease of discussion, we define 
\begin{align}\label{eq:def_F}
F\left(\beta_{1:T}; z_{1:T}, w_{1:T}, h_{1:T} \right)&=\sumt \left( \left(\beta_t - \beta_{t-1}\right) h_t + \beta_t^{1 - \varepsilon} z_t + \beta_t^{-2} w_t \right)
\end{align}
and
\begin{align*}
    G\left( z_{1: T},w_{1: T}, h_{1: T}\right)&= \sum_{t=1}^{T} z_t \left(\sum_{s \le t} \frac{z_s}{h_s} \right)^{\frac{1-\varepsilon}{\varepsilon}}+ \sum_{t=1}^{T} w_t \left(\sum_{s \le t} \frac{w_s}{h_s} \right)^{-\frac{2}{3}}\,.
\end{align*}
We then have the following lemma.
\begin{lem}\label{lem:F2G}
Suppose $h_t \leq \hath_t$ holds for all $t\in[T]$. Then, the learning rate $\beta_t$ defined in Eq.~\eqref{eq:meq9} satisfies 
  \begin{align*}
      F\left(\beta_{1: T} ; z_{1: T},w_{1: T}, h_{1: T}\right) \le 2 G\left(z_{1: T}, w_{1: T}, h_{1: T}\right)\,,
  \end{align*}
  and the learning rate $\beta_t$ defined in Eq.~\eqref{eq:beta_t} satisfies 
  \begin{align*}
      F\left(\beta_{1: T} ; z_{1: T}, w_{1: T}, h_{1: T}\right)
      \le 4 G\left(z_{1:T}, w_{1: T}, \hath_{2: T+1}\right)+8\left( z_{\max } \beta_1^{1-\varepsilon}+w_{\max } \beta_1^{-2} \right)+\beta_1 h_1\,.
  \end{align*}
\end{lem}
\begin{proof}
    Consider learning rate $\beta_t$ in Eq.~\eqref{eq:meq9}. Then
    \begin{align}\label{eq:meq41}
        F\left(\beta_{1: T} ; z_{1: T},w_{1: T}, h_{1: T}\right)&\coloneqq\sum_{t=1}^{T}\left(z_t \beta_t^{1-\varepsilon}+w_t \beta_t^{-2}+\left(\beta_t-\beta_{t-1}\right) h_t\right)\notag\\
        &\leq \sum_{t=1}^{T}\left(z_t \beta_t^{1-\varepsilon}+w_t \beta_t^{-2}+\left(\beta_t-\beta_{t-1}\right) \hath_t\right)\notag\\
        &=2 \sum_{t=1}^{T}\left( z_t \beta_t^{1-\varepsilon}+ w_t \beta_t^{-2}\right)\,.
    \end{align}
    Also from Eq.~\eqref{eq:meq9}, we have
    \begin{align}\label{eq:meq42}
        \beta_t^{\varepsilon}\ge\beta_t^{\varepsilon-1} \beta_{t-1}+\frac{z_t}{\hath_t} 
        \ge \beta_{t-1}^{\varepsilon}+\frac{z_t}{\hath_t} \ge \sum_{1\le s \le t} \frac{z_s}{\hath_s}\,.
    \end{align}
    Analogously, from Eq.~\eqref{eq:meq9}, we have 
    \begin{align}\label{eq:meq42_stab2}
        \beta_t^3\ge\beta_t^{2} \beta_{t-1}+\frac{w_t}{\hath_t} 
        \ge \beta_{t-1}^{3}+\frac{w_t}{\hath_t} \ge \sum_{1\le s \le t} \frac{w_s}{\hath_s}\,.
    \end{align}
    Combining Eqs.~\eqref{eq:meq41}, \eqref{eq:meq42} and \eqref{eq:meq42_stab2} leads to 
    \begin{align*}
        F\left(\beta_{1: T} ;z_{1: T}, w_{1: T}, h_{1: T}\right)&\le 2 \sum_{t=1}^{T}\left( z_t \beta_t^{1-\varepsilon}+ w_t \beta_t^{-2}\right)\\
        &\le 2\sum_{t=1}^{T} z_t \left(\sum_{s \le t} \frac{z_s}{\hath_s} \right)^{(1-\varepsilon)/\varepsilon} + 2\sum_{t=1}^{T}w_t \left(\sum_{s \le t} \frac{w_s}{\hath_s} \right)^{-2/3} \notag\\
        &\eqqcolon 2 G\left(z_{1: T}, w_{1: T}, h_{1: T}\right)\,.
    \end{align*}
    
    We now consider bounding $F\left(\beta_{1: T} ; z_{1: T},w_{1: T}, h_{1: T}\right)$ for $\beta_t$ defined in Eq.~\eqref{eq:beta_t}. By definition, we have
    \begin{align}\label{eq:meq43}
        &\quad F\left(\beta_{1: T} ;z_{1: T}, w_{1: T}, h_{1: T}\right)\notag\\
        &\le z_1 \beta_1^{1-\varepsilon}+w_1 \beta_1^{-2}+\beta_1 h_1
        +\sum_{2 \le t \le T}\left(z_t \beta_t^{1-\varepsilon}+w_t \beta_t^{-2}+\left(\beta_t-\beta_{t-1}\right) \hath_t\right) \notag\\
&=z_1 \beta_1^{1-\varepsilon}+w_1 \beta_1^{-2}+\beta_1 h_1
+\sum_{2 \le t \le T}\left(z_t \beta_t^{1-\varepsilon}+z_{t-1} \beta_{t-1}^{1-\varepsilon}\right)
+\sum_{2 \le t \le T}\left(w_t \beta_t^{-2}+w_{t-1} \beta_{t-1}^{-2}\right) \notag\\
&\le \beta_1 h_1+2 \sum_{t=1}^T z_t \beta_t^{1-\varepsilon}+2 \sum_{t=1}^T w_t \beta_t^{-2}\,.
    \end{align}
    Also from the definition of $\beta_t$ in Eq.~\eqref{eq:beta_t}, we have
    \begin{align*}
        \beta_t^{\varepsilon}\ge\beta_t \cdot \beta_{t-1}^{\varepsilon-1} \ge \beta_{t-1}^{\varepsilon}+\frac{z_{t-1}}{\hath_t} \ge \beta_1^{\varepsilon}+\sum_{s=2}^t \frac{z_{s-1}}{\hath_s}\,,
    \end{align*}
    which implies that $\bar{\beta}_t\coloneqq\left(\beta_1^{\varepsilon}+\sum_{s=2}^t \frac{z_{s-1}}{\hath_s}\right)^{\frac{1}{\varepsilon}}$ satisfies $\bar{\beta}_t\le \beta_t$.

    Denote $T=\left\{t \in[T]:\bar{\beta}_{t+1} \le \sqrt{2} \bar{\beta}_t\right\}$ and $\gT^c=[T]\setminus \gT$. Then
    \begin{align}\label{eq:beta2_g1}
&\quad
\sum_{t=1}^{T} z_t \beta_t^{1-\varepsilon} 
\le 
\sum_{t=1}^{T} z_t \bar{\beta}_t^{1-\varepsilon}
=
\sum_{t \in \gT} z_t \bar{\beta}_t^{1-\varepsilon}+\sum_{t \in \gT^c} z_t \bar{\beta}_t^{1-\varepsilon}\notag\\
&\le 
2^{\frac{\varepsilon-1}{2}} \sum_{t \in \gT} z_t \bar{\beta}_{t+1}^{1-\varepsilon}+z_{\max} \sum_{t \in \gT^c} \bar{\beta}_t^{1-\varepsilon} 
\le 2^{\frac{\varepsilon-1}{2}} \sum_{t \in \gT} z_t \bar{\beta}_{t+1}^{1-\varepsilon}+z_{\max } \beta_1^{1-\varepsilon} \sum_{s=0}^{\infty}\left(\frac{1}{\sqrt{2}}\right)^s \notag\\
&\le 
2^{\frac{\varepsilon-1}{2}} \sum_{t \in \gT} z_t \bar{\beta}_{t+1}^{1-\varepsilon}+z_{\max } \beta_1^{1-\varepsilon} \frac{1}{1-1/ \sqrt{2}} \notag\\
&\le 2^{\frac{\varepsilon-1}{2}} \sum_{t=1}^T z_t\left(\beta_1^{\varepsilon}+\sum_{s=2}^{t+1} \frac{z_{s-1}^{\prime}}{\hath_s}\right)^{(1-\varepsilon) / \varepsilon}+ (2+\sqrt{2})z_{\max } \beta_1^{1-\varepsilon} \notag\\
&\le 2^{\frac{\varepsilon-1}{2}} \sum_{t=1}^T z_t\left(\sum_{s=1}^{t} \frac{z_{s}}{\hath_{s+1}}\right)^{(1-\varepsilon) / \varepsilon}+ (2+\sqrt{2})z_{\max } \beta_1^{1-\varepsilon} \,.
    \end{align}
Similarly, we have
\begin{align}\label{eq:beta2_g2}
    \sum_{t=1}^{T} w_t \beta_t^{-2} \le 2 \sum_{t=1}^T w_t\left(\sum_{s=1}^{t} \frac{w_{s}}{\hath_{s+1}}\right)^{-2 / 3}+ (2+\sqrt{2})w_{\max} \beta_1^{-2}\,.
\end{align}
    Combining Eqs.~\eqref{eq:meq43}, \eqref{eq:beta2_g1} and \eqref{eq:beta2_g2} leads to
  \begin{align*}
      F\left(\beta_{1: T} ; z_{1: T}, w_{1: T}, h_{1: T}\right)
      \le 4 G\left(z_{1:T}, w_{1: T}, \hath_{2: T+1}\right)+8\left( z_{\max } \beta_1^{1-\varepsilon}+w_{\max } \beta_1^{-2} \right)+\beta_1 h_1\,.
  \end{align*}
\end{proof}

\begin{lem}\label{lem:nlem3}
For any $J\in \sZ_{\ge 0}$, it holds that
\begin{align*}
 G\left(z_{1:T}, w_{1: T}, h_{1: T}\right)
    \le
    \min&\left\{
    \varepsilon  2^{1-\frac{1}{\varepsilon}}\left(J^{\varepsilon-1} \sum_{t=1}^{T} h_t^{\varepsilon-1} z_t\right)^{\frac{1}{\varepsilon}}+\varepsilon\left[\left(2^{-J}\right)^{\varepsilon-1} T \cdot h_{\max} z_{\max}\right]^{\frac{1}{\varepsilon}}\right.\\
&\quad +3\cdot  2^{\frac{2}{3}}\left(J^{2} \sum_{t=1}^{T} h_t^{2} w_t\right)^{\frac{1}{3}}+3\left[\left(2^{-J}\right)^{2} T \cdot h_{\max} w_{\max}\right]^{\frac{1}{3}}\,,\\
&\quad\left.\varepsilon h_{\max }^{1-\frac{1}{\varepsilon}}\left(\sum_{t \le T} z_t\right)^{\frac{1}{\varepsilon}}+ 3 h_{\max }^{\frac{2}{3}}\left(\sum_{t \le T} w_t\right)^{\frac{1}{3}}\right\}\,.
\end{align*}
\end{lem}
\begin{proof}
Fix arbitrary $J\in \sZ_{\ge 0}$.
In what follows, we let $\theta_0>\theta_1>\cdots>\theta_J>0$ be an arbitrary positive and monotone decreasing
sequence such that $\theta_0 \geq h_{\max }$. 
For all $j\in[J]$, define $\mathcal{T}_j=\left\{t \in[T]: \theta_{j-1} \geq h_t>\theta_j\right\}$ and $\mathcal{T}_{J+1}=\left\{t \in[T]: \theta_J \geq h_t\right\}$. 

Recall
$G\left( z_{1: T},w_{1: T}, h_{1: T}\right)\coloneqq 
\sum_{t=1}^{T} z_t \left(\sum_{s \le t} \frac{z_s}{h_s} \right)^{(1-\varepsilon)/\varepsilon}+ \sum_{t=1}^{T} w_t \left(\sum_{s \le t} \frac{w_s}{h_s} \right)^{-2/3}$. We first bound $\sum_{t=1}^{T} z_t \left(\sum_{s \le t} \frac{z_s}{h_s} \right)^{(1-\varepsilon)/\varepsilon}$ as follows:
    \begin{align}
&\quad\sum_{t=1}^T z_t \left(\sum_{s \leq t} \frac{z_s}{h_s}\right)^{\frac{(1-\varepsilon)}{\varepsilon}}
=\sum_{j=1}^{J+1} \sum_{t \in \gT_j} z_t \left(\sum_{s \leq t} \frac{z_s}{h_s}\right)^{\frac{(1-\varepsilon)}{\varepsilon}} 
\le 
\sum_{j=1}^{J+1} \sum_{t \in \gT_j} z_t \left(\sum_{s \in[t] \cap \gT_j} \frac{z_s}{h_s}\right)^{\frac{(1-\varepsilon)}{\varepsilon}}  \notag\\
&\le \sum_{j=1}^{J+1} \sum_{t \in \gT_j} z_t \left(\sum_{s \in[t] \cap \gT_j} \frac{z_s}{\theta_{j-1}}\right)^{\frac{(1-\varepsilon)}{\varepsilon}}  
=
\sum_{j=1}^{J+1} \theta_{j-1}^{(\varepsilon-1) / \varepsilon} \sum_{t \in \gT_j} z_t\left(\sum_{s \in[t] \cap \gT_j} z_s\right)^{\frac{1}{\varepsilon}-1}  \notag\\
&\le\sum_{j=1}^{J+1} \theta_{j-1}^{(\varepsilon-1) / \varepsilon} \sum_{t \in \gT_j} z_t\cdot \frac{\varepsilon}{z_t}\left[\left(\sum_{s \in[t] \cap \gT_j} z_s\right)^{\frac{1}{\varepsilon}}-\left(\sum_{s \in[t-1] \cap \gT_j} z_s\right)^{\frac{1}{\varepsilon}}\right]\label{eq:local_dskfs}\\
&\le \varepsilon\sum_{j=1}^{J+1} \theta_{j-1}^{(\varepsilon-1) / \varepsilon} \left(\sum_{t \in  \gT_j} z_t\right)^{1 / \varepsilon}\,,\notag
    \end{align}
    where Eq.~\eqref{eq:local_dskfs} is due to that 
    \begin{align*}
         \frac{z_t}{\varepsilon}\left(\sum_{s \in[t] \cap \gT_j} z_s\right)^{\frac{1}{\varepsilon}-1} \leq\left(\sum_{s \in[t] \cap \gT_j} z_s\right)^{\frac{1}{\varepsilon}}-\left(\sum_{s \in[t-1] \cap \gT_j} z_s\right)^{\frac{1}{\varepsilon}}\,,
    \end{align*}
    as $x\mapsto x^{1/\varepsilon}$ is a concave function.
Similarly, we have 
\begin{align*}
    \sum_{t=1}^{T} w_t \left(\sum_{s \le t} \frac{w_s}{h_s} \right)^{-\frac{2}{3}}\le 3\sum_{j=1}^{J+1} \theta_{j-1}^{2 / 3} \left(\sum_{t \in  \gT_j} w_t\right)^{\frac{1}{3}}\,.
\end{align*}
Therefore,
\begin{align*}
    G\left( z_{1: T},w_{1: T}, h_{1: T}\right)\le \varepsilon\sum_{j=1}^{J+1} \theta_{j-1}^{(\varepsilon-1) / \varepsilon} \left(\sum_{t \in  \gT_j} z_t\right)^{\frac{1}{\varepsilon}}+3\sum_{j=1}^{J+1} \theta_{j-1}^{2 / 3} \left(\sum_{t \in  \gT_j} w_t\right)^{\frac{1}{3}}\,.
\end{align*}
Based on the above display, by setting $J=0$ and $\theta_0=h_{\max}$, we have
\begin{align}
    G\left( z_{1: T},w_{1: T}, h_{1: T}\right) 
    \le \varepsilon h_{\max }^{1-\frac{1}{\varepsilon}}\left(\sum_{t \le T} z_t\right)^{\frac{1}{\varepsilon}}+ 3 h_{\max }^{\frac{2}{3}}\left(\sum_{t \le T} w_t\right)^{\frac{1}{3}}\,.\label{eq:G_upper_bound1}
\end{align}
On the other hand, by setting $\theta_j=2^{-j} h_{\max }$, we have
\begin{align}
 G\left( z_{1: T},w_{1: T}, h_{1: T}\right) 
&\le \varepsilon \sum_{j=1}^J\left[\left(\frac{\theta_{j-1}}{\theta_j}\right)^{\varepsilon-1} \sum_{t \in \gT_j} h_t^{\varepsilon-1} z_t\right]^{\frac{1}{\varepsilon}}+\varepsilon \theta_J^{1-\frac{1}{\varepsilon}}\left(\sum_{t \in \gT_{J+1}} z_t\right)^{\frac{1}{\varepsilon}} \notag\\
&\quad+3 \sum_{j=1}^J\left[\left(\frac{\theta_{j-1}}{\theta_j}\right)^{2} \sum_{t \in \gT_j} h_t^{2} w_t\right]^{\frac{1}{3}}+3 \theta_J^{\frac{2}{3}}\left(\sum_{t \in \gT_{J+1}} w_t\right)^{\frac{1}{3}} \notag\\
&=\varepsilon  2^{1-\frac{1}{\varepsilon}} \sum_{j=1}^J\left(\sum_{t \in \gT_j} h_t^{\varepsilon-1} z_t\right)^{\frac{1}{\varepsilon}}+\varepsilon\left[\left(2^{-J}\right)^{\varepsilon-1} h_{\max } \sum_{t \in \gT_{J+1}} z_t\right]^{\frac{1}{\varepsilon}} \notag\\
&\quad+3 \cdot 2^{\frac{2}{3}} \sum_{j=1}^J\left(\sum_{t \in \gT_j} h_t^{2} w_t\right)^{\frac{1}{3}}+3\left[\left(2^{-J}\right)^{2} h_{\max } \sum_{t \in \gT_{J+1}} w_t\right]^{\frac{1}{3}} \notag\\
&\le \varepsilon  2^{1-\frac{1}{\varepsilon}}\left(J^{\varepsilon-1} \sum_{t=1}^{T} h_t^{\varepsilon-1} z_t\right)^{\frac{1}{\varepsilon}}+\varepsilon\left[\left(2^{-J}\right)^{\varepsilon-1} T \cdot h_{\max} z_{\max}\right]^{\frac{1}{\varepsilon}} \notag\\
&\quad+3\cdot  2^{\frac{2}{3}}\left(J^{2} \sum_{t=1}^{T} h_t^{2} w_t\right)^{\frac{1}{3}}+3\left[\left(2^{-J}\right)^{2} T \cdot h_{\max} w_{\max}\right]^{\frac{1}{3}}\,.\label{eq:G_upper_bound2}
\end{align}

Combining Eq.~\eqref{eq:G_upper_bound1} and Eq.~\eqref{eq:G_upper_bound2} concludes the proof.
\end{proof}

As an immediate consequence of \Cref{lem:F2G} and \Cref{lem:nlem3}, we have the following theorem.
\begin{thm}\label{thm:thm3}
    Suppose $h_t \le \hath_t$ for all $t\in[T]$.
    For learning rate $\beta_t$ in Eq.~\eqref{eq:meq9} and any $\delta,\delta^{\prime}\ge 1/T$, we have
    \begin{align*}
        &\quad F\left(\beta_{1: T} ; z_{1: T}, w_{1: T}, h_{1: T}\right)\\
        &=\gO\left(\min\left\{\left[\left(\frac{\log (T \delta)}{\varepsilon-1}\right)^{\varepsilon-1} \sum_{t=1}^T \hath_t^{\varepsilon-1} z_t+\frac{\hath_{\max } z_{\max }}{\delta}\right]^{\frac{1}{\varepsilon}}
        +\left[\frac{\log^2 (T \delta^{\prime})}{4} \sum_{t=1}^T \hath_t^{2} w_t+\frac{\hath_{\max } w_{\max }}{\delta^{\prime}}\right]^{\frac{1}{3}}, \right.\right.\\
        &\quad\left.\left.\hath_{\max }^{1-\frac{1}{\varepsilon}}\left(\sum_{t \le T} z_t\right)^{\frac{1}{\varepsilon}}+\hath_{\max }^{\frac{2}{3}}\left(\sum_{t \le T} w_t\right)^{\frac{1}{3}}\right\}\right)\,.
    \end{align*}
    For learning rate $\beta_t$ in Eq.~\eqref{eq:beta_t} and any $\delta,\delta^{\prime}\ge 1/T$, we have
    \begin{align*}
        &\quad F\left(\beta_{1: T} ; z_{1: T}, w_{1: T}, h_{1: T}\right)\\
        &=\gO\left(\min\left\{\left[\left(\frac{\log (T \delta)}{\varepsilon-1}\right)^{\varepsilon-1} \sum_{t=1}^T \hath_{t+1}^{\varepsilon-1} z_t+\frac{\hath_{\max } z_{\max }}{\delta}\right]^{\frac{1}{\varepsilon}}
        +\left[\frac{\log^2 (T \delta^{\prime})}{4} \sum_{t=1}^T \hath_{t+1}^{2} w_t+\frac{\hath_{\max } w_{\max }}{\delta^{\prime}}\right]^{\frac{1}{3}},\right.\right.\\
        &\quad\left.\left.\hath_{\max }^{1-\frac{1}{\varepsilon}}\left(\sum_{t \le T} z_t\right)^{\frac{1}{\varepsilon}}+\hath_{\max }^{\frac{2}{3}}\left(\sum_{t \le T} w_t\right)^{\frac{1}{3}}\right\}
        +z_{\max } \beta_1^{1-\varepsilon}+w_{\max } \beta_1^{-2}+\beta_1 \hath_1\right)\,.
    \end{align*}
\end{thm}

We are now ready to introduce the proof of \Cref{prop:prop1}.
\begin{proof}[Proof of \Cref{prop:prop1}]
To start with, by Eq.~\eqref{eq:regret_F} and Eq.~\eqref{eq:stable_update}, we have
\begin{align}
    \Reg 
    &\lesssim  
\E\left[\sumt \left( \left(\beta_t - \beta_{t-1}\right) h_t + \beta_t^{1 - \varepsilon} z_t + \beta_t^{-2} w_t \right)\right]+\bar{\beta} \bar{h} \notag\\
    &\lesssim  
\E\left[\sumt \left( \left(\beta_t - \beta_{t-1}\right) h_{t-1} + \beta_t^{1 - \varepsilon} z_t + \beta_t^{-2} w_t \right)\right]+\bar{\beta} \bar{h} \notag\\
    &=
    \mathbb{E}\left[F\left(\beta_{1: T} ;z_{1: T}, w_{1: T}, h_{0: T-1}\right)\right]+\bar{\beta} \bar{h}\,,\label{eq:prop1_step1}
\end{align}
where the last step follows from the definition of $F$ in Eq.~\eqref{eq:def_F}.

\paragraph{Regret in Adversarial Regime}
Applying Theorem \ref{thm:thm3} in Eq.~\eqref{eq:prop1_step1}, one can see that
\begin{align*}
    &\quad\Reg\\
&= \gO\left(\mathbb{E}\left[
\hath_{\max }^{1-\frac{1}{\varepsilon}}\left(\sum_{t \le T} z_t\right)^{\frac{1}{\varepsilon}}+\hath_{\max }^{\frac{2}{3}}\left(\sum_{t \le T} w_t\right)^{\frac{1}{3}}
\right]+z_{\max } \beta_1^{1-\varepsilon}+w_{\max } \beta_1^{-2}+\beta_1 \hath_1+\bar{\beta} \bar{h}\right) \\
&= \gO\left(h_1^{1-\frac{1}{\varepsilon}}\left(z_{\max} T\right)^{\frac{1}{\varepsilon}}+h_1^{\frac{2}{3}}\left(w_{\max} T\right)^{\frac{1}{3}}+\kappa\right)\,.
\end{align*}

\paragraph{Regret in Adversarial Regime with Self-bounding Constraint}
Let $Q_T\coloneqq \mathbb{E}\left[\sum_{t=1}^T \langle \Delta, p_t \rangle \right]$. Also by Theorem \ref{thm:thm3}, for any $\delta, \delta^{\prime}\ge 1/T$, we have
\begin{align}
    &\quad\Reg\notag\\
    &=\mathbb{E}\left[
    \left(\left(\frac{\log (T \delta)}{\varepsilon-1}\right)^{\varepsilon-1} \sum_{t=1}^T {h}_t^{\varepsilon-1} z_t+\frac{{h}_{\max } z_{\max }}{\delta}\right)^{\frac{1}{\varepsilon}}
        +\left(\frac{\log^2 (T \delta^{\prime})}{4} \sum_{t=1}^T {h}_t^{2} w_t+\frac{{h}_{\max } w_{\max }}{\delta^{\prime}}\right)^{\frac{1}{3}}
    \right]\notag\\
    &\quad+z_{\max } \beta_1^{1-\varepsilon}+w_{\max } \beta_1^{-2}+\beta_1 h_0 +\bar{\beta} \bar{h}\notag\\
    &\le \left(\left(\frac{\log (T \delta)}{\varepsilon-1}\right)^{\varepsilon-1}\mathbb{E}\left[ \sum_{t=1}^T {h}_t^{\varepsilon-1} z_t\right]+\frac{{h}_{\max } z_{\max }}{\delta}\right)^{\frac{1}{\varepsilon}} +\left(\frac{\log^2 (T \delta^{\prime})}{4}\mathbb{E}\left[ \sum_{t=1}^T {h}_t^{2} w_t\right]+\frac{{h}_{\max } w_{\max }}{\delta^{\prime}}\right)^{\frac{1}{3}} +\kappa \label{eq:local_dj23kjf}\\
    &\le \left(\left(\frac{\log (T \delta)}{\varepsilon-1}\right)^{\varepsilon-1}
    \zeta(\Delta)Q_T
    +\frac{{h}_{\max } z_{\max }}{\delta}\right)^{\frac{1}{\varepsilon}}
    +\left(\frac{\log^2 (T \delta^{\prime})}{4}
    \omega(\Delta)Q_T
    +\frac{{h}_{\max } w_{\max }}{\delta^{\prime}}\right)^{\frac{1}{3}}+\kappa \label{eq:local_dj2sdjf}\\
    &\lesssim \left(\frac{\log (T \delta)}{\varepsilon-1}\right)^{\frac{\varepsilon-1}{\varepsilon}}
    \zeta(\Delta)^{\frac{1}{\varepsilon}}Q_T^{\frac{1}{\varepsilon}}
    +\left(\frac{{h}_{\max } z_{\max }}{\delta}\right)^{\frac{1}{\varepsilon}}
    +\frac{\log^{\frac{2}{3}} (T \delta^{\prime})}{2^{\frac{2}{3}}}
    \omega(\Delta)^{\frac{1}{3}}Q_T^{\frac{1}{3}}
    +\left(\frac{{h}_{\max } w_{\max }}{\delta^{\prime}}\right)^{\frac{1}{3}}+\kappa\,,\notag
\end{align}
where Eq.~\eqref{eq:local_dj23kjf} follows from Jensen's inequality and the fact that $\varepsilon\in (1,2]$, and Eq.~\eqref{eq:local_dj2sdjf} is by Eq.~\eqref{eq:round_wise_0}.

Using the definition of $(\Delta, C, T)$-self-bounding constraint in \Cref{def:def1}, for any $\lambda \in (0,1]$, we further have
\begin{align}
    \Reg
    &= (1+\lambda)\Reg-\lambda \Reg\notag\\
    &\lesssim (1+\lambda)\left[\left(\frac{\log (T \delta)}{\varepsilon-1}\right)^{\frac{\varepsilon-1}{\varepsilon}}
    \zeta(\Delta)^{\frac{1}{\varepsilon}}Q_T^{\frac{1}{\varepsilon}}
    +\left(\frac{{h}_{\max } z_{\max }}{\delta}\right)^{\frac{1}{\varepsilon}}
    +\log^{\frac{2}{3}} (T \delta^{\prime})
    \omega(\Delta)^{\frac{1}{3}}Q_T^{\frac{1}{3}}\right.\notag\\
    &\quad\left.+\left(\frac{{h}_{\max } w_{\max }}{\delta^{\prime}}\right)^{\frac{1}{3}}+\kappa \right]-\lambda(Q_T- C)\notag\\
    &\lesssim \left[(1+\lambda)\left(\frac{\log (T \delta)}{\varepsilon-1}\right)^{\frac{\varepsilon-1}{\varepsilon}}
    \zeta(\Delta)^{\frac{1}{\varepsilon}} \right]^{\frac{\varepsilon}{\varepsilon-1}}\left(\varepsilon\cdot \frac{\lambda}{2}\right)^{\frac{1}{1-\varepsilon}}+\left(\frac{{h}_{\max } z_{\max}}{\delta}\right)^{\frac{1}{\varepsilon}} \notag\\
    &\quad+\left[(1+\lambda)\log^{\frac{2}{3}} (T \delta^{\prime})
    \omega(\Delta)^{\frac{1}{3}} \right]^{\frac{3}{2}}\left(3\cdot \frac{\lambda}{2}\right)^{-\frac{1}{2}}+\left(\frac{{h}_{\max } w_{\max }}{\delta^{\prime}}\right)^{\frac{1}{3}}+\kappa+\lambda C \label{eq:local_dk834}\\
    &= \zeta(\Delta)^{\frac{1}{\varepsilon-1}}\left(\frac{\varepsilon}{2}\right)^{\frac{1}{1-\varepsilon}} \frac{\log (T \delta)}{\varepsilon-1}\left(\frac{(1+\lambda)^{\varepsilon}}{\lambda}\right)^{\frac{1}{\varepsilon-1}}
    +\left(\frac{{h}_{\max } z_{\max }}{\delta}\right)^{\frac{1}{\varepsilon}} \notag\\
    &\quad+\omega(\Delta)^{\frac{1}{2}}\left(\frac{3}{2}\right)^{-\frac{1}{2}} \log (T \delta^{\prime})\left(\frac{(1+\lambda)^{3}}{\lambda}\right)^{\frac{1}{2}}+\left(\frac{{h}_{\max } w_{\max }}{\delta^{\prime}}\right)^{\frac{1}{3}}+\kappa+\lambda C\notag\\
    &\le \underbrace{\zeta(\Delta)^{\frac{1}{\varepsilon-1}}\left(\frac{\varepsilon}{2}\right)^{\frac{1}{1-\varepsilon}} \frac{\log (T \delta)}{\varepsilon-1}}_{\eqqcolon \iota}
    \left(\frac{1}{\lambda}\right)^{\frac{1}{\varepsilon-1}}2^{\frac{\varepsilon}{\varepsilon-1}}
    +\left(\frac{{h}_{\max } z_{\max }}{\delta}\right)^{\frac{1}{\varepsilon}}\notag\\
    &\quad+\underbrace{\omega(\Delta)^{\frac{1}{2}}\left(\frac{3}{2}\right)^{-\frac{1}{2}} \log (T \delta^{\prime})}_{\eqqcolon \iota^{\prime}}
    \left(\frac{1}{\lambda}\right)^{\frac{1}{2}}2^{\frac{3}{2}}+\left(\frac{{h}_{\max } w_{\max }}{\delta^{\prime}}\right)^{\frac{1}{3}}
    +\kappa+\lambda C \label{eq:local_dk8345}\\
    &=\iota\left(\frac{1}{\lambda}\right)^{\frac{1}{\varepsilon-1}}2^{\frac{\varepsilon}{\varepsilon-1}}
    +\left(\frac{{h}_{\max } z_{\max }}{\delta}\right)^{\frac{1}{\varepsilon}}+\iota^{\prime}
    \left(\frac{1}{\lambda}\right)^{\frac{1}{2}}2^{\frac{3}{2}}
    +\left(\frac{{h}_{\max } w_{\max }}{\delta^{\prime}}\right)^{\frac{1}{3}}
    +\kappa+\lambda C \notag\,,
\end{align}
where 
 Eq.~\eqref{eq:local_dk834} follows from $f(x)=a x^{\frac{1}{c}}-b x=\gO(a^{\frac{c}{c-1}}(cb)^{\frac{1}{1-c}})$ for any $a,b,x>0$ and $c>1$ as $f(x)$ is a concave function and Eq.~\eqref{eq:local_dk8345} uses the condition $\lambda\in (0,1]$.

 To proceed with the above, we consider two cases. If $C=0$, setting $\lambda=1$ guarantees that 
     \begin{align*}
         \Reg
         \lesssim 2^{\frac{\varepsilon}{\varepsilon-1}}\iota+\iota^{\prime}+\left(\frac{{h}_{\max } z_{\max }}{\delta}\right)^{\frac{1}{\varepsilon}}+\left(\frac{{h}_{\max } w_{\max }}{\delta^{\prime}}\right)^{\frac{1}{3}}+\kappa\,.
     \end{align*}
On the other hand, if $C>0$, setting $\lambda=2C^{\frac{1-\varepsilon}{\varepsilon}} (\iota+\iota^{\prime})^{\frac{\varepsilon-1}{\varepsilon}}$ guarantees that 
     \begin{align*}
         \Reg\lesssim C^{\frac{1}{\varepsilon}}(\iota+\iota^{\prime})^{\frac{\varepsilon-1}{\varepsilon}}
         +\left(\frac{{h}_{\max } z_{\max }}{\delta}\right)^{\frac{1}{\varepsilon}}+\left(\frac{{h}_{\max } w_{\max }}{\delta^{\prime}}\right)^{\frac{1}{3}}+\kappa\,.
     \end{align*}
 Hence, in either case, it holds that  
     \begin{align}
         \Reg&\lesssim 2^{\frac{\varepsilon}{\varepsilon-1}}\iota+\iota^{\prime}+C^{\frac{1}{\varepsilon}}(\iota+\iota^{\prime})^{\frac{\varepsilon-1}{\varepsilon}}
         +\left(\frac{{h}_{\max } z_{\max }}{\delta}\right)^{\frac{1}{\varepsilon}}+\left(\frac{{h}_{\max } w_{\max }}{\delta^{\prime}}\right)^{\frac{1}{3}}+\kappa\notag\\
         &\le 2^{\frac{\varepsilon}{\varepsilon-1}}\iota+\iota^{\prime}+C^{\frac{1}{\varepsilon}}(\iota+\iota^{\prime})^{\frac{\varepsilon-1}{\varepsilon}}+\kappa\,, \label{eq:local_34kd}
     \end{align}
     where Eq.~\eqref{eq:local_34kd} is by choosing $\delta=\frac{z_{\max } h_1}{\zeta(\Delta)^{\varepsilon /(\varepsilon-1)}+C}$ and $\delta^{\prime}=\frac{w_{\max } h_1}{\omega(\Delta)^{3 /2}+C}$, and in this case, 
     \begin{align*}
         \iota&=\zeta(\Delta)^{\frac{1}{\varepsilon-1}}\left(\frac{\varepsilon}{2}\right)^{\frac{1}{1-\varepsilon}} \log _{+}\left(\frac{z_{\max } h_1 T}{\zeta(\Delta)^{\varepsilon/(\varepsilon-1)}+C}\right) /(\varepsilon-1)\,,\\
         \iota^{\prime}&=\omega(\Delta)^{\frac{1}{2}}\left(\frac{3}{2}\right)^{-\frac{1}{2}} \log _{+}\left(\frac{w_{\max} h_1 T}{\omega(\Delta)^{3/2}+C}\right)\,.
     \end{align*}
The proof is thus concluded.
\end{proof}

\subsection{Proof of Theorem~\ref{thm:bobw_lb}}\label{app:sec:bobw_htlb}
In this section, we formally prove \Cref{thm:bobw_lb}, via verifying that the required conditions in \Cref{prop:prop1} are indeed satisfied.
\begin{proof}[Proof of \Cref{thm:bobw_lb}]
\paragraph{Verifying Condition in Eq.~\eqref{eq:regret_F}}
First, notice that $\tildeell_t$ defined in Eq.~\eqref{eq:vr_linear_loss_est} satisfies
\begin{align}\label{eq:unbiased_diff}
    \left\langle q_t-p^*, \ell_t \right\rangle
    =
\suma q_{t,a}\left(\ell_{t,a}-\ell_{t,a^*}\right)
=
\suma q_{t,a}\E_{t-1}\left[\tildeell_{t,a}-\tildeell_{t,a^*}\right]
=
\E_{t-1}\left[\left\langle q_t-p^*, \tildeell_t \right\rangle\right]\,,
\end{align}
where the second equality is by \Cref{lem:unbiased_diff}.

With this in hand, similar to the proof of \Cref{thm:thm1}, we decompose the regret as
\begin{align}
    &\quad \Reg \notag\\
    &=  \mathbb{E}\left[\sum_{t=1}^T \ell_{t, a_t}-\ell_{t, a^*}\right] =  \mathbb{E}\left[\sum_{t=1}^T \left\langle p_t-p^*, \ell_t\right\rangle\right] \notag\\
    &=  \mathbb{E}\left[\sumt\left\langle q_t-p^*, \ell_t\right\rangle\right]+\mathbb{E}\left[\sumt \gamma_t\left\langle p_0-q_t, \ell_t\right\rangle\right]\notag\\
    &=  \mathbb{E}\left[\sumt\left\langle q_t-p^*, \tildeell_t\right\rangle\right]+\mathbb{E}\left[\sumt \gamma_t\left\langle p_0-q_t, \ell_t\right\rangle\right] \label{eq:using_unbiased_diff}\\
    &=  \mathbb{E}\left[\sumt\left\langle q_t-p^*, \hatell_t-b_t\right\rangle\right]
    +\mathbb{E}\left[\sumt\left\langle q_t-p^*, \tildeell_t-\hatell_t+b_t\right\rangle\right]+\mathbb{E}\left[\sumt \gamma_t\left\langle p_0-q_t, \ell_t\right\rangle\right] \notag\\
    &\le \mathbb{E}\left[\sumt \left(\left\langle q_t-q_{t+1}, \hat{\ell}_t-b_t\right\rangle-\beta_t D_{\psi}\left(q_{t+1}, q_t\right)+\left(\beta_t-\beta_{t-1}\right) h_t\right)+\bar{\beta} \bar{h} \right] \notag\\
    &\quad+\mathbb{E}\left[\sumt\left\langle q_t-p^*, \tildeell_t-\hatell_t+b_t\right\rangle\right]+2\E\left[\sumt \sigma^{\frac{1}{\varepsilon}}\gamma_t \right] \label{eq:dfjlksll_1}\\
    &\le \gO\left(\mathbb{E}\left[\sumt\left(\left(\beta_t-\beta_{t-1}\right) h_{t}+z_t \beta_t^{1-\varepsilon}+w_t \beta_t^{-2}\right)+\bar{\beta} \bar{h}\right] \right) \label{eq:dfjlksll_2}\\
    &= \gO\left(\mathbb{E}\left[F\left(\beta_{1: T} ;z_{1: T}, w_{1: T}, h_{1: T}\right)\right]+\bar{\beta} \bar{h}\right)\,.\notag
\end{align}
Here, Eq.~\eqref{eq:using_unbiased_diff} follows from Eq.~\eqref{eq:unbiased_diff}.  
In Eq.~\eqref{eq:dfjlksll_1}, we apply Eq.~\eqref{eq:lower_moment} together with the standard regret analysis of FTRL (see, \textit{e.g.}, Exercise~28.12 of \citet{lattimore2020bandit}), and set 
$h_t = -\psi(q_t)$ and $\bar{h} = -\bar{\psi}(q_1) \le \tfrac{1}{\bar{\alpha}} K^{1-\bar{\alpha}}$.  
And Eq.~\eqref{eq:dfjlksll_2} follows from \Cref{lem:stab_zt_wt}.

\paragraph{Verifying Condition in Eq.~\eqref{eq:stable_update}}
To begin with, note that $h_t$ can be lower bounded as follows
\begin{align}
    h_t=-\psi\left(q_t\right) \geq \frac{q_{t *}^\alpha}{\alpha}\left(1-2^{\alpha-1}\right) \geq \frac{(1-\alpha) q_{t *}^\alpha}{4 \alpha}\,,\notag
\end{align}
which, along with the update rule of $\beta_t$ in Eq.~\eqref{eq:beta_t}, shows that
\begin{align}\label{eq:beta_diff}
    &\quad\beta_{t+1}-\beta_t=\frac{1}{\hath_{t+1}}\left(\beta_{t}^{1-\varepsilon}z_{t}+\beta_{t}^{-2}w_{t} \right)=\frac{1}{{h}_{t}}\left(\beta_{t}^{1-\varepsilon}z_{t}+\beta_{t}^{-2}w_{t} \right)\notag\\
    &\le \frac{4 \alpha}{(1-\alpha) q_{t *}^\alpha}\left((1-\alpha)^{1-\varepsilon}\sigma q_{t*}^{(\varepsilon-1)(1-\alpha)}d^{\frac{\varepsilon}{2}}(1-\|p_t\|_{\infty})^{2-\varepsilon}\cdot\beta_{t}^{1-\varepsilon}+ (1-\alpha)^{-2}\sigma^{\frac{3}{\varepsilon}} d q_{t*}^{2(1-\alpha)}\cdot\beta_{t}^{-2} \right) \notag\\
    &\le 8 \alpha(1-\alpha)^{-3} q_{t*}^{(\varepsilon-1)(1-\alpha)-\alpha}d^{\varepsilon}\cdot\beta_{1}^{1-\varepsilon}
    \le \frac{\bar{\beta} \alpha q_{t *}^{\bar{\alpha}-\alpha}}{8}\,,
\end{align}
where the last inequality follows from that $\bar{\alpha}=(\varepsilon-1)(1-\alpha)$ and $\bar{\beta}\ge 64 (1-\alpha)^{-3} d^{\varepsilon}\beta_{1}^{1-\varepsilon}\cdot\max\{\sigma^{\frac{3}{\varepsilon}},\sigma\}$.
Therefore, combining Eq.~\eqref{eq:loss_magnitude} and Eq.~\eqref{eq:beta_diff}, one can see that Eq.~\eqref{eq:stable_update} holds by applying \Cref{lem:Shinji24_lem11} and \Cref{lem:Shinji24_lem12}, where we treat $\ell=\hatell_t$, $\beta=\beta_t$, and $\beta^{\prime}=\beta_{t+1}$.

Moreover, we have 
\begin{align}
    h_t=\frac{1}{\alpha}\left(\sum_{a\in\gA} q_{t,a}^\alpha-1\right) 
    \leq \frac{1}{\alpha}\left(\sum_{a\in\gA} q_{t,a}^\alpha-q_{t, a^*}^\alpha\right)
    =\frac{1}{\alpha} \sum_{a \in\gA \backslash\left\{a^*\right\}} q_{t,a}^\alpha 
    \leq \frac{(K-1)^{1-\alpha}}{\alpha}\,,\notag
\end{align}
which, along with 
the definitions of $z_t$ and $w_t$ in Eq.~\eqref{eq:bobw_htlb:zwt} that
\begin{align*}
    z_t = (1 - \alpha)^{1 - \varepsilon} \sigma q_{t*}^{(\varepsilon - 1)(1 - \alpha)} d^{\frac{\varepsilon}{2}} (1 - \|p_t\|_{\infty})^{2 - \varepsilon}\quad\text{and} \quad
    w_t = \sigma^{\frac{3}{\varepsilon}} (1 - \alpha)^{-2}  d q_{t*}^{2(1 - \alpha)}\,,
\end{align*}
leads to 
\begin{align}
    h_1^{1-\frac{1}{\varepsilon}}z_{\max }^{\frac{1}{\varepsilon}}&\le \left(K^{1-\alpha}\alpha^{-1}\right)^{1-\frac{1}{\varepsilon}} \left((1-\alpha)^{1-\varepsilon} d^{\varepsilon} \right)^{\frac{1}{\varepsilon}}
    =(\alpha(1-\alpha))^{\frac{1}{\varepsilon}-1}\sigma^{\frac{1}{\varepsilon}} d^{\frac{1}{2}} K^{(1-\alpha)(1-\frac{1}{\varepsilon})} \,,\notag\\
    h_1^{\frac{2}{3}} (w_{\max})^{\frac{1}{3}}&\le\left(K^{1-\alpha}\alpha^{-1}\right)^{\frac{2}{3}}\left((1-\alpha)^{-2}d \right)^{\frac{1}{3}}
    =(\alpha(1-\alpha))^{-\frac{2}{3}}\sigma^{\frac{1}{\varepsilon}}d^{\frac{1}{3}} K^{\frac{2}{3}(1-\alpha)}\,.\notag
\end{align}

\paragraph{Verifying Condition in Eq.~\eqref{eq:round_wise_0}}
We upper bound $h_t$ using $\left\langle \Delta,q_t\right\rangle$ via
\begin{align}
h_t \leq \frac{1}{\alpha}\left(\sum_{a \in\gA \backslash\left\{a^*\right\}} q_{t,a}^\alpha\right)
&=\frac{1}{\alpha}\left(\sum_{a \in\gA \backslash\left\{a^*\right\}} \frac{1}{\Delta_a^\alpha}\left(\Delta_a q_{t,a}\right)^\alpha\right) \\
&\leq \frac{1}{\alpha}\left(\sum_{a \in\gA \backslash\left\{a^*\right\}} \frac{1}{\Delta_a^{\frac{\alpha}{1-\alpha}}}\right)^{1-\alpha}\left\langle \Delta,q_t\right\rangle^\alpha\,.
\end{align}
Therefore, we have
\begin{align}
    h_t^{\varepsilon-1} z_t &\le \alpha^{1-\varepsilon}\sigma \left(\sum_{a \in\gA \backslash\left\{a^*\right\}} \Delta_a^{\frac{\alpha}{\alpha-1}}\right)^{(1-\alpha)(\varepsilon-1)}\left\langle \Delta,q_t\right\rangle^{\alpha(\varepsilon-1)} \notag\\
    &\quad\cdot (1-\alpha)^{1-\varepsilon} q_{t*}^{(\varepsilon-1)(1-\alpha)}d^{\frac{\varepsilon}{2}}(1-\|p_t\|_{\infty})^{2-\varepsilon} \notag\\
    &= \alpha^{1-\varepsilon}\sigma\left(\sum_{a \in\gA \backslash\left\{a^*\right\}} \Delta_a^{\frac{\alpha}{\alpha-1}}\right)^{(1-\alpha)(\varepsilon-1)}\left\langle \Delta,q_t\right\rangle^{(\alpha-1)(\varepsilon-1)+(\varepsilon-1)} \notag\\
    &\quad\cdot (1-\alpha)^{1-\varepsilon} q_{t*}^{(\varepsilon-1)(1-\alpha)}d^{\frac{\varepsilon}{2}}(1-\|p_t\|_{\infty})^{2-\varepsilon} \notag\\
    &\le \alpha^{1-\varepsilon}\sigma\left(\sum_{a \in\gA \backslash\left\{a^*\right\}} \Delta_a^{\frac{\alpha}{\alpha-1}}\right)^{(1-\alpha)(\varepsilon-1)}\left(\Delta_{\min}(1-q_{t,a^*})\right)^{(\alpha-1)(\varepsilon-1)}
    \left\langle \Delta,q_t\right\rangle^{(\varepsilon-1)} \notag\\
    &\quad\cdot (1-\alpha)^{1-\varepsilon}\left(1- q_{t,a^*}\right)^{(\varepsilon-1)(1-\alpha)}d^{\frac{\varepsilon}{2}}(1-\|p_t\|_{\infty})^{2-\varepsilon} \notag\\
    &\le \alpha^{1-\varepsilon}\sigma\left(\sum_{a \in\gA \backslash\left\{a^*\right\}} \Delta_a^{\frac{\alpha}{\alpha-1}}\right)^{(1-\alpha)(\varepsilon-1)}\Delta_{\min}^{(\alpha-1)(\varepsilon-1)}
    \left(2\left\langle \Delta,p_t\right\rangle\right)^{(\varepsilon-1)} \notag\\
    &\quad\cdot (1-\alpha)^{1-\varepsilon}d^{\frac{\varepsilon}{2}}(1-p_{t,a^*})^{2-\varepsilon} \notag\\
    &\le \alpha^{1-\varepsilon}\sigma\left(\sum_{a \in\gA \backslash\left\{a^*\right\}} \Delta_a^{\frac{\alpha}{\alpha-1}}\right)^{(1-\alpha)(\varepsilon-1)}\Delta_{\min}^{(\alpha-1)(\varepsilon-1)}
    \left(2\left\langle \Delta,p_t\right\rangle\right)^{(\varepsilon-1)} \notag\\
    &\quad\cdot (1-\alpha)^{1-\varepsilon}d^{\frac{\varepsilon}{2}}\left\langle \Delta, p_t\right\rangle^{2-\varepsilon}\Delta_{\min}^{\varepsilon-2} \notag\\
    &= 2^{\varepsilon-1}\left((1-\alpha)\alpha\right)^{1-\varepsilon}\sigma\left(\sum_{a \in\gA \backslash\left\{a^*\right\}} \Delta_a^{\frac{\alpha}{\alpha-1}}\right)^{(1-\alpha)(\varepsilon-1)}\Delta_{\min}^{\alpha(\varepsilon-1)-1}d^{\frac{\varepsilon}{2}}\cdot
    \left\langle \Delta,p_t\right\rangle\,. \notag
\end{align}
On the other hand, since $\alpha\ge \frac{1}{2}$ and $\Delta_{\max}\le 1$, we have
\begin{align}
    {h}_t^{2} w_t
    &\le \frac{1}{\alpha^2}\sigma^{\frac{3}{\varepsilon}}\left(\sum_{a \in\gA \backslash\left\{a^*\right\}} \frac{1}{\Delta_a^{\frac{\alpha}{1-\alpha}}}\right)^{2(1-\alpha)}\left\langle \Delta,q_t\right\rangle^{2\alpha}\cdot (1-\alpha)^{-2}d q_{t*}^{2(1-\alpha)} \notag\\
    &\le \left(\alpha(1-\alpha)\right)^{-2}\sigma^{\frac{3}{\varepsilon}}\left(\sum_{a \in\gA \backslash\left\{a^*\right\}} \frac{1}{\Delta_a^{\frac{\alpha}{1-\alpha}}}\right)^{2(1-\alpha)}d \cdot \left\langle \Delta,q_t\right\rangle \notag\\
    &\lesssim \left(\alpha(1-\alpha)\right)^{-2}\sigma^{\frac{3}{\varepsilon}}\left(\sum_{a \in\gA \backslash\left\{a^*\right\}} \frac{1}{\Delta_a^{\frac{\alpha}{1-\alpha}}}\right)^{2(1-\alpha)}d \cdot \left\langle \Delta,p_t\right\rangle\,.\notag
\end{align}
Therefore, one can see that Eq.~\eqref{eq:round_wise_0} in \Cref{prop:prop1} holds with
\begin{align*}
    \zeta(\Delta)=2^{\varepsilon-1}\left((1-\alpha)\alpha\right)^{1-\varepsilon}\sigma\left(\sum_{a \in\gA \backslash\left\{a^*\right\}} \Delta_a^{\frac{\alpha}{\alpha-1}}\right)^{(1-\alpha)(\varepsilon-1)}\Delta_{\min}^{\alpha(\varepsilon-1)-1}d^{\frac{\varepsilon}{2}}\,,
\end{align*}
and
\begin{align*}
    \omega(\Delta)=\left(\alpha(1-\alpha)\right)^{-2}\sigma^{\frac{3}{\varepsilon}}\left(\sum_{a \in\gA \backslash\left\{a^*\right\}} \Delta_a^{\frac{\alpha}{\alpha-1}}\right)^{2(1-\alpha)}d\,.
\end{align*}
\end{proof}

\begin{lem}\label{lem:stab_zt_wt}
For $z_t$ and $w_t$ defined in Eq.~\eqref{eq:bobw_htlb:zwt}, it holds that
\begin{align}\label{eq:stability_regret}
    &\quad\mathbb{E}_{t-1}\left[\sigma^{\frac{1}{\varepsilon}}\gamma_t+\left\langle q_t-q_{t+1}, \hatell_t-b_t\right\rangle-\beta_t D_{\psi}\left(q_{t+1}, q_t\right)+\left\langle q_t-p^*, \tildeell_t-\hatell_t+b_t\right\rangle\right] \notag\\
    &=\gO\left(z_t \beta_t^{1-\varepsilon}+w_t \beta_t^{-2}\right)\,.
\end{align}
\end{lem}
\begin{proof}
First note that by the definition that $\gamma_t\coloneqq 256(1-\alpha)^{-2}\sigma^{\frac{2}{\varepsilon}} d \beta_t^{-2} q_{t*}^{2(1-\alpha)}$, we have
\begin{align}\label{eq:exp_wt}
    \mathbb{E}_{t-1}\left[\sigma^{\frac{1}{\varepsilon}}\gamma_t\right]
    \lesssim (1-\alpha)^{-2}\sigma^{\frac{3}{\varepsilon}} d \beta_t^{-2} q_{t*}^{2(1-\alpha)}=w_t\beta_t^{-2}
    \,.
\end{align}

    It remains to show that the summation of the stability term and the bias term can indeed be upper bounded by $\gO\left(z_t \beta_t^{1-\varepsilon}\right)$ for $z_t$ defined in Eq.~\eqref{eq:bobw_htlb:zwt}.

Fix an arbitrary $a\in\gA$.
    By setting $s_{t,a}=(1-\alpha)\beta_t {q}_{t*}^{\alpha-1}/8$,
    we have $|\hatell_{t, a}|=|\tildeell_{t,a} \cdot \inlineindicator{| \tildeell_{t, a}| \le s_{t, a}} |\le s_{t,a}=(1-\alpha)\beta_t \tilde{q}_{t*}^{\alpha-1}/8$. On the other hand, we have
    \begin{align}\label{eq:bta_sta}
        b_{t,a}&\coloneqq 
        \sigma\sum_{b\in\gA} p_{t,b}\left|\bphia^{\top} V_t^{-1} \bphib\right|^{\varepsilon} s_{t,a}^{1-\varepsilon}
        \le 
        \sigma\left(\sum_{b\in\gA} p_{t,b}\left|\bphia^{\top} V_t^{-1} \bphib\right|^2\right)^{\frac{\varepsilon}{2}}s_{t,a}^{1-\varepsilon}
        =
        \sigma\left(\bphia^{\top} V_t^{-1} \bphia\right)^{\frac{\varepsilon}{2}}s_{t,a}^{1-\varepsilon}\notag\\
    &\le 
    \sigma\left(\bphia^{\top} (\gamma_t V(p_0))^{-1} \bphia\right)^{\frac{\varepsilon}{2}}s_{t,a}^{1-\varepsilon}
    \le 
    \sigma\left(4d\cdot \gamma_t^{-1}\right)^{\frac{\varepsilon}{2}}s_{t,a}^{1-\varepsilon}
    \le 
    \left(4d\cdot \frac{s_{t,a}^2}{4d}\right)^{\frac{\varepsilon}{2}}s_{t,a}^{1-\varepsilon}=s_{t,a}\,,
    \end{align}
    where the first inequality is by Jensen's inequality and the fact that $\varepsilon\in(1,2]$, the second inequality is due to that
    \begin{align*}
    V_t
    &
    =
    (1 - \gamma_t) \E_{a \sim q_t}[(\phi_a - \mu(p_t))(\phi_a - \mu(p_t))^\top ]
    +
    \gamma_t \E_{a \sim p_0}[(\phi_a - \mu(p_t))(\phi_a - \mu(p_t))^\top ]\notag
    \\
    &
    \succeq
    \gamma_t \E_{a \sim p_0}[(\phi_a - \mu(p_t))(\phi_a - \mu(p_t))^\top ]\notag
    \\
    &
    =
    \gamma_t \E_{a \sim p_0}[(\phi_a - \mu(p_0) + \mu(p_0) - \mu(p_t))(\phi_a - \mu(p_0) + \mu(p_0) - \mu(p_t))^\top ]
    \\
    &
    =
    \gamma_t \left(
        \E_{a \sim p_0}[(\phi_a - \mu(p_0))(\phi_a - \mu(p_0))^\top ]
        +
        (\mu(p_0) - \mu(p_t)) (\mu(p_0) - \mu(p_t))^\top
    \right)
    \\
    &
    \succeq
    \gamma_t V(p_0)\,,
\end{align*}
the third inequality follows from the definition of $p_0$ in Eq.~\eqref{eq:barp0} and \Cref{lem:barp0_exp}, and the last inequality is by noticing that 
    $\gamma_t\coloneqq 256(1-\alpha)^{-2}\sigma^{\frac{2}{\varepsilon}} d \beta_t^{-2} q_{t*}^{2(1-\alpha)}= 4\sigma^{\frac{2}{\varepsilon}}ds_{t,a}^{-2}$.

    Hence we have 
    \begin{align}\label{eq:loss_magnitude}
        |\hatell_{t, a}-b_{t, a}| \le 2s_{t,a}= \frac{(1-\alpha)\beta_t {q}_{t*}^{\alpha-1}}{4}\,.
    \end{align}
    By \Cref{lem:stab_condition_Shinji24}, one can deduce that 
    \begin{align}\label{eq:stab_fdxxs1}
        \E_{t-1}\left[\langle q_t-q_{t+1}, \hatell_t-b_t\rangle-\beta_t D_{\psi}\left(q_{t+1}, q_t\right)\right]
        &\lesssim \frac{1}{1-\alpha} \beta_t^{-1}\mathbb{E}_{t-1}\left[\sum_{a\in\gA} \tilde{q}_{t, a}^{2-\alpha}\left(\hatell_{t , a}+b_{t , a}\right)^2\right] \notag\\
        &\lesssim \frac{1}{1-\alpha} \beta_t^{-1}\mathbb{E}_{t-1}\left[\sum_{a\in\gA} \tilde{q}_{t, a}^{2-\alpha}\left(\hatell_{t , a}^2+b_{t , a}^2\right)\right]\,.
    \end{align}
    Further, notice that
    \begin{align}\label{eq:ell_square}
&\quad\mathbb{E}_{t-1}\left[\hatell_{t,a}^2\right]=\mathbb{E}_{t-1}\left[\tildeell_{t,a}^2 \inlineindicator{|{\tildeell_{t,a}}| \le s_{t,a}}\right]\notag
=\mathbb{E}_{t-1}\left[|\tildeell_{t,a}|^{\varepsilon} \abs{\tildeell_{t,a}}^{2-\varepsilon} \inlineindicator{|{\tildeell_{t,a}}| \le s_{t,a}}\right] \notag\\
&\le \mathbb{E}_{t-1}\left[|\tildeell_{t,a}|^{\varepsilon} s_{t,a}^{2-\varepsilon}\right]
= \mathbb{E}_{t-1}\left[\left|\bphia^{\top} V_t^{-1} \bphiat\right|^{\varepsilon}\left|\ell_{t,a_t}\right|^{\varepsilon} s_{t,a}^{2-\varepsilon}\right]
=\mathbb{E}_{t-1}\left[\left|\bphia^{\top} V_t^{-1} \bphiat\right|^{\varepsilon}\E_{\ell_t}\left[\left|\ell_{t,a_t}\right|^{\varepsilon}\right] s_{t,a}^{2-\varepsilon}\right] \notag\\
&\le\sigma\mathbb{E}_{t-1}\left[\left|\bphia^{\top} V_t^{-1} \bphiat\right|^{\varepsilon} s_{t,a}^{2-\varepsilon}\right]= \sigma s_{t,a}^{2-\varepsilon}\sum_{b \in \gA} p_{t, b}\left|\bphia^{\top} V_t^{-1} \bphib\right|^{\varepsilon} \,,
    \end{align}
and
\begin{align}\label{eq:bta_square}
b_{t,a}^2 \le s_{t,a} b_{t,a} 
=
s_{t,a}\sigma\left(\sum_{b\in\gA} p_{t, b}\left|\bphia^{\top} V_t^{-1} \bphib\right|^{\varepsilon} s_{t,a}^{1-\varepsilon}\right) 
\le 
\sigma s_{t,a}^{2-\varepsilon} \sum_{b\in\gA} p_{t, b}\left|\bphia^{\top} V_t^{-1} \bphib\right|^{\varepsilon} \,,
\end{align}
where the first inequality is by Eq.~\eqref{eq:bta_sta}.

Substituting Eq.~\eqref{eq:ell_square} and Eq.~\eqref{eq:bta_square} into Eq.~\eqref{eq:stab_fdxxs1} leads to 
\begin{align}
    &\quad\E_{t-1}\left[\langle q_t-q_{t+1}, \hat{\ell}_t-b_t\rangle-\beta_t D_{\psi}\left(q_{t+1}, q_t\right)\right]\notag\\
    &\lesssim \frac{1}{1-\alpha}\sigma \beta_t^{-1}\sum_{a\in\gA} \tilde{q}_{t, a}^{2-\alpha}s_{t,a}^{2-\varepsilon} \sum_{b\in\gA} p_{t, b}\left|\bphia^{\top} V_t^{-1} \bphib\right|^{\varepsilon} \label{eq:internal_dfsdfs1}\\
    &\lesssim (1-\alpha)^{1-\varepsilon}\sigma \beta_t^{1-\varepsilon}\sum_{a\in\gA} \tilde{q}_{t, a}^{2-\alpha}{q}_{t*}^{(2-\varepsilon)(\alpha-1)}\sum_{b\in\gA} p_{t, b}\left|\bphia^{\top} V_t^{-1} \bphib\right|^{\varepsilon} \notag\\
    &\le (1-\alpha)^{1-\varepsilon}\sigma \beta_t^{1-\varepsilon}\sum_{a\in\gA} 
    {q}_{t, a}{q}_{t*}^{1-\alpha} {q}_{t*}^{(2-\varepsilon)(\alpha-1)}\sum_{b\in\gA} p_{t, b}\left|\bphia^{\top} V_t^{-1} \bphib\right|^{\varepsilon} \notag\\
    &= (1-\alpha)^{1-\varepsilon}\sigma \beta_t^{1-\varepsilon} q_{t*}^{(\varepsilon-1)(1-\alpha)}\sum_{a,b\in\gA}q_{t, a} p_{t, b}\left|\bphia^{\top} V_t^{-1} \bphib\right|^{\varepsilon} \notag\\
    &\le 2(1-\alpha)^{1-\varepsilon}\beta_t^{1-\varepsilon}\sigma q_{t*}^{(\varepsilon-1)(1-\alpha)}\sum_{a,b\in\gA}p_{t, a} p_{t, b}\left|\bphia^{\top} V_t^{-1} \bphib\right|^{\varepsilon} \label{eq:internal_dfsdfs2}\\
    &\lesssim (1-\alpha)^{1-\varepsilon}\sigma \beta_t^{1-\varepsilon} q_{t*}^{(\varepsilon-1)(1-\alpha)}d^{\frac{\varepsilon}{2}}(1-\|p_t\|_{\infty})^{2-\varepsilon}\,, \label{eq:internal_dfsdfs3}
\end{align}
where in Eq.~\eqref{eq:internal_dfsdfs1} we use the definition that $s_{t,a}=(1-\alpha)\beta_t {q}_{t*}^{\alpha-1}/8$ and $\bar{\alpha}=(\varepsilon-1)(1-\alpha)$, Eq.~\eqref{eq:internal_dfsdfs2} is because $\gamma_t\le 1/2$ and thus $q_{t,a}\le 2p_{t,a}$ for all $a\in\gA$, and Eq.~\eqref{eq:internal_dfsdfs3} follows from \Cref{lem:loss_var}.

Meanwhile, for the bias term, one can see that
\begin{align}
    &\quad\E_{t-1}\left[\left\langle q_t-p^*, \tildeell_t-\hatell_t+b_t\right\rangle \right]\notag\\
    &=\E_{t-1}\left[\sum_{a\in\gA} \left(q_{t,a}-p^*_a\right) |\tildeell_{t, a}| \indicator{|\tildeell_{t,a}|>s_{t, a}}\right]+\left\langle q_t-p^*,b_t\right\rangle \notag\\
    &\le \E_{t-1}\left[\sum_{a\in\gA} \left|q_{t,a}-p^*_a\right| |\tildeell_{t, a}| \indicator{|\tildeell_{t,a}|>s_{t, a}}\right]
    +\left\langle q_t-p^*,b_t\right\rangle \notag\\
    &\le \E_{t-1}\left[\sum_{a\in\gA} \left|q_{t,a}-p^*_a\right| |\tildeell_{t, a}|^{\varepsilon} s_{t, a}^{1-\varepsilon}\right]
    +\left\langle q_t-p^*,b_t\right\rangle \label{eq:bias_sdff1}\\
    &=\E_{t-1}\left[\sum_{a\in\gA} \left|q_{t,a}-p^*_a\right| 
    |\bphia V_t^{-1}\bphiat|^{\varepsilon}\E_{\ell_t}[|\ell_{t,a_t}|^{\varepsilon}]
    s_{t, a}^{1-\varepsilon}\right]
    +\left\langle q_t-p^*,b_t\right\rangle \notag\\
    &\le\sigma \sum_{a\in\gA}\left|q_{t,a}-p^*_a\right|s_{t, a}^{1-\varepsilon}\sum_{b\in\gA} p_{t,b} 
    |\bphia V_t^{-1}\bphib|^{\varepsilon}
    +\sum_{a \neq a^*} q_{t,a} b_{t,a}+\left(q_{t,a^*}-1\right) b_{t,a^*} \notag\\
    &=2\sigma\sum_{a\neq a^{*}}q_{t,a}s_{t, a}^{1-\varepsilon}\sum_{b\in\gA} p_{t,b} 
    |\bphia V_t^{-1}\bphib|^{\varepsilon} \label{eq:bias_sdff3}\\
    &= 2^{3\varepsilon-2}\sigma(1-\alpha)^{1-\varepsilon}\beta^{1-\varepsilon}_t\sum_{a\neq a^{*}}q_{t,a}{q}_{t*}^{(1-\alpha)(\varepsilon-1)}\sum_{b\in\gA} p_{t,b} 
    |\bphia V_t^{-1}\bphib|^{\varepsilon} \label{eq:bias_sdff4}\\
    &\le 2^{3\varepsilon-1}(1-\alpha)^{1-\varepsilon}\sigma\beta^{1-\varepsilon}_t q_{t*}^{(1-\alpha)(\varepsilon-1)}\sum_{a,b\in\gA} p_{t,a} p_{t,b} 
    |\bphia V_t^{-1}\bphib|^{\varepsilon} \notag\\
    &\le 2^{3\varepsilon-1}(1-\alpha)^{1-\varepsilon}\sigma\beta_t^{1-\varepsilon} q_{t*}^{(\varepsilon-1)(1-\alpha)}d^{\frac{\varepsilon}{2}}(1-\|p_t\|_{\infty})^{2-\varepsilon} \label{eq:bias_sdff5}\,,
\end{align}
where in Eq.~\eqref{eq:bias_sdff1} we use the fact that $\varepsilon\in (1,2]$, Eq.~\eqref{eq:bias_sdff3} and Eq.~\eqref{eq:bias_sdff4} follow from the definitions that $b_{t,a}=\sum_{b\in\gA} p_{t, b}\left|\bphia^{\top} V_t^{-1} \bphib\right|^{\varepsilon} s_{t,a}^{1-\varepsilon}$ and $s_{t,a}=\frac{(1-\alpha)\beta_t {q}_{t*}^{\alpha-1}}{8}$ for all $a\in\gA$, and Eq.~\eqref{eq:bias_sdff5} is by \Cref{lem:loss_var}.

Combining Eq.~\eqref{eq:internal_dfsdfs3} and Eq.~\eqref{eq:bias_sdff5} shows that
\begin{align}\label{eq:stab_bias_zt}
    &\quad\mathbb{E}_{t-1}\left[\left\langle q_t-q_{t+1}, \hatell_t-b_t\right\rangle-\beta_t D_{\psi}\left(q_{t+1}, q_t\right)+\left\langle q_t-p^*, \tildeell_t-\hatell_t+b_t\right\rangle\right] \notag\\
    &\lesssim (1 - \alpha)^{1 - \varepsilon} \sigma q_{t*}^{(\varepsilon - 1)(1 - \alpha)} d^{\frac{\varepsilon}{2}} (1 - \|p_t\|_{\infty})^{2 - \varepsilon}\beta_t^{1-\varepsilon}
    =z_t \beta_t^{1-\varepsilon}\,.
\end{align}

Finally, the proof is concluded by combining Eq.~\eqref{eq:exp_wt} and Eq.~\eqref{eq:stab_bias_zt}.
\end{proof}

\section{Properties of the Exploration Distribution and Variance-reduced Loss Estimator}
We begin with the following lemma, which shows that the ``optimal design distribution'' $p_0$, defined in Eq.~\eqref{eq:barp0}, induces a bounded quadratic form.
\begin{lem}\label{lem:barp0_exp}
For any distribution $p\in\argmax_{p^{\prime} \in \gP(\gA)} \log \det (V(p^{\prime}))$, it holds that
\begin{align}
    \| x - y \|_{V(p)^{-1}}
    \le
    2 \sqrt{d}\,,
\end{align}
for any $x, y \in \mathrm{Conv}(\{\phia\}_{a\in\gA})$.
\end{lem}
\begin{proof}
Fix arbitraty $p\in\argmax_{p^{\prime} \in \gP(\gA)} \log \det (V(p^{\prime}))$.
For all $a \in \mathcal{A}$,
we define $\phi'_a = \begin{bmatrix} 1 \\ \phi_a \end{bmatrix} \in \mathbb{R}^{d+1}$.
and
$H'(p) = \E_{a \sim p}[\phi'_a \phi_a'^\top]$.
Then,
as $H'(p) = \begin{bmatrix}1 & \mu(p)^\top \\ \mu(p) & H(p) \end{bmatrix}$,
we have
\begin{align}
    \det H'(p) = \det \begin{bmatrix}
        1 & \mu(p)^\top \\ \mu(p) & H(p) 
    \end{bmatrix}
    =
    \det \begin{bmatrix}
        1 & \mu(p)^\top \\ 0  & H(p) - \mu(p)  \mu(p)^\top
    \end{bmatrix}
    =
    \det V(p)\,,\notag
\end{align}
which implies that
$p$ is a G-optimal design distribution over $\{ \phi'_a \}_{a \in \mathcal{A}}$.
Therefore, by Kiefer–Wolfowitz theorem \citep{kiefer1960equivalence}, one can see that
\begin{align}
    \phi_a'^\top H'(p)^{-1} \phi_a' \le d+1
    \label{eq:aHp1}
\end{align}
for all $a \in \mathcal{A}$.
 
Besides, it is clear that
\begin{align}
    H'(p)^{-1}
    =
 \begin{bmatrix}
        1 + \mu(p)^\top V(p)^{-1} \mu(p) & - \mu(p)^\top V(p)^{-1} \\ - V(p)^{-1} \mu(p) & V(p)^{-1}
    \end{bmatrix}\,,\notag
\end{align}
and thus
\begin{align}
    \phi_a'^\top
    H'(p)^{-1}
    \phi_a'
    &
    =
    1
    +
    \mu(p)^\top
    V(p)^{-1}
    \mu(p)
    -
    2
    \mu(p)^\top
    V(p)^{-1}
    \phi_a
    +
    \phi_a^\top
    V(p)^{-1}
    \phi_a \notag
    \\
    &
    =
    1 + (\phi_a - \mu(p))^\top V(p)^{-1} (\phi_a - \mu(p))\,.\notag
\end{align}
Combining the above display with Eq.~\eqref{eq:aHp1} leads to
\begin{align}
    \| \phi_a - \mu(p) \|_{V(p)^{-1}}
    \le
    \sqrt{d}\,,\notag
\end{align}
for all $a\in\gA$.

Therefore, for any $x, y \in \mathrm{Conv}(\{\phia\}_{a\in\gA})$, it holds that
\begin{align}
    \| x - y \|_{V(p)^{-1}}
    \le
    \| x - \mu(p) \|_{V(p)^{-1}}
    +
    \| \mu(p) - y \|_{V(p)^{-1}}
    \le
    2 \sqrt{d}\,,
\end{align}
which concludes the proof.
\end{proof}

Then, we show that though our variance-reduced least-squares loss estimator $\tilde{\ell}_{t,a}$ in Eq.~\eqref{eq:vr_linear_loss_est} is not an unbiased loss estimator of $\ell_{t,a}$, $\E_{t-1}[\tilde{\ell}_{t,a}-\tilde{\ell}_{t,b}]$ still serves as an unbiased loss estimator of $\ell_{t,a}-\ell_{t,b}$ for any $a,b\in\gA$.
\begin{lem}\label{lem:unbiased_diff}
For any $a,b\in\gA$ and $t\in[T]$, the loss estimator $\tilde{\ell}_{t,a}$ defined in Eq.~\eqref{eq:vr_linear_loss_est} satisfying that $\E_{t-1}[\tilde{\ell}_{t,a}-\tilde{\ell}_{t,b}]=\ell_{t,a}-\ell_{t,b}$.
\end{lem}
\begin{proof}
Recall $\mu_t = \mathbb{E}_{a \sim p_t}[\phi_a]$, $\bar{\phi}_a = \phi_a - \mu_t$, and $V_t = \mathbb{E}_{a \sim p_t}[\bar{\phi}_a \bar{\phi}_a^\top]$.
The definition of $\tilde{\ell}_{t,a}$ in Eq.~\eqref{eq:vr_linear_loss_est} directly implies that
    \begin{align}
        \E_{t-1}\left[\tilde{\ell}_{t,a}\right] 
        &= \E_{t-1}\left[\bar{\phi}_a^\top V_t^{-1} \bar{\phi}_{a_t} \ell_{t,a_t}\right]
        = \E_{t-1}\left[\bar{\phi}_a^\top V_t^{-1} \bar{\phi}_{a_t} \phiat^\top\theta_t\right] \notag\\
        &= \E_{t-1}\left[\bar{\phi}_a^\top V_t^{-1} \bar{\phi}_{a_t} (\bphiat+\mu_t)^\top\theta_t\right]
        = \bar{\phi}_a^\top\theta_t\,,\notag
    \end{align}
    where the last equality follows from $\E_{t-1}[\bar{\phi}_{a_t}]=0$. Noticing that $\E_{t-1}[\tilde{\ell}_{t,a}-\tilde{\ell}_{t,b}]=(\bar{\phi}_a-\bar{\phi}_b)^\top\theta_t=(\phi_a-\phi_b)^\top\theta_t=\ell_{t,a}-\ell_{t,b}$ concludes the proof.
\end{proof}

The following lemma establishes a key property of our variance-reduced least-squares loss estimator $\tilde{\ell}_{t,a}$ in Eq.~\eqref{eq:vr_linear_loss_est}: its ``variance'' can be upper bounded by $(1 - \|p\|_{\infty})^{2 - \varepsilon}$. This property is a crucial ingredient for applying the self-bounding inequality to derive the BOBW regret bound.
\begin{lem}\label{lem:loss_var}
Fix arbitrary $p\in \gP(\gA)$. Let $\mu_{\phi}(p)\coloneqq \E_{a\sim p}[\phia]$, $\bar{\phi}_a\coloneqq\phi_a-\mu_{\phi}(p)$, and $V_{\phi}(p)\coloneqq \E_{a\sim p}[\bphia\bphia^{\top}]$.
 Then we have 
    \begin{align*}
        \sum_{a,b\in\gA}p_{ a} p_{ b}\left|\bphia^{\top} V_{\phi}(p)^{-1} \bphib\right|^{\varepsilon}
        \le 4d^{\frac{\varepsilon}{2}}(1-\|p\|_{\infty})^{2-\varepsilon}\,.
    \end{align*}
\end{lem}
\begin{proof}
    In the following, we define $\tilde{a}\in \arg\max_{a\in\gA} p_{a}$. For notational convenience, we let $x_a\coloneqq V_{\phi}(p)^{-\frac{1}{2}}\bphia$, $\mu_x\left(p\right):=\mathbb{E}_{a \sim p}[x_a]$, $H_x\left(p\right):=\mathbb{E}_{a \sim p}\left[x_a x_a^{\top}\right]$, and $\mathbb{V}_x\left(p\right):=H_x(p)-\mu_x(p) \mu_x(p)^{\top}$. By definition, it is clear that $\mu_x\left(p\right)=0$ and $H_x\left(p\right)=I_d$.
Then, 
\begin{align}\label{eq:loss_var_sdf1}
    \sum_{a,b\in\gA}p_{ a} p_{ b}\left|\bphia^{\top} V_{\phi}(p)^{-1} \bphib\right|^{\varepsilon}
    =p_{\tilde{a}}^2|x_{\tilde{a}}^{\top}x_{\tilde{a}}|^{\varepsilon}
    +2\sum_{a \neq \tilde{a}} p_{\tilde{a}} p_a\left|x_a^{\top} x_{\tilde{a}}\right|^{\varepsilon}
    +\sum_{a \neq \tilde{a}, b \neq \tilde{a}} p_{a} p_b\left|x_a^{\top} x_b\right|^{\varepsilon}\,.
\end{align}
To bound the first term on the RHS of the above display, note that
\begin{align*}
H_x(p)&=p_{\tilde{a}} x_{\tilde{a}} x_{\tilde{a}}^{\top}+(1-p_{\tilde{a}}) H_x(p^{\prime})
=p_{\tilde{a}} x_{\tilde{a}} x_{\tilde{a}}^{\top}+(1-p_{\tilde{a}})\left(\mathbb{V}_x(p^{\prime})+\mu_x(p^{\prime}) \mu_x(p^{\prime})^{\top}\right) \\
&\succeq p_{\tilde{a}} x_{\tilde{a}} x_{\tilde{a}}^{\top}+(1-p_{\tilde{a}}) \mu_x(p^{\prime}) \mu_x(p^{\prime})^{\top}
= p_{\tilde{a}} x_{\tilde{a}} x_{\tilde{a}}^{\top}+ \frac{p_{\tilde{a}}^2}{(1-p_{\tilde{a}})} x_{\tilde{a}} x_{\tilde{a}}^{\top} 
= \frac{p_{\tilde{a}}}{(1-p_{\tilde{a}})} x_{\tilde{a}} x_{\tilde{a}}^{\top}\,,
\end{align*}
where $p^{\prime}\in \gP(\gA)$ is such that $p^{\prime}_a=\frac{1}{1-p_{\tilde{a}}} p_a$ for any $a \neq \tilde{a}$ and $p^{\prime}_{\tilde{a}}=0$, and we use the property that $\mu_x(p^{\prime})=-\frac{1}{(1-p_{\tilde{a}})} x_{\tilde{a}} p_{\tilde{a}}$.
This, along with $H_x\left(p\right)=I_d$, implies 
\begin{align}
        \| x_{\tilde{a}} \|^2_2 
    = x_{\tilde{a}} ^\top I_d x_{\tilde{a}} 
    \ge x_{\tilde{a}}^\top \left( \frac{p_{\tilde{a}}}{1 - p_{\tilde{a}}} x_{\tilde{a}} x_{\tilde{a}}^\top \right) x_{\tilde{a}} 
    = 
    \frac{p_{\tilde{a}}}{1 - p_{\tilde{a}}}
    \| x_{\tilde{a}} \|^4_2\,,\notag
\end{align}
and thus it holds that
\begin{align}\label{eq:loss_var_sdf1x}
    \|x_{\tilde{a}}\|_2^2\le \frac{(1-p_{\tilde{a}})}{p_{\tilde{a}}}\,.
\end{align}

Besides, one can derive that 
\begin{align}\label{eq:loss_var_sdf3}
    \sum_{a \neq \tilde{a}}  p_a\|x_a\|^{\varepsilon}_2
&=  (1-p_{\tilde{a}})\sum_{a \neq \tilde{a}} \frac{p_a}{(1-p_{\tilde{a}})}\|x_a\|_2^{2 \cdot \frac{\varepsilon}{2}} 
\le (1-p_{\tilde{a}})\left(\sum_{a \neq \tilde{a}} \frac{p_a}{(1-p_{\tilde{a}})}\|x_a\|_2^2\right)^{\frac{\varepsilon}{2}} \notag\\
&= (1-p_{\tilde{a}})^{1-\frac{\varepsilon}{2}}\left(\sum_{a \neq \tilde{a}} p_a\|x_a\|_2^2\right)^{\frac{\varepsilon}{2}} \le  (1-p_{\tilde{a}})^{1-\frac{\varepsilon}{2}} d^{\frac{\varepsilon}{2}}\,,
\end{align}
where the first inequality is by Jensen's inequality, and in the last inequality we use $\sum_{a\in \gA} p_a \|x_a\|_2^2=\operatorname{tr}(H_x(p))=d$.

For the third term on the RHS of Eq.~\eqref{eq:loss_var_sdf1}, notice that
\begin{align}
    \sum_{a \neq \tilde{a}, b\neq \tilde{a}} p_a p_b \left| x_a^\top x_b \right|^\varepsilon
    &
    =
    (1 - p_{\tilde{a}})^2
    \sum_{a \neq \tilde{a}, b\neq \tilde{a}} 
    \frac{p_a}{1 - p_{\tilde{a}}}
    \frac{p_b}{1 - p_{\tilde{a}}}
    \left( (x_a^\top x_b)^2  \right)^{\varepsilon/2}\notag
    \\
    &
    \le
    (1 - p_{\tilde{a}})^2
    \left(
    \sum_{a \neq \tilde{a}, b\neq \tilde{a}} 
    \frac{p_a}{1 - p_{\tilde{a}}}
    \frac{p_b}{1 - p_{\tilde{a}}}
    (x_a^\top x_b)^2  \right)^{\varepsilon/2}\label{eq:loss_variance_term3x1}
    \\
    &
    =
    (1 - p_{\tilde{a}})^{2-\varepsilon}
    \left(
    \sum_{a \neq \tilde{a}, b\neq \tilde{a}} 
    p_a p_b
    (x_a^\top x_b)^2  \right)^{\varepsilon/2}\notag
    \\
    &
    \le
    (1 - p_{\tilde{a}})^{2-\varepsilon}
    \left(
    \sum_{a \in \mathcal{A}, b \in \mathcal{A}} 
    p_a p_b
    (x_a^\top x_b)^2  \right)^{\varepsilon/2}\notag
    \\
    &
    =
    (1 - p_{\tilde{a}})^{2-\varepsilon}
    \left(
        \mathrm{tr}
        \left(
            \sum_{a \in \mathcal{A}} 
            p_a x_a x_a^\top
            \sum_{b \in \mathcal{A}} 
            p_b x_b x_b^\top
        \right)
    \right)^{\varepsilon/2}\notag
    \\
    &
    =
    (1 - p_{\tilde{a}})^{2-\varepsilon}
    d^{\varepsilon/2}\,.\label{eq:loss_variance_term3x2}
\end{align}
where Eq.~\eqref{eq:loss_variance_term3x1} follows from applying Jensen's inequality twice.

Substituting Eqs.~\eqref{eq:loss_var_sdf1x},~\eqref{eq:loss_var_sdf3}, and~\eqref{eq:loss_variance_term3x2} into Eq.~\eqref{eq:loss_var_sdf1} shows that
\begin{align*}
    \sum_{a,b\in\gA}p_{ a} p_{ b}\left|\bphia^{\top} V_{\phi}(p)^{-1} \bphib\right|^{\varepsilon}
    &\le p_{\tilde{a}}^2|x_{\tilde{a}}^{\top}x_{\tilde{a}}|^{\varepsilon}
    +2p_{\tilde{a}}\| x_{\tilde{a}}\|^{\varepsilon}_2 \sum_{a \neq \tilde{a}}  p_a\|x_a\|^{\varepsilon}_2
    + \sum_{a \neq \tilde{a}, b \neq \tilde{a}} p_{a} p_b\|x_a\|_2 ^{\varepsilon}\|y_b\|_2^{\varepsilon}\\
    &\le 
    p_{\tilde{a}}^{2-\varepsilon} (1-p_{\tilde{a}})^{\varepsilon}
    +
    2d^{\frac{\varepsilon}{2}} p_{\tilde{a}}^{1-\varepsilon/2} (1-p_{\tilde{a}})
    +
    (1-p_{\tilde{a}})^{2-\varepsilon} d^{\frac{\varepsilon}{2}}\\
    &\le 
    (1-p_{\tilde{a}})^{2-\varepsilon}
    +
    2d^{\frac{\varepsilon}{2}} (1-p_{\tilde{a}})^{2-\varepsilon}+(1-p_{\tilde{a}})^{2-\varepsilon} d^{\frac{\varepsilon}{2}}
    =4 d^{\frac{\varepsilon}{2}}(1-p_{\tilde{a}})^{2-\varepsilon}\,,
\end{align*}
where the first inequality comes from Hölder's inequality and the third inequality is due to the fact that $\varepsilon\in(1,2]$.
\end{proof}

\section{Auxiliary Lemmas}
The following lemmas are included in this section for completeness.
\begin{lem}[Lemma 8 of \citet{Ito2022nearly}]\label{lem:stab_condition_Shinji22}
    Fix arbitrary $q \in \gP(\gA)$. Let  $\psi(q)= \sum_{a \in \gA}q_a\log q_a$ be the (negative) Shannon entropy. For any $\ell:\gA\to \sR$, we have
    \begin{align*}
        \left\langle\ell, q-p \right\rangle-\beta D_{\psi}(p,q) \leq \beta \suma q_a \xi\left(\frac{\ell_a}{\beta}\right)\,,\quad where\quad \xi(x)=\exp (-x)+x-1\,,
    \end{align*}
    for any $p\in \gP(\gA)$.
\end{lem}

\begin{lem}[Lemma 10 of \citet{ito2024adaptive}]\label{lem:stab_condition_Shinji24}
    Fix arbitrary $q \in \gP(\gA)$. Let $q_{*}\coloneqq \min\{\|q\|_{\infty}, 1-\|q\|_{\infty}\}$, $\tilde{q}_{a}\coloneqq \min\{q_a,q_{*}\}$, and $\psi(q)=-\frac{1}{\alpha} \sum_{a \in \gA}\left(q_a^\alpha-q_a\right)$ be the $\alpha$-Tsallis entropy. If $\ell_a\le \frac{1-\alpha}{4}q_{*}^{\alpha-1}$ holds for all $a\in\gA$, we have
    \begin{align*}
        \langle\ell, q-p\rangle-D_{\psi}(p, q) \leq \frac{4}{1-\alpha}\sum_{a \in\gA} \tilde{q}_a^{2-\alpha} \ell_a^2
    \end{align*}
    for any $p\in \gP(\gA)$.
\end{lem}

The following lemma is a slightly modified version of Lemma 10 of \citet{ito2024adaptive} in the sense that it leads to the same upper bound of the stability term while requiring a slightly weaker condition on the magnitudes of the losses. Its proof is nearly the same as that of Lemma 10 of \citet{ito2024adaptive} and is thus omitted.
\begin{lem}\label{lem:stab_condition_m1}
    Fix arbitrary $q \in \gP(\gA)$. Let $q_{*}\coloneqq \min\{\|q\|_{\infty}, 1-\|q\|_{\infty}\}$, $\tilde{q}_{a}\coloneqq \min\{q_a,q_{*}\}$, and $\psi(q)=-\frac{1}{\alpha} \sum_{a \in \gA}\left(q_a^\alpha-q_a\right)$ be the $\alpha$-Tsallis entropy.
    If $\ell_a\le \frac{1-\alpha}{4}\tilde{q}_{a}^{\alpha-1}$ holds for all $a\in\gA$, we have
    \begin{align*}
        \langle\ell, q-p\rangle-D_{\psi}(p, q) \leq \frac{4}{1-\alpha}\sum_{a \in\gA} \tilde{q}_a^{2-\alpha} \ell_a^2
    \end{align*}
    for any $p\in \gP(\gA)$.
\end{lem}

\begin{lem}[Lemma 11 of \citet{ito2024adaptive}]\label{lem:Shinji24_lem11}
    Fix arbitrary $x>1$. For any $p,q\in\gP(\gA)$, if $p_a\le xq_a$ for all $a\in\gA$, then it holds that $-\psi(p)\le -1(1+(x-1)\alpha)\psi(q)\le -x\psi(q)$.
\end{lem}

\begin{lem}[Lemma 12 of \citet{ito2024adaptive}]\label{lem:Shinji24_lem12}
    Let $\omega=\sqrt{2}$. Suppose $q, r \in \gP(\gA)$ are given by
    \begin{align}
        & q \in \underset{p \in \gP(\gA)}{\arg \min }\{\langle L, p\rangle+\beta \psi(p)+\bar{\beta} \bar{\psi}(p)\}, \notag\\
        & r \in \underset{p \in \gP(\gA)}{\arg \min }\left\{\langle L+\ell, p\rangle+\beta^{\prime} \psi(p)+\bar{\beta} \bar{\psi}(p)\right\}\notag
    \end{align}
    with 
    \begin{align}
        \psi(p)=-\frac{1}{\alpha} \sum_{a\in\gA}\left(p_a^\alpha-p_a\right), \quad \bar{\psi}(p)=-\frac{1}{\bar{\alpha}} \sum_{a\in\gA}\left(p_a^{\bar{\alpha}}-p_a\right)\notag
    \end{align}
    where $0 \leq \bar{\alpha}<\alpha<1,0<\beta \leq \beta^{\prime}$, and $0 \leq \bar{\beta}$. Denote $q_*=\min \left\{1-\max _{a\in\gA} q_a, \max _{a\in\gA} q_a\right\}$. If 
    \begin{align}
\|\ell\|_{\infty} &\leq \max \left\{\frac{1-\omega^{\alpha-1}}{2} \beta q_*^{\alpha-1}, \frac{1-\omega^{\bar{\alpha}-1}}{2} \bar{\beta} q_*^{\bar{\alpha}-1}\right\} \notag\\
0 \leq \beta^{\prime}-\beta &\leq \max \left\{\left(1-\omega^{\alpha-1}\right) \beta, \frac{1-\omega^{\bar{\alpha}-1}}{\omega} \bar{\beta} q_*^{\bar{\alpha}-\alpha}\right\}\,,\notag
    \end{align}
    then it holds that $r_a \leq 2 q_a$ for all $a\in\gA$.
\end{lem}

\end{document}